\documentclass{article}

\usepackage[preprint]{neurips_2024}

\usepackage[utf8]{inputenc} %
\usepackage[T1]{fontenc}    %
\usepackage{hyperref}       %
\usepackage{xcolor}         %

\hypersetup{
    colorlinks=true,
    linkcolor=[rgb]{0.0,0.0,0.9},
    filecolor=[rgb]{0.0,0.0,0.9},      
    urlcolor=[rgb]{0.0,0.0,0.9},
    citecolor=[rgb]{0.0,0.0,0.9},
    pdftitle={Trajectory Data Suffices for Statistically Efficient Learning in Offline RL with Linear q-pi Realizability and Concentrability},
    pdfpagemode=FullScreen,
    }

\usepackage{url}            %
\usepackage{booktabs}       %
\usepackage{amsfonts}       %
\usepackage{nicefrac}       %
\usepackage{microtype}      %

\usepackage{mathextra}

\usepackage[textsize=tiny,disable]{todonotes}
\newcommand{\todov}[2][]{\xspace}
\newcommand{\todoc}[2][]{\xspace}
\newcommand{\todocr}[2][]{\xspace}
\newcommand{\todog}[2][]{\xspace}

\title{Trajectory Data Suffices for Statistically Efficient Learning in Offline RL with Linear $q^\pi$-Realizability and Concentrability}

\author{%
  Volodymyr Tkachuk \\ %
  University of Alberta, Edmonton, Canada \\
  \And
  Gell\'ert Weisz\\
  Google DeepMind, London, UK\\
 \And
 {Cs}aba {Sz}epesv\'ari\\
 Google DeepMind, Edmonton, Canada\\
 University of Alberta, Edmonton, Canada\\
}

\begin{document}

\maketitle

\begin{abstract}
We consider offline reinforcement learning (RL) in $H$-horizon Markov decision processes (MDPs) under the linear $q^\pi$-realizability assumption, where the action-value function of every policy is linear with respect to a given $d$-dimensional feature function.
The hope in this setting is that learning a good policy will be possible without requiring a sample size that scales with the number of states in the MDP.
\citet{foster2021offline} have shown this to be impossible even under $\text{\textit{concentrability}}$, 
a data coverage assumption where a coefficient $C_\text{conc}$ bounds the extent to which the state-action distribution of any policy can veer off the data distribution.
However, the data in this previous work was in the form of a sequence of individual transitions.
This leaves open the question of whether the negative result mentioned could be overcome if the data was composed of sequences of full trajectories.
In this work we answer this question positively by proving 
that with trajectory data, a dataset of size $\text{poly}(d,H,C_\text{conc})/\epsilon^2$ is sufficient for deriving an $\epsilon$-optimal policy, regardless of the size of the state space. 
The main tool that makes this result possible is due to 
\citet{weisz2023online}, who demonstrate that 
linear MDPs can be used to approximate
linearly $q^\pi$-realizable MDPs.
The connection to trajectory data is that the linear MDP approximation relies on ``skipping'' over certain states. The associated estimation problems are thus easy when working with trajectory data, while they remain nontrivial when working with individual transitions.
The question of computational efficiency under our assumptions remains open.

\end{abstract}

\section{Introduction} \label{sec:introduction}
We study the offline reinforcement learning (RL) setting, where the objective is to derive a near-optimal policy for an $H$-horizon Markov decision process (MDP) using \textit{offline data}. 
This contrasts with the online RL paradigm, where learners interact directly with an MDP -- or its simulator -- to collect new data. 
Offline RL is especially relevant when acquiring new data guided by the learner is infeasible or ill-advised for safety reasons. 

Deriving a near-optimal policy is only possible from offline data that covers the MDP well enough.
One way 
to formalize this as an assumption, which we adopt for this work, is called \textit{concentrability}. 
This assumption posits that the offline data sufficiently covers the distribution of state-action pairs that are accessible through running any policy.
Challenges also arise in MDPs characterized by large or infinite state spaces. 
In such scenarios, an efficient learner's data requirements cannot scale with the state space size. %
An approach to remove state space dependence is to assume that the state-action value function of any policy can be linearly represented using a $d$-dimensional feature map, an assumption known as \textit{linear \(q^\pi\)-realizability}.

While linear \(q^\pi\)-realizability facilitates efficient online RL \citep{weisz2023online}, its applicability has been limited in offline contexts. 
For instance, \citet{foster2021offline} proves that no learning algorithm can derive an $\epsilon$-optimal policy under linear \(q^\pi\)-realizability and concentrability bounded by $\conc$, with a $\poly(d,H,\conc,\epsilon^{-1})$ number of samples.
However, their result does not apply to \textit{trajectory data}, where the offline data contains full sequences of state, action, and reward tuples 
obtained by following some policy from the initial state to the terminal state.
The following problem is left open:
\vspace{-1mm} \begin{center}
    \textit{``Does there exist an efficient learner that outputs an $\eps$-optimal policy, under the assumptions of linear $q^\pi$-realizability, concentrability, and trajectory data?''}
\end{center}
\vspace{-1mm} Our findings affirmatively answer this question in terms of sample complexity, highlighting a notable distinction in the requirements for trajectory data versus general offline data for effective learning. 
This underscores the practical value of accumulating trajectory data whenever feasible.

\begin{table}[t]
\centering
\begin{tabular}
{|C{3.2cm}|C{2.2cm}|C{2.7cm}|C{2.7cm}|C{0.9cm}|}
\hline

&
&\multicolumn{2}{c|}{Assumptions}
& 
\\

Work 
& Task 
& Data 
& Structural
& Result 
\\\hline\hline

\citep{xiong2022nearly}
& $\pi$-opt 
& $\lambdamin$ lower bound 
& Linear MDP
& \textcolor{darkgreen}{\checkmark}
\\\hline

\citep{chen2019information}
& $\pi$-opt 
& Conc
& Bellman complete
& \textcolor{darkgreen}{\checkmark}
\\\hline

\citep{xie2021batch}
& $\pi$-opt 
& Strong Conc
& $q^\pi$
& \textcolor{darkgreen}{\checkmark}
\\\hline

\rowcolor{Gray}
\textbf{[This work]}
& $\pi$-opt
& Conc \& \textbf{Traj data}
& $q^\pi$
& \textcolor{darkgreen}{\checkmark}
\\\hline

\citep{foster2021offline}
& $\pi$-opt or $\pi$-eval
& Conc
& $q^\pi$
& \textcolor{red}{x}
\\\hline

\citep{wang2020statistical}
& $\pi$-eval
& $\lambdamin$ lower bound
& $q^\pi$
& \textcolor{red}{x}
\\\hline

\citep{jia2024offline}
& $\pi$-eval
& Conc
& Restricted $q^\pi$
& \textcolor{red}{x}
\\\hline

\end{tabular}
\medskip
\caption{
Notation is defined as: $\pi$-opt = policy optimization, $\pi$-eval = policy evaluation, Conc = Concentrability, $q^\pi$ = linear $q^\pi$-realizability, Traj = Trajectory, \textcolor{darkgreen}{\checkmark} = $\poly(d, H, \conc, 1/\eps)$ sample complexity, \textcolor{red}{x} = exponential lower bound in terms of one of $d, H, \conc$.
}
\label{tab:bpse}
\end{table}

\section{Related Works} \label{sec:related works}

In \cref{tab:bpse} we provide a comparison of our result to the other works in offline RL discussed below. 

\textbf{Lower bounds:}
As we have already discussed in \cref{sec:introduction}, the work of \citet{foster2021offline} shows a lower bound that depends on the size of the state space (in the same setting as ours), except they do not assume access to trajectory data.
The work by \citet{jia2024offline} is perhaps the most relevent to ours.
They show an exponential lower bound in the horizon for policy evaluation, under the assumptions of trajectory data, concentrability, and a \textit{restricted} linear $q^\pi$-realizability where the value function of only the target policy is linear.
While
we anticipate that evaluating policies (their focus) is no more difficult than optimizing policies (our focus),
our $q^\pi$ realizability is for all memoryless policies (\cref{ass:approximate q-pi realizability}), while theirs is restricted to the target policy.
\citet{zanette2021exponential} shows an exponential lower bound in terms of the feature dimension $d$, under linear $q^\pi$-realizability, and various other structural assumptions; however, their setting would result in a%
 concentrability coefficient larger than the size of the state space.  
\citet{wang2020statistical,amortila2020variant} show a lower bound that is exponential in the horizon, under linear $q^\pi$-realizability.
However, they use a \textit{$\lambdamin$ lower bound} condition, which requires a lower bound on the minimum eigenvalue $\lambdamin$ of the expected covariance matrix used for least-squares estimation. 
This is seen as a weaker condition than ours,
as it only posits good coverage in terms of the feature space, not the (possibly much richer) state-action space.

\textbf{Upper bounds:}
\citet{chen2019information,munos2008finite} present efficient algorithms under concentrability and \textit{Bellman completeness}, 
an assumption that the Bellman optimality operator
outputs a linearly realizable function when its input is linearly realizable.
As linear $q^\pi$-realizability does not imply Bellman completeness \citep{zanette2020learning}, these results do not transfer to our setting.
\citet{xie2021batch} show an upper bound under linear $q^\pi$-realizability, albeit, using a stronger notion of data coverage than concentrability, which we call \textit{strong concentrability}.
The work of \citet{xie2021bellman,xie2022armor} give data-dependent sample complexity bounds that hold under both Bellman completeness and linear $q^\pi$-realizability even in the absence of explicit data coverage assumptions.
\citet{jin2021pessimism} assume a general function approximation setting and also provide data-dependent bounds, 
while \citet{duan2020minimax,xiong2022nearly} 
show upper bounds for \textit{linear MDPs} (a stronger assumption than linear $q^\pi$-realizability \citep{zanette2020learning}) with the $\lambdamin$ lower bound condition.

\section{Setting} \label{sec:setting}

Throughout we fix the integer $d \ge 1$. 
Let $\vec 0 \in \bR^d$ be the $d$-dimensional, all zero vector.
For $L>0$, let $\cB(L)=\{x \in \bR^d: \norm{x}_2 \le L\}$ denote the $d$-dimensional Euclidean ball of radius $L$ centered at the origin, where $\|\cdot\|_2$ denotes the Euclidean norm.
The inner product $\langle x, y \rangle$ for $x, y \in \bR^d$ is defined as the dot product $x^\T y$.
Let $\I{B}$ be the indicator function 
 of a boolean-valued (possibly random) variable $B$, taking the value $1$ if $B$ is true and $0$ if false.
Let $\Dists(X)$ denote the set of probability distributions over the set $X$. %
Let $\bE_{B \sim \cP}$ denote the expectation of random variable $B$ under distribution $\cP$. 
For integers $i,j$, let $[i]=\{1,2,\dots,i\}$ and $[i:j]=\{i,\dots,j\}$.
For a symmetric matrix $M\in\R^{d\times d}$ we write $\lambdamin(M)$ and $\lambdamax(M)$ for its minimum and maximum eigenvalue.

The environment is modeled by a finite horizon Markov decision process (MDP). Fix the horizon to $H$.
This MDP is defined by a tuple \((\mathcal{S}, \mathcal{A}, P, \mathcal{R})\). 
Here, the state space \(\mathcal{S}\) is finite\footnote{The state space is assumed to be finite to simplify presentation.
Our results extend to infinite state spaces.},
and organized by stages: $\cS = \bigcup_{h \in [H+1]} \cS_h$,
starting from a designated initial state $s_1$ ($\mathcal{S}_1 = \{s_1\}$)\footnote{A deterministic start state $\startstate$ is added for simplicity of presentation. 
It is easy to show that adding an additional stage to the MDP allows for the transition dynamics to encode an arbitrary start state distribution.}
, and culminating in a designated terminal state $\termstate$ ($\termstatespace=\{\termstate\}$)\footnote{A terminal state $\termstate$ is added purely as a technical convenience for the analysis. %
We will focus on the interaction of learners for stages $h \in [H]$ (not $[H+1]$), since the terminal state will have no affect on the learner.}.
Without loss of generality, we assume $\cS_h$ and $\cS_{h'}$ for  $h \neq h'$ are disjoint sets.
Define the function
 \(\stage: \mathcal{S} \to [H+1]\), such that \(\stage(s) = h\) if \(s \in \mathcal{S}_h\). 
The action space \(\mathcal{A}\) is finite. 
The transition kernel is \(P: (\bigcup_{h \in [H]} \mathcal{S}_h) \times \mathcal{A} \to \Dists(\mathcal{S})\), with the property that transitions occur between successive stages. 
Specifically, for any \(h \in [H]\), state \(s_h \in \mathcal{S}_h\), and action \(a \in \mathcal{A}\), $P(s_h, a) \in \Dists(\cS_{h+1})$.
The reward kernel is $\cR:\cS \times \cA \to \Dists([0,1])$. 
So that the terminal state $\termstate$ has no influence on the learner we force the reward kernel to deterministically give zero reward for all actions $a \in \cA$ in $\termstate$ (i.e. $\cR(\termstate, a)(r) = \I{0 = r}$). %
An agent interacts with this environment sequentially across an episode of $H+1$ stages, 
by selecting an action \(a \in \mathcal{A}\) in the current state. 
The environment (except at stage $H+1$) then transitions to a subsequent state according to \(P\) and provides a reward in $[0, 1]$ as specified by \(\mathcal{R}\)\footnote{Here, the reward and next-state are independent, given the current state and last action. Independence is nonessential and is assumed only to simplify the presentation.}.

We define an agent's interaction with the MDP through a \emph{policy} \(\pi\), which assigns a probability distribution over actions based on the history of interactions (including states, actions, and rewards).
For this work, we restrict policies to be \emph{memoryless}, that is, their action distribution depends solely on the most recent state in the history.
The set of all memoryless policies is
\(\Pi = \{\pi\,:\, \pi : \mathcal{S} \to \Dists(\mathcal{A})\}\).
For $\pi\in\Pi$, we write $\pi(a|s)$ to denote the probability $\pi(s)$ assigns to action $a$. For deterministic policies only (i.e., those that for each state place a unit probability mass on some action) we sometimes abuse notation by writing $\pi(s)$ to denote $\argmax_{a\in\cA} \pi(a|s)$.
Starting from any state \(s\) within the MDP and using a policy \(\pi\) induces a probability distribution over trajectories, denoted as \(\mathbb{P}_{\pi, s}\).
For any $a \in \cA$, $\bP_{\pi,s,a}$ is the distribution over the trajectories when first action $a$ is used in state $s$, after which policy $\pi$ is followed.
Specifically, for some $h\in[H+1]$ and $(s, a) \in\cS_h \times \cA$, we write $\trajectoryrand\sim\bP_{\pi,s,a}$ to denote that $\trajectoryrand=(S_h,A_h,R_h,\dots,S_{H+1},A_{H+1},R_{H+1})$
for a random trajectory that follows the distribution specified by $\bP_{\pi,s,a}$, that is,
$S_h=s$, $A_h=a$, $A_i\sim \pi(S_i)$ for $i\in[h+1:H+1]$, $S_{i+1}\sim P(S_i,A_i)$ for $i\in[h:H]$, and $R_i\sim \cR(S_i,A_i)$ for $i\in[h:H+1]$.
Writing $\trajectoryrand\sim\bP_{\pi,s}$ has an analogous meaning, with the only difference being that $A_h$ is not fixed, and instead $A_h\sim\pi(S_h)$. 
For $h\in[H+1]$, we write $\bPmarg{h}_{\pi,s}$ (and $\bPmarg{h}_{\pi,s,a}$) for the marginal distribution of $(S_h, A_h)$ (i.e., the state-action pair of stage $h$) under the joint distribution of $\bP_{\pi,s}$ (and $\bP_{\pi,s,a}$). 

For $1 \le t \le t' \le H+1$, we use the notation  $x_{t:t'} = \textstyle(x_u)_{u \in [t:t']}$ 
throughout, except when $(x_u)_{u \in [t:t']}$ are a sequence of scalar rewards.
In that case, for convenience, we write $r_{t:t'} = \textstyle\sum_{u=t}^{t'} r_u$ and $R_{t:t'} = \sum_{u=t}^{t'} R_u$. %
The state-value and action-value functions $v^\pi$ and $q^\pi$ are defined as the expected total reward along the rest of the trajectory while $\pi$ is used: 
\begin{align*}
v^\pi(s)=\!\!\!\bigE_{\trajectoryrand\sim\bP_{\pi,s}} \!\!\!\!\!\!R_{\stage(s):H} \,\, \text{for }s\in\cS\,
\quad \text{and} \quad 
q^\pi(s,a)=\!\!\!\bigE_{\trajectoryrand\sim\bP_{\pi,s,a}} \!\!\!\!\!\!R_{\stage(s):H} \,\, \text{for }(s, a) \in \cS \times \cA \, .
\end{align*}
Let $\pi^\star\in\Pi$ be an optimal policy, satisfying $q^{\pi^\star}(s,a)=\sup_{\pi\in\Pi}q^\pi(s,a)$ for all $(s,a) \in \cS \times \cA$.
Let $q^\star(s,a)=q^{\pi^\star}(s,a)$ and $v^\star(s)=\max_{a \in \cA}q^\star(s,a)$ for all $(s,a) \in \cS \times \cA$.
By definition, we have %
 
\begin{align}
    v^\star(\termstate) = v^\pi(\termstate) = 0 \quad \text{and} \quad q^\star(\termstate, a) = q^\pi(\termstate, a) = 0 && \text{for all } \pi \in \Pi, a \in \cA \, .
    \label{eq:v and q at term state}
\end{align}

\subsection{Assumptions and Problem Statement} \label{ss:assumptions and problem}

A feature map is defined as $\phi: \cS \times \cA \to \cB(L_1)$ for some $L_1 > 0$. The representative power of a feature map for an MDP is described by the following assumption:

\begin{assumption}
[$(\misspec, \thetabound)$-Approximately Linear $q^\pi$-Realizable MDP] \label{ass:approximate q-pi realizability}
    For some $\misspec \ge 0, \thetabound > 0$,
    assume that the MDP (together with a feature map $\phi$) is such that
    \begin{align*}
        \sup_{\pi \in \Pi} \min_{\theta_h \in \cB(\thetabound)} \max_{(s_h,a_h)\in\cS_h\times\cA} \, \abs{q^\pi(s_h, a_h) - \ip{\phi(s_h, a_h), \theta_h}}
        \le \misspec 
         && \text{for all } h \in [H+1] %
        \, .
    \end{align*}
    For any $h\in[H+1]$, let $\psi_h:\Pi\to\cB(L_2)$ be a mapping from policies to parameter values $\theta_h$ that attain the $\min$ in the above display.
    For $h=H+1$, we restrict this mapping to $\psi_{H+1}(\cdot)=\vec 0$, which satisfies the above display by definition. %
    We write $\psi_{h:t}(\pi)$ for $(\psi_{h}(\pi), \dots, \psi_{t}(\pi))$.
\end{assumption}

We also make the following assumptions on the offline data:

\begin{assumption}[Full Length Trajectory Data] \label{ass:batch data}
    Assume the learner is given a dataset of full length trajectories and corresponding features of size $n \ge 1$: %
    \begin{align*}
        \left(\trajectory^1, \dots, \trajectory^n\right) \quad\text{and}\quad
        \left((\phi(s_h^1,\cdot))_{h \in [H]}, \dots, (\phi(s_h^n,\cdot))_{h \in [H]}\right) \, ,
    \end{align*}
    where for some ``data collection policy''%
    $\behavepi \in \Pi$ unknown to the learner,
    $(\trajectory^j)_{j=1}^n$ are independent samples from $\bP_{\behavepi, \startstate}$ where $\trajectory^j = (s_t^j, a_t^j, r_t^j)_{t \in [H+1]}$. 
    To simplify notation we write
    \[\phi_h^j = \phi(s_h^j, a_h^j) \quad\text{for all $h \in [H], j \in [n]$}\,.\]
\end{assumption}

\begin{definition}[Admissible Distribution] \label{def:admissible dist}
    A sequence of $H$ state-action distributions $\nu = (\nu_h)_{h \in [H]} \in (\Dists(\cS_h \times \cA))^H$ is admissible for an MDP if there exists a policy $\pi \in \Pi$ such that \begin{align*}
        \nu_h(s_h, a_h) = \bPmarg{h}_{\pi, \startstate}(s_h, a_h) \quad \text{for all } (s_h, a_h) \in \cS_h \times \cA, \ h \in [H] \, .
    \end{align*}
\end{definition}

\vspace{-2mm}Define the state-action occupancy measure of the data collection  policy $\behavepi$ as $\mu = (\mu_h)_{h \in [H]}$ such that
\begin{align*}
    \mu_h(s_h, a_h) = \bPmarg{h}_{\behavepi, \startstate}(s_h, a_h) \quad \text{for all } (s_h, a_h) \in \cS_h \times \cA, \ h \in [H] \, .
\end{align*}
\begin{assumption}[Concentrability] \label{ass:concentrability}
    Assume there exists a constant $\conc \ge 1$, such that for all admissible distributions $\nu = (\nu_h)_{h \in [H]}$
    \begin{align*}
        \max_{h \in [H]} \max_{(s_h, a_h) \in \cS_h \times \cA} \left \{ \frac{\nu_h(s_h, a_h)}{\mu_h(s_h, a_h)} \right \} \le \conc \, .
    \end{align*}
\end{assumption}

\begin{problem}\label{problem}
    Let $\eps > 0$.
    Under \cref{ass:approximate q-pi realizability,ass:batch data,ass:concentrability}, does there exist a learner, with access to only $n = \poly(1/\epsilon, H, d, \conc)$ full length trajectories (as defined in \cref{ass:batch data}), 
    that outputs a policy $\pi$ such that, with probability at least $1 - \delta$, 
    \begin{align*}
        v^\star(\startstate) - v^\pi(\startstate) \le \epsilon \, ?
    \end{align*}
\end{problem}

\vspace{-1mm}\section{Result} \label{sec:result}

\vspace{-1mm}We resolve \cref{problem} in the positive by defining a learner (formal definition in \cref{ss:method}) that: selects parameters optimistically from modified MDPs that ``skip over'' certain states while preserving tight $q$-value estimation guarantees (achieved by solving \cref{opt:bpse start state});
then, outputs a greedy policy $\optpi$ (defined in \cref{eq:opttheta and optpi defn}) over the selected parameters.
This result is made formal in following theorem (proof in \cref{sec:proof of the result}):

\begin{theorem} \label{thm:bpse}
    Let $\eps \in (0,H]$.
    Under \cref{ass:approximate q-pi realizability,ass:batch data,ass:concentrability}, 
    if the number of full length trajectories $n = \tilde \Theta(\conc^4 H^7 d^4/\eps^2)$ 
    and $\misspec = \bigOt{\eps^2/ (\conc^2 H^5 d^2)}$\footnote{The bound on $\misspec$ is assumed for clarity of presentation, to avoid presenting two error terms in the final bound.},
    then, 
    with probability at least $1 - \delta$, 
    the policy $\optpi$ output by our learner (defined in \cref{ss:method}) satisfies
    \begin{align*}
        v^\star(\startstate) - v^{\optpi}(\startstate) 
        \le \eps \, ,
    \end{align*}
\end{theorem}

\vspace{-1mm}where $\tilde \Omega, \tilde \cO$ and $\tilde \Theta$ are the counterparts of $\Omega, \cO$ and $\Theta$ from the big-Oh notation that hide polylogarithmic factors of the problem parameters $(1/\eps, 1/\delta, H, d, \conc, \featurebound, \thetabound)$. 
The following subsections focus on introducing the theory needed to formally present our learner, giving intuition behind our learner, and presenting our learner.

\subsection{Background Theory} 
\label{ss:background theory}

Our learner relies on the observation due to \citet{weisz2023online} that linearly $q^\pi$-realizable MDPs are linear MDPs, as long as they contain no low-range states.
The \emph{range} of a state is the largest possible regret from that state, that is, the largest difference in action-value that the choice of action in that state can make (up to misspecification):
\begin{align}\label{eq:range-def}
\range(s)=\textstyle\sup_{\pi\in\Pi} \textstyle\max_{a,a'\in\cA} \ip{\phi(s,a,a'), \psi_{\stage(s)}(\pi)}
\quad \text{ for all } %
s\in\cS\,,
\end{align}
where $\phi(s,a,a')=\phi(s,a)-\phi(s,a')$ is a notation we use to denote feature differences. Intuitively, the choice of actions in low-range states are unimportant, as %
\begin{equation}
\abs{v^\pi(s)-q^\pi(s,a)} \le \range(s)+2\eta \quad\quad\text{ for any $\pi \in \Pi$ and all  $(s,a)\in\cS\times\cA$.} \label{eq:vstar-vs-range}
\end{equation}
If we modify our linearly $q^\pi$-realizable MDP to ``skip over'' low-range states it will be a linear MDP. 
The key fact about linear MDPs that we will use is that for any function $f:\cS\to[0,H]$ (e.g., $v$-value approximators), and any $h\in [H]$, there is some parameter $\theta_h\in \R^d$ so that
for \emph{any} $(s, a) \in \cS_h \times \cA$,
 $\ip{\phi(s,a),\theta_h}$ gives 
the expectation of the reward plus $f$'s value on the next state.%
In our modified MDP this result transfers to the fact that the expected sum of rewards along a skipped path, plus $f$'s value on the next state after the skipped path, is linearly realizable.
Before making this result formal in \cref{lem:approx linear MDP},
we clarify the skipping behavior.

First, we address the fact that we need an approximate, parametric bound on $\range(\cdot)$ 
with a parameter count that is independent of $|\cS|$.
For $h\in[2:H]$, let $\Psi_h=\{\psi_h(\pi)\,:\, \pi\in\Pi\} \subseteq \cB(\thetabound)$ be the (compact)%
set of parameter values corresponding to all policies.
For all $h\in[2:H]$,  
fix a subset $\bar G_h\subset \Psi_h$ of size $|\bar G_h|=d_0:=\lceil 4d\log\log(d)+16 \rceil$ that is the basis of a near-optimal design for $\Psi_h$ (more precisely, satisfying
\cref{def:nearopt}).
The existence of such a near-optimal design follows from \cite[Part (ii) of Lemma 3.9]{todd2016minimum}.
Let $\trueG = \bar G_{2:H}$, 
which we call the \emph{true guess}. Now notice that $\trueG\in \bbG$ where
\begin{align}\label{eq:bbG-def}
        \bbG = (\cB(\usedtobesqrtdoneplusone))^{[2:H]\times [d_0]}\,.        
    \end{align}
For $G\in \bbG$ we will use the notation that $G = G_{2:H}$, where $G_h =    (\vartheta_h^i)_{i \in [d_0]}  \in \cB(\usedtobesqrtdoneplusone)^{d_0}$.
Any $G = G_{2:H}\in\bbG$ can be used to define an approximate, low parameter-count ``version'' of $\range$ that is completely specified by $\tilde O(Hd^2)$ parameters:
\begin{align}\label{eq:rangeq-def}
\textstyle \range^G(s)=\max_{\vartheta \in G_h} \max_{a,a'\in\cA} \ip{\phi(s,a,a'), \vartheta}\quad\quad\text{for all } h\in[2:H], s\in\cS_h\,.
\end{align}
As shown in Proposition 4.5 of \citep{weisz2023online}, $\range^\trueG$ can be used to bound the true $\range$: %
\begin{lemma}\label{lem:range-G-accurate}
For all
$h\in[2:H]$ and $s\in\cS_h$, 
$\range(s) \le \sqrt{2d} \cdot \range^\trueG(s)$.
\end{lemma}

Based on any $G\in\bbG$, we are interested in simulating a modified MDP that ``skips over'' states $s$ that have a low $\range^G(s)$, by taking an action according to $\behavepi$, and presenting as the reward the summed up rewards along paths consisting of skipped states.
This ``modified MDP'' only serves as intuition, and will not be formally defined or used in our formal arguments.
Instead, we define the ``skipping probability'' parameter at state $s \in \cS$, with $\alpha>0$ (set later in \cref{def:alpha}), as 
\begin{align}
    \omega_G(s)
    = \begin{cases}
        1 & \text{if } s \not\in \cS_1\cup\cS_{H+1} \text{ and } \range^{G}(s) \le \alpha / \sqrt{2d} \\
        2 - \sqrt{2d} \cdot \range^{G}(s) / \alpha & \text{if } s \not\in \cS_1\cup\cS_{H+1} \text{ and } \alpha /\sqrt{2d} \le \range^{G}(s) \le 2\alpha / \sqrt{2d} \\
        0 & \text{otherwise.}
    \end{cases} \label{eq:omega}
\end{align}
The skipping behavior is probabilistic\footnote{The resulting smoothness of skipping behavior is beneficial for a later technical covering argument (\cref{eq:omega cover bound} in \cref{lem:G cover useful results}).}: it never skips for stages $1$ and $H+1$ (where $\range^G$ is not defined); it always skips for ranges lower than some threshold, never skips for ranges higher than twice this threshold, and linearly interpolates between the two in between the thresholds.
For $h \in [H]$, and $1\le l \le h$,
let $\trajectory = (s_t, a_t, r_t)_{l\le t \le H+1}$ be any fixed trajectory that starts from some stage $l$. 
Let $\tau \sim F_{G, \trajectory, h+1} \in \cM_1([h+1:H+1])$ be the random stopping stage, 
when starting from state $s_h$ and skipping subsequent states with probability $\omega_G(\cdot)$.
Formally, for $t \in [h+1:H+1]$ let $F_{G, \trajectory, h+1}(\tau = t) = (1 - \omega_G(s_t)) \prod_{u = h+1}^{t-1} \omega_G(s_u)$.
We will often write $F^j_{G, h+1}$ to denote $F_{G, \trajectory^j, h+1}$ where $\trajectory^j = (s_t^j, a_t^j, r_t^j)_{t \in [H+1]}, j \in [n]$. 

Next, we present a key tool derived from results of \citet{weisz2023online}:
as long as the skips are informed by the true guess, the resulting MDP is approximately linear (proof 
in \cref{proof:lem:approx linear MDP}):
\begin{lemma}[Approximate Linear MDP under the true guess] \label{lem:approx linear MDP}
    Let $\misspec \ge 0, \thetabound > 0$.
    Let $M$ be an $(\misspec, \thetabound)$-approximately linear $q^\pi$-realizable MDP (\cref{ass:approximate q-pi realizability}) with corresponding feature map $\phi$.
    Let $\modthetabound = \thetabound(8H^2d_0/\alpha+1)$. 
    Then,
    for each $f:\cS\to[0,H]$ 
    with $f(\termstate)=0$,
    policy $\pi \in \Pi$,
    and stage $h \in [H]$, there exists
    a parameter $\rho_h^\pi(f) \in \cB(\modthetabound)$
    such that for all $(s, a) \in \cS_h \times \cA$,
    \begin{align*}
        \abs{\bigE_{\trajectoryrand \sim \bP_{\pi, s, a}} \bigE_{\tau \sim F_{\trueG, \trajectoryrand, h+1}} \left[R_{h:\tau-1} + f(S_\tau)\right] 
            - \ip{\phi(s,a),\rho_h^\pi(f)}} 
        \le \modmisspec \, ,
    \end{align*}
    where $\modmisspec = \eta(10H^2d_0/\alpha+1)$.
\end{lemma}

\subsection{The Benefit of Trajectory Data} \label{ss:trajectory data}

Our learner will heavily rely on the result presented in \cref{lem:approx linear MDP}.
We will need to learn good estimates of the parameters $\rho^\behavepi_h(f)$, for any $f:\cS \to [0, H], h \in [H]$.
However, to estimate a $\rho^\behavepi_h(f)$ parameter well we will require least-squares targets that have bounded noise and expectation equal to $\ip{\phi(s, a), \rho^\pi_h(f)}$ for all $(s, a) \in \cS_h \times \cA$.
Full trajectory data (\cref{ass:batch data}) makes this possible.
Each full length trajectory $\trajectory^j = (s_t^j, a_t^j, r_t^j)_{t \in [H+1]} j \in [n]$ can be used to create the following least-squares target (which has the desired properties): 
\begin{align*}
    \bigE_{\tau \sim F^j_{\trueG, h+1}}\left[r^j_{h:\tau-1} + f\left(s^j_\tau\right)\right] \, .
\end{align*}

\vspace{-2mm}
Importantly, it is because we have full length trajectories that we can transform the data available to simulate arbitrary length skipping mechanisms.

\subsection{Intuition Behind our Learner} \label{ss:intuition}

Next, we describe the high-level intuition and ideas behind our learner.
Consider the ``modified'' MDP where low-range states are skipped. 
As the learner has access to trajectory data (\cref{ass:batch data}), it can transform this data accordingly to simulate trajectories from the modified MDP.
Any near-optimal policy for the modified MDP is also near-optimal for the original MDP (due to \cref{eq:vstar-vs-range}).
Thus, our previous linear realizability property (\cref{lem:approx linear MDP}) 
allows for an offline RL version of the algorithm \textsc{Eleanor} \citep{zanette2020learning} to statistically efficiently derive a near-optimal policy for the modified MDP.
Indeed, the optimization problem underlying \textsc{Eleanor} serves as a starting point for \cref{opt:bpse start state}, which is at the heart of our learner.

The challenge is that the true guess $\trueG$ that \cref{lem:approx linear MDP} relies upon is not known to the learner.
This means that the learner is not given any explicit information of what states to ``skip over''.
To overcome this, we design a learner to output the policy $\pi'$ defined together by \cref{eq:opttheta and optpi defn} and \cref{opt:bpse start state},
where the optimization problem considers all guesses for the possible values of $\trueG$.
For each $G\in\bbG$, it considers the MDP that skips over low-range states when the range is calculated according to $G$.
It then calculates sets $\Theta_{G, h}$ for each stage $h$, that are guaranteed (with high probability) to include the parameter $\psi_h(\pi^\star_G)$ realizing $q^{\pi^\star_G}$ (where $\pi^\star_G$, defined in \cref{eq:pi*G definition}, is the optimal policy in the MDP with skipping based on $G$).
We achieve this by defining $\Theta_{G, h}$ backwards for $h=H,H-1,\dots,1$. By induction, if $\Theta_{G, h+1},\dots,\Theta_{G, H}$ all contain the desired parameter for their stage, then 
\emph{some} parameter sequence in the Cartesian product $\Theta_{G, h+1}\times\dots\times\Theta_{G, H}$ allows us to near-perfectly (up to some misspecification error) compute $q^{\pi^\star_G}$-values of stages $>h$.
Therefore, the least-squares parameter based on this sequence will be near the true parameter for stage $h$.
Defining $\hat\Theta_{G, h}$ to be all least-squares predictors for sequences in the aforementioned Cartesian product, and $\Theta_{G, h}$ to be unions of the confidence ellipsoids around these predictors ensures the true parameter $\psi_h(\pi^\star_G)$ realizing $q^{\pi^\star_G}$ for stage $h$ is included in $\Theta_{G, h}$.
This argument is made precise in \cref{lem:pi*G in Theta}.

There are two problems remaining. One is that some values of $G$ considered by \cref{opt:bpse start state} lead to skipping over important large-value states, degrading the performance of the best policy $\pi^\star_G$ available under that skipping.
The other problem is that at the expense of making sure the true parameters are included in the sets $\Theta_{G, h}$, these sets might become large, in the sense of containing parameters that lead to very different predictions.
Avoiding the first problem would make $v^{\pi^\star_G}(s_1)$ nearly as large as $v^\star(s_1)$.
Avoiding the second problem would lead to tight $q$-value estimators, and therefore to $v^{\pi'_G}(s_1)$ being nearly as large as $v^{\pi^\star_G}(s_1)$, for a policy $\pi'_G$ that is greedy with respect to our hypothetically tight $q$-value estimator.
A key idea is to reject from consideration any $G\in\bbG$ that leads to $q$-estimations that are not sufficiently tight (\cref{eq:opt prob cond}).%
The reason we can do this is because for $G=\trueG$ we can show that this condition passes (with high probability), and therefore we do not reject $\trueG$ (precise statement in \cref{lem:true guess is feasible}).
We can show this since we have trajectory data (\cref{ass:batch data}), allowing us to use least-squares targets of the form used in \cref{eq:hat-Theta defintion}, which we know are linearly realizable when $G = \trueG$ (discussed in \cref{ss:trajectory data}).
Finally, we resolve the first problem by selecting among these tight estimators the one that guarantees the highest policy value from $s_1$, which can be no worse than the value guaranteed by the choice of $G=\trueG$, which itself can be seen to be close to $v^\star(s_1)$.

\vspace{-1mm}\subsection{Learner} \label{ss:method}

\vspace{-1mm}Next, we formally introduce our learner,
at the heart of which lies \cref{opt:bpse start state}. 
We define various $q$ and $v$-value estimators that we use.
For $x \in \bR$, let $\clip_{[0, H]} x = \max\set{0, \min\set{H, x}}$.
Then, for $h \in [H], s \in \cup_{t \in [h:H+1]} \cS_t, a \in \cA, \theta \in \bR^d, \theta_{h:H+1} = (\theta_h, \dots, \theta_{H+1}) \in \bR^{d(H-h+2)}$, let
\begin{align}
    q_{\theta}(s, a) = \ip{\phi(s, a), \theta},
    &\quad q_{\theta_{h:H+1}}(s, a) = q_{\theta_{\stage(s)}}(s, a), \label{eq:q_theta definition} \\
    \textstyle v_{\theta}(s) = \max_{a \in \cA} q_{\theta}(s, a),
    &\quad v_{\theta_{h:H+1}}(s) = \max_{a \in \cA} q_{\theta_{h:H+1}}(s, a), \label{eq:v_theta definition} \\
    \bar q_{\theta}(s, a) = \clip_{[0, H]} q_{\theta}(s, a),
    &\quad \bar q_{\theta_{h:H+1}}(s, a) = \bar q_{\theta_{\stage(s)}}(s, a), \label{eq:bar-q_theta definition} \\
    \bar v_{\theta}(s) = \clip_{[0, H]} v_{\theta}(s),
    &\quad \bar v_{\theta_{h:H+1}}(s) = \bar v_{\theta_{\stage(s)}}(s) \, . \label{eq:bar-v_theta definition}
\end{align}

\begin{optproblem} \label{opt:bpse start state}
    \begin{align}
        &\qquad\qquad\qquad \textstyle\argmax_{G \in \bbG,\theta^\dag_{1:H+1}\in\Theta_{G, 1}\times\dots\times\Theta_{G, H+1}}  \, \bar v_{\theta^\dag_1}(\startstate) \qquad \text{subject to, for all $h \in [H]$} \nonumber \\
        &X_h = \lambda I + \textstyle\sum_{j \in [\nh]} \phi^j_h (\phi^{j}_h)^\top \, , 
        \label{def:covariance matrix} \\
        &\hat \Theta_{G, h} = \Biggl\{X_h^{-1} \sum_{j \in [\nh]} \phi^j_h \bigE_{\tau \sim F^j_{G, h+1}}\left[r^j_{h:\tau-1} + \bar v_{\theta_{h+1:H+1}} \left(s^j_{\tau}\right)\right]: \theta_{h+1:H+1} \in \bigtimes_{u=h+1}^{H+1}\Theta_{G, u}\Biggr\} \, , 
        \label{eq:hat-Theta defintion} \\ 
        &\Theta_{G, h} = \left\{\theta_h \in \cB(\tilde L_2): \textstyle\min_{\hat \theta_h \in \hat \Theta_{G, h}} \snorm{\theta_h - \hat \theta_h}_{X_h} \le \beta\right\}, \,\, \Theta_{G, H+1} = \{\vec 0\} \,\,\,\, \text{$\beta$ defn \cref{def:beta}} \, ,
        \label{eq:Theta definition} \\
        &\textstyle\frac{1}{\nh} \textstyle\sum_{j \in [\nh]} \left(\textstyle\max_{\theta \in \Theta_{G, h}} \bar q_{\theta}(s_h^j, a_h^j) - \textstyle\min_{\theta \in \Theta_{G, h}} \bar q_{\theta}(s_h^j, a_h^j)\right)
        \le \bar \eps,  \qquad\qquad\,\, \text{$\bar \eps$ defn \cref{def:bar-eps}} \, . 
        \label{eq:opt prob cond}
    \end{align}
\end{optproblem}

Let $(\optG, \opttheta{1:H+1})$ denote the solution to \cref{opt:bpse start state}. 
Notice that 
unlike \textsc{Eleanor}, we optimize over all $G\in\bbG$, which can be seen as an optimization over all possible ``modified MDPs'' with different skipping mechanisms.
Another observation is that
apart from $h=1$, the choice of $\opttheta{h}$ from $\Theta_{\optG,h}$ made by the optimization is arbitrary.
Indeed, unlike \textsc{Eleanor}, which chooses globally optimistic least-squares predictors for each stage, our optimization does not need to care about (or optimize for) the specific choice of $q$-value predictors from their respective confidence sets $\Theta_{G, h}$. 
We can be agnostic to the choice of $q$-value predictors because all choices lead to similar predictions due to \cref{eq:opt prob cond}, as will be shown formally, in \cref{lem:Theta elements are close}.
However, fixing an arbitrary concrete choice in the optimization allows us to define an output policy $\optpi$ that is parametrized only by these vectors, making both the memory and computational requirements of representing and executing $\optpi$ small.

Thus, our learner solves \cref{opt:bpse start state}, and then outputs the following policy:

\vspace{-2mm}\begin{align}
    \textstyle
    \optpi(a|s) = 
    \I{a = \argmax_{a' \in \cA} \bar q_{\opttheta{\stage(s)}}(s, a')}
    && \text{for all } (s, a) \in \cS \times \cA \, .
    \label{eq:opttheta and optpi defn}
\end{align}

\vspace{-2mm}\section{Proof of \cref{thm:bpse}} \label{sec:proof of the result}

\vspace{-1mm}Before giving the proof, we formally define the optimal policy in the modified MDP that skips according to $\range^G$, for any $G\in\bbG$ as
\begin{align}
    &\pi^\star_G(a|s) 
    = \behavepi(a|s) \omega_G(s) + \I{a = \textstyle\argmax_{a' \in \cA} q^{\pi^\star_G}(s, a')} (1 - \omega_G(s)) \, .
    \label{eq:pi*G definition}
\end{align}

\vspace{-1mm}Notice that for states $s\in\cS_h$ for some $h\in[H]$, $\pi^\star_G(\cdot|s)$ in the above definition depends on the value of $q^{\pi^\star_G}(s', \cdot)$ for some $s'\in\cS_{h+1}$. 
We can therefore interpret the above recursive definition as 
defining $\pi^\star_G(\cdot|s)$ for $s\in\cS_h$, first for $h=H+1$,
then $h=H$, etc., down to $h=1$. 
Every time we define the policy for some stage $h$ in such a way, the policy and $q^{\pi^\star_G}(s', \cdot)$ are already defined on later stages, making the definition valid, and resolving the recursive nature of \cref{eq:pi*G definition}.

\begin{proof}
    $v^\star(\startstate) - v^{\optpi}(\startstate)$ can be decomposed into the following error terms.
    \begin{align*}
        v^\star(\startstate) - v^{\optpi}(\startstate)
        = \underbrace{v^\star(\startstate) - v^{\pi^\star_\trueG}(\startstate)}_{(\RN{1})} 
            + \underbrace{v^{\pi^\star_\trueG}(\startstate) - v_{\opttheta{1}}(\startstate)}_{(\RN{2})}
            + \underbrace{v_{\opttheta{1}}(\startstate) - v^{\optpi}(\startstate)}_{(\RN{3})} \, .
    \end{align*}
    The remainder of the proof focuses on bounding these error terms.
    Following the intuition described in \cref{sec:result}, showing that terms $(\RN{1})$ and $(\RN{2})$ are small can be seen as addressing the first problem of potentially skipping over large-value states, while showing that term $(\RN{3})$ is small can be seen as addressing the second problem of $\optpi$ being greedy w.r.t. to a potentially inaccurate estimates $\opttheta{1:H+1}$.

    \textbf{Bounding} $(\RN{1}) = v^\star(\startstate) - v^{\pi^\star_\trueG}(\startstate)$:
    This term cannot be too large since the $\range^\trueG$ function is approximately correct (\cref{lem:range-G-accurate}), and we only skip over states with low $\range^\trueG$ (\cref{eq:omega}), implying the action we take doesn't affect the value function much (\cref{eq:vstar-vs-range}).
    In \cref{ss:bounding term 1} we formalize this intuition, and show the following result.   
    \begin{align}
        (\RN{1})
        = v^\star(\startstate) - v^{\pi^\star_\trueG}(\startstate)
        &\le H(2 \alpha + 2 \eta) \, .
        \label{eq:term 1 bound}
    \end{align}

    \textbf{Bounding} $(\RN{2}) = v^{\pi^\star_\trueG}(\startstate) - \bar v_{\opttheta{1}}(\startstate)$: 
    This term can be bounded by approximately zero due to \cref{opt:bpse start state} being optimistic from the start state.
    First, note that $v^{\pi^\star_\trueG}$ is approximately equal to $\bar v_{\psi_1(\pi^\star_\trueG)}$ (\cref{ass:approximate q-pi realizability}).
    Then, in \cref{lem:pi*G in Theta} we show that $\psi_1(\pi^\star_\trueG) \in \Theta_{\trueG, 1}$,
    and in \cref{lem:true guess is feasible} we show that $\trueG$ is a feasible solution to \cref{opt:bpse start state}. 
    Since $(\optG, \opttheta{1:H+1})$ is the solution to \cref{opt:bpse start state}, it holds that $\bar v_{\opttheta{1}}(\startstate) \ge \bar v_{\theta}(\startstate)$ for any $\theta \in \Theta_{G, 1}$ where $G$ is a feasible solution to \cref{opt:bpse start state}.
    Thus $\bar v_{\opttheta{1}}(\startstate) \ge v_{\psi_1(\pi^\star_\trueG)}(\startstate)$. 
    In \cref{ss:bounding term 2} we formalize this intuition, and show that with probability at least $1 - \delta$,   
    \begin{align}
        (\RN{2})
        = v^{\pi^\star_\trueG}(\startstate) - \bar v_{\opttheta{1}}(\startstate)
        \le \eta \, .
        \label{eq:term 2 bound}
    \end{align}

    \textbf{Bounding} $(\RN{3}) = \bar v_{\opttheta{1}}(\startstate) - v^{\optpi}(\startstate)$:
    To bound term $(\RN{3})$ we will first show in \cref{lem:Theta elements are close} that value estimates in terms of $\opttheta{h}$ and $\psi_h(\pi^\star_\optG)$ are close with high probability for all $h \in [H+1]$.
    This lemma allows us to relate $\bar v_\opttheta{1}$ to $\bar v_{\psi_1(\pi^\star_\optG)}$ and then \cref{ass:approximate q-pi realizability} relates $\bar v_{\psi_1(\pi^\star_\optG)}$ to $v^{\pi^\star_\optG}$.
    We are then left with relating $v^{\pi^\star_\optG}$ to $v^{\optpi}$.
    To do this we claim that $\optpi$ is an approximate policy improvement step w.r.t. $v^{\pi^\star_\optG}$, which can be seen by recalling that $\optpi$ is greedy w.r.t. $\bar q_{\opttheta{1:H+1}}$, and as we mentioned a couple sentences ago, $\bar v_\opttheta{h}$ and $\bar v_{\psi_h(\pi^\star_\optG)}$ are close for all $h \in [H+1]$. 

    To formalize this intuition we begin by decomposing $\bar v_{\opttheta{1}}(\startstate) - v^{\optpi}(\startstate)$ into the following error terms
    \begin{align}
        \bar v_{\opttheta{1}}(\startstate) - v^{\optpi}(\startstate)
        &= \bar v_{\opttheta{1}}(\startstate) - \bar q_{\psi_1(\pi^\star_\optG)}(\startstate,\optpi(\startstate)) 
            + \bar q_{\psi_1(\pi^\star_\optG)}(\startstate,\optpi(\startstate)) - v^{\optpi}(\startstate) \, .
        \label{eq:term 4 decomposition}
    \end{align}
    To bound $\bar v_{\opttheta{1}}(\startstate) - \bar q_{\psi_1(\pi^\star_\optG)}(\startstate,\optpi(\startstate))$ we introduce a useful lemma (proof in \cref{proof:lem:Theta elements are close}).
    \begin{lemma} \label{lem:Theta elements are close}
        There is an event $\cE_2$, that occurs with probability at least $1 - \delta/3$, such that under event $\cE_2$, 
        for all $G \in \bbG$ that are feasible solutions to \cref{opt:bpse start state}, 
        for all $h \in [H]$, 
        for all $(\theta_{s,a})_{(s,a)\in\cS_h\times\cA}$ and $(\check\theta_{s,a})_{(s,a)\in\cS_h\times\cA} \in \Theta_{G, h}^{\cS_h\times\cA}$,
        and for all admissible distributions $\nu = (\nu_t)_{t \in [H]}$, 
        it holds that 
        \begin{align*}
            &\bigE_{(S, A) \sim \nu_h} \left[\bar q_{\theta_{S, A}}(S, A) - \bar q_{\check \theta_{S, A}}(S, A)\right]
            \le \tilde \eps && \text{$\tilde \eps$ defn \cref{def:tilde-eps}} \, .
        \end{align*}
    \end{lemma}
    To use \cref{lem:Theta elements are close} we must show that $\bar v_{\opttheta{1}}(\startstate) - \bar q_{\psi_1(\pi^\star_\optG)}(\startstate,\optpi(\startstate))$ satisfies its requirements.
    First, note that $\bar v_{\opttheta{1}}(\startstate) = \bar q_{\opttheta{1}}(\startstate,\optpi(\startstate))$, by definition of $\optpi$ (\cref{eq:opttheta and optpi defn}).
    Second, $\optG$ is the solution to \cref{opt:bpse start state}, thus, a feasible solution.
    Third, by \cref{lem:pi*G in Theta}, there is an event $\cE_1$, which occurs with probability at least $1 - \delta/3$, such that under event $\cE_1$, $\psi_1(\pi^\star_\optG) \in \Theta_{\optG, 1}$,
    and by definition $\opttheta{1} \in \Theta_{\optG, 1}$. 
    Let $\nu_h(s, a) = \bPmarg{h}_{\optpi, s_1}(s, a)$ for all $h \in [H], (s, a) \in \cS_h \times \cA$.
    Clearly $\nu = (\nu_h)_{h \in [H]}$ is an admissible distribution by \cref{def:admissible dist}. 
    Thus, under event $\cE_1 \cap \cE_2$, by \cref{lem:Theta elements are close}, 
    \begin{align*}
        \bar v_{\opttheta{1}}(\startstate) - \bar q_{\psi_1(\pi^\star_\optG)}(\startstate,\optpi(\startstate))
        = \bigE_{(S, A) \sim \nu_1} \left[\bar q_{\opttheta{1}}(S, A) - \bar q_{\psi_1(\pi^\star_\optG)}(S, A)\right]
        \le \tilde \eps \, .
    \end{align*}
    
    It is left to bound $\bar q_{\psi_1(\pi^\star_\optG)}(\startstate,\optpi(\startstate)) - v^{\optpi}(\startstate)$ in \cref{eq:term 4 decomposition}.
    To do this, first note that 
    \begin{align*}
        \bar q_{\psi_1(\pi^\star_\optG)}(\startstate,\optpi(\startstate)) - v^{\optpi}(\startstate) 
        &\le q^{\pi^\star_\optG}(\startstate,\optpi(\startstate)) - v^{\optpi}(\startstate) 
            + \eta \, .
    \end{align*}
    where the inequality holds since we have approximate linear $q^\pi$-realizability (\cref{ass:approximate q-pi realizability}).
    To bound $q^{\pi^\star_\optG}(\startstate,\optpi(\startstate)) - v^{\optpi}(\startstate)$ notice that $v^{\optpi}(\startstate) = q^{\optpi}(\startstate,\optpi(\startstate))$, which implies that
    \begin{align*}
        q^{\pi^\star_\optG}(\startstate,\optpi(\startstate)) - v^{\optpi}(\startstate)
        = q^{\pi^\star_\optG}(\startstate,\optpi(\startstate)) - q^{\optpi}(\startstate,\optpi(\startstate))
        = \bigE_{\trajectoryrand \sim \bP_{\optpi, \startstate}} \left[v^{\pi^\star_\optG}(S_2) - v^{\optpi}(S_2)\right] \, .
    \end{align*}
    Next, we give a bound on $\bigE_{\trajectoryrand \sim \bP_{\optpi, \startstate}} \left[v^{\pi^\star_\optG}(S_2) - v^{\optpi}(S_2)\right]$ (proof in \cref{proof:lem:term 4 inductive result}):
    \begin{lemma} \label{lem:term 4 inductive result}
        Under event $\cE_1 \cap \cE_2$, for any $h \in [2:H+1]$, it holds that 
        \begin{align*}
            \bigE_{\trajectoryrand \sim \bP_{\optpi, \startstate}} \left[v^{\pi^\star_\optG}(S_h) - v^{\optpi}(S_h)\right] 
            \le 2(H-h+2) (\eta + \tilde \eps) \, .
        \end{align*}       
    \end{lemma}
    Intuitively, the above lemma holds since $\optpi$ can be thought of as an approximate policy improvement step w.r.t. $v^{\pi^\star_\optG}$. 
    To see this, recall that $\optpi$ is greedy w.r.t. $\bar q_{\opttheta{1:H+1}}$ (\cref{eq:opttheta and optpi defn}).
    Then, with \cref{lem:Theta elements are close}, we can show $\bar v_\opttheta{h}$ and $\bar v_{\psi_h(\pi^\star_\optG)}$(which is close to $v^{\pi^\star_\optG}$ (\cref{ass:approximate q-pi realizability})) are close for all $h \in [H+1]$. 
    The above bounds imply that under event $\cE_1 \cap \cE_2$, which occurs with probability at least $1 - 2\delta/3$,
    \begin{align}
        (\RN{3})
        = \bar v_{\opttheta{1}}(\startstate) - v^{\optpi}(\startstate)
        \le 2H (\eta + \tilde \eps) + \tilde \eps + \eta \, . 
        \label{eq:term 3 bound}
    \end{align}

    \textbf{Combining the Bounds:}
    To finish the proof we combine the bounds on all three terms (\cref{eq:term 1 bound,eq:term 2 bound,eq:term 3 bound}), 
    to get that under event $\cE_1 \cap \cE_2 \cap \cE_3$, 
    which occurs with probability at least $1 - \delta$, 
    \begin{align*}
         v^\star(\startstate) - v^{\optpi}(\startstate)       
         \le H(2\alpha + 2 \eta) + \eta + 2H (\eta + \tilde \eps) + \tilde \eps + \eta
         \le 4 (H+1) (\alpha + \eta + \tilde \eps) \, .
    \end{align*}
    To bound the above display by $\eps$ we set $\alpha = \eps / (12(H+1)) < 1$.
    If $n = \tilde \Theta\left(\conc^4 H^7 d^4/\eps^2\right)$ and
    $\eta = \bigOt{\alpha/\sqrt{nH}}$ (\cref{def:eta}), we show that $\tilde \eps = \tilde \cO\left(\conc^2 H^{5/2} d^2/\sqrt{n}\right)$ (\cref{eq:tilde-eps definition}).
    This implies that 
    \begin{align*}
        4 (H+1) (\alpha + \eta + \tilde \eps)
        \le \eps \, . 
        & 
        \qedhere
    \end{align*}
\end{proof}

\section{Limitations and Conclusions} \label{sec:future work}

In this work we resolved an open problem in the positive, by presenting the first statistically efficient learner (\cref{ss:method}) that outputs a near optimal policy in the offline RL setting with approximate linear $q^\pi$-realizability (\cref{ass:approximate q-pi realizability}), trajectory data (\cref{ass:batch data}), and concentrability (\cref{ass:concentrability}).
One limitation of this work is that we are not aware of any computationally efficient implementation of \cref{opt:bpse start state}, which is at the heart of our learner.
As such, it is left as an open problem whether computationally efficient learning is possible in the setting we considered.
Another limitation is that we are not sure if our statistical rate in \cref{thm:bpse} is optimal.
Showing a matching lower bound or improving the rate is left for future work.

Another limitation of our work originates from our setting underpinning our result (\cref{sec:result}), namely the three assumptions: approximate linear $q^\pi$-realizability, trajectory data, and concentrability.  
Approximate linear $q^\pi$-realizability requires the value function of all memoryless policies to be linear in a fixed and known $d$-dimensional feature map. 
While strictly weaker than the linear MDP assumption \citep{zanette2020learning},
this assumption is still strong.
Trajectory data requires full sequences of interactions with an environment to be collected by a single policy. 
For long horizon problems this can be practically challenging.
Concentrability requires the state and action spaces to be well-covered.
This can be challenging to guarantee since often the state and action spaces are unknown at the time of data collection.
Further, since we require the trajectory data to be collected by a single policy, it may be the case that no single policy exists that covers the state and action spaces well, and a mixture of policies must be considered, which our current result does not immediately hold for.
Although the assumptions appear strong, a justification for them is that under many variations of weaker assumptions (for instance: general data, or linear $q^\pi$-realizability of only one policy, or only coverage of the feature space), polynomial statistical rates have been shown to be impossible to achieve by any learner (\cref{tab:bpse}).

Since this work is focused on foundational theoretical research it is unlikely to have any direct and immediate
societal impacts.

\newpage

\bibliographystyle{abbrvnat}
\bibliography{references}

\newpage
\appendix

\section*{Appendix}

\section{Parameter Settings and Notation} \label{sec:parameter settings}

\begin{align}
    n 
    &= \tilde \Theta\left(\frac{\conc^4 H^7 d^4}{\eps^2}\right) 
    && \text{Set at the end of \cref{sec:proof of the result}} 
    \label{def:n} \\
    d_0 
    &= \lceil 4d\log\log(d)+16 \rceil
    && \text{Defined above \cref{eq:bbG-def}} 
    \label{def:d_0} \\
    \featurebound 
    &= \text{Upper bound on $2$-norm of features $\phi$} 
    && \text{Defined above \cref{ass:approximate q-pi realizability}} 
    \label{def:feature bound} \\
    \thetabound 
    &= \text{Upper bound on $2$-norm of true parameters $\psi_h, h \in [H]$} 
    && \text{\cref{ass:approximate q-pi realizability}} 
    \label{def:parameter bound} \\
    \modthetabound 
    &= \thetabound(8H^2d_0/\alpha+1) 
    && \text{Defined in \cref{lem:approx linear MDP}} 
    \label{def:mod parameter bound} \\
    \sqrt{\lambda}
    &= H^{3/2} d/\modthetabound
    && \text{Defined in \cref{eq:lambda definition}} 
    \label{def:lambda} \\
    \misspec 
    &\le \frac{H^{3/2} d}{\sqrt{n} \left(10 H^2 d_0/\alpha + 1\right)} = \tilde \Theta\left(\frac{\alpha}{\sqrt{nH}}\right) = \tilde \Theta\left(\frac{\eps^2}{\conc^2 H^5 d^2}\right)
    && \text{Defined in \cref{eq:eta definition}} 
    \label{def:eta} \\
    \modmisspec 
    &= \eta(10H^2d_0/\alpha+1) = \tilde \cO\left(\frac{H^{3/2} d}{\sqrt{n}}\right)
    && \text{Defined in \cref{lem:approx linear MDP}} 
    \label{def:modmisspec} \\
    \check \eps 
    &= \tilde \cO\left(d / \sqrt{n}\right) 
    && \text{Defined in \cref{eq:check-eps definition}}
    \label{def:check-eps} \\
    \bar \eps 
    &= \tilde \cO\left(\frac{\conc H^{5/2} d^2}{\sqrt{n}}\right)
    && \text{Defined in \cref{eq:bar-eps definition}} 
    \label{def:bar-eps} \\
    \tilde \eps 
    &= \tilde \cO\left(\frac{\conc^2 H^{5/2} d^2}{\sqrt{n}}\right) 
    && \text{Defined in \cref{eq:tilde-eps definition}}
    \label{def:tilde-eps} \\
    \bar \beta 
    &= \tilde \cO \left(H^{3/2} d\right)
    && \text{Defined in \cref{eq:bar-beta definition}} 
    \label{def:bar-beta} \\
    \beta 
    &= \tilde \cO \left(H^{3/2} d\right)
    && \text{Defined in \cref{eq:beta definition}} 
    \label{def:beta} \\
    |C_\xi^\bbG| 
    &\le (1 + 2\usedtobesqrtdoneplusone/\xi))^{dHd_0}, \, \xi > 0
    && \text{Defined in \cref{lem:G cover useful results}} 
    \label{def:G cover size} \\
    \modxi 
    &= 12 \sqrt{2d} H^2 \featurebound \xi \alpha^{-1} \left(2\sqrt{\nh}\featurebound\modthetabound / (H^{3/2}d)\right)^{H}, \, \xi > 0 
    && \text{Defined in \cref{eq:bar-xi definition}} 
    \label{def:modxi} \\
    \alpha 
    &= \frac{\eps}{12(H+1)} < 1
    && \text{Defined at the end of \cref{sec:proof of the result}} 
    \label{def:alpha}
\end{align}

\newpage
\section{Proof of \cref{lem:approx linear MDP}} 

\begin{proof} \label{proof:lem:approx linear MDP}
We follow a proof technique introduced in \citep{weisz2023online}.
We start by quoting their definition of admissible functions and their admissible realizability lemma.

\begin{definition}[Definition 4.6 in \citep{weisz2023online}] \label{def:alpha-admissible}
For any $h\in[H]$,
$f:\cS_h\to\R$ is $\alpha'$-admissible for some $\alpha'>0$ if for all $s\in\cS_h$,
$|f(s)|\le \range(s)/\alpha'$.
\end{definition}

\begin{lemma}[Admissible-realizability (Lemma 4.7 in \citep{weisz2023online})]\label{lem:admissible-realizability}
If $f:\cS_h\to\R$ is $\alpha'$-admissible
then it
is realizable, that is, for all $t\in[h-1]$
and $\pi\in\Pi$,
there exists some $\tilde\theta\in\R^d$ with $\norm{\tilde\theta}_2\le 4d_0\thetabound/\alpha'$ such that for all $(s,a)\in\cS_t\times\cA$,
\[
\abs{
\bigE_{\trajectoryrand \sim \bP_{\pi, s, a}} f(S_h) - \ip{\phi(s,a), \tilde\theta}} \le \eta_0\quad\quad\quad\quad\text{where ${\eta_0}=5d_0\eta/\alpha'$.} %
\]
\end{lemma}

Next, we fix some $f:\cS\to[0,H]$ with $f(\termstate)=0$,
    and policy $\pi \in \Pi$.
For $2\le h \le H+1$, define $\tilde g_h : \cS_h\to [-H,H]$
as 
$\tilde g_{H+1}(\cdot)=0$, and
\begin{align*}
\tilde g_h(s)
&=
\bigE_{\trajectoryrand \sim \bP_{\pi,s}} 
\bigE_{\tau \sim F_{\trueG, \trajectoryrand, h}} \left[-R_{\tau:H} + f(S_\tau)\right] \\
&=
\bigE_{\trajectoryrand \sim \bP_{\pi,s}} \sum_{t=h}^H \left[-R_{t:H} + f(S_t)\right] \left(1-\omega_{\trueG}(S_t)\right) \prod_{u=h}^{t-1} \omega_{\trueG}(S_u) \, .
\end{align*}
Notice that for any $h\in[H]$, $(s,a)\in\cS_h\times\cA$,
the function of $(s,a)$ we aim to linearly realize can be written as
\begin{align}
\label{eq:target-into-q-and-tilde-g}
\begin{split}
\bigE_{\trajectoryrand \sim \bP_{\pi, s, a}} \bigE_{\tau \sim F_{\trueG, \trajectoryrand, h}} \left[R_{h:\tau-1} + f(S_\tau)\right] 
&=
\bigE_{\trajectoryrand \sim \bP_{\pi, s, a}} \tilde g_{h+1}(S_{h+1}) + R_{h:H}\\
&=
\bigE_{S_{h+1} \sim P(s, a)} \tilde g_{h+1}(S_{h+1}) + q^\pi(s,a) \, .
\end{split}
\end{align}
The second term of the sum, $q^\pi(s,a)$ is linearly realizable by \cref{ass:approximate q-pi realizability} with parameter $\psi_h(\pi)$.
The first term needs more work before \cref{lem:admissible-realizability} can be applied.
To this end, for $h\in[2:H]$ we define $g_h : \cS_h\to \R$ %
as \[
g_h(s) = \left(1-\omega_{\trueG}(s)\right)
\bigE_{\trajectoryrand \sim \bP_{\pi, s}} \left[-R_{h:H} + f(s) - \tilde g_{h+1}(S_{h+1})\right] \, .
\]
Notice that $\tilde g_h$ can be decomposed into a sum of $g_t$ functions as for all $h\in[2:H]$, $s\in\cS_h$,
\begin{align}\label{eq:tilde-g-decompose}
\tilde g_h(s) = 
\bigE_{\trajectoryrand \sim \bP_{\pi, s}} \sum_{t=h}^H g_t(S_t) \, .
\end{align}
The benefit of decomposing $\tilde g_h$ into $g_t$ functions is that $g_t$ are $\alpha'$-admissible under \cref{def:alpha-admissible} for $\alpha'=\alpha/(2H)$.
To see this, 
note that for any trajectory and $s$, $-R_{h:H} + f(s) - \tilde g_{h+1}(S_{h+1})\in[-2H,2H]$.
$g_t(s)$ multiplies this by $1-\omega_{\trueG}(s)$ which by \cref{eq:omega} is between $0$ and $1$, and satisfies $1-\omega_{\trueG}(s)=0$ if $\range^\trueG(s) \le \alpha/\sqrt{2d}$.
By \cref{lem:range-G-accurate},
$\range(s)\le \sqrt{2d} \cdot \range^\trueG(s)$, so $g_t(s)=0$ for any $s$ with $\range(s)\le \alpha$. On any other $s$, $|g_t(s)|\le 2H$. Therefore $g_t$ is $\alpha'$-admissible.
This allows us to use \cref{lem:admissible-realizability} to get that for any $h\in[H-1]$, there exist $\tilde \theta_{h+1:H}\in\B(8Hd_0\thetabound/\alpha)^{H-h}$, 
such that for any stage $t\in[h+1:H]$,
for all $(s,a)\in\cS_h\times\cA$,
\[
\abs{\bigE_{\trajectoryrand \sim \bP_{\pi, s, a}}
g_t(S_t)-\ip{\phi(s,a), \tilde\theta_t}} \le 10Hd_0\eta/\alpha \, .
\]
By combining this with \cref{eq:tilde-g-decompose}, for all $h\in[H]$ there exists $\tilde \theta = \sum_{t=h+1}^H \tilde \theta_t$ with $\tilde \theta\in\B(8H^2d_0\thetabound/\alpha)$ such that for all $(s,a)\in\cS_h\times\cA$,
\[
\abs{\bigE_{S_{h+1}\sim P(s,a)}\tilde g_{h+1}(S_{h+1})-\ip{\phi(s,a), \tilde\theta}} \le 10H^2d_0\eta/\alpha \, .
\]
Combined with \cref{eq:target-into-q-and-tilde-g} and the parameter $\psi_h(\pi)$ from \cref{ass:approximate q-pi realizability}, there exists $\theta=\tilde\theta+\psi_h(\pi)$ with $\theta\in\B(\thetabound(8H^2d_0/\alpha+1))$ such that for all $(s,a)\in\cS_h\times\cA$,
\[
\abs{\bigE_{\trajectoryrand \sim \bP_{\pi, s, a}} \bigE_{\tau \sim F_{\trueG, \trajectoryrand, h}} \left[R_{h:\tau-1} + f(S_\tau)\right] 
-\ip{\phi(s,a), \theta}} \le \eta(10H^2d_0/\alpha+1) = \modmisspec \, .
\]
To finish the proof, we define $\rho_h^\pi(f)=\theta$ for the arbitrary $h\in[H]$, $\pi$, and $f$ picked above.
\end{proof}

\newpage
\section{Results used in \cref{sec:proof of the result}} \label{sec:results used in proof}

\subsection{Bounding term $(\RN{1})$} \label{ss:bounding term 1}
    We begin by defining an alternative policy $\pi^\dag_\trueG$ as
    \begin{align*}
        \pi^\dag_\trueG(a|s) 
        = \behavepi(a | s) \omega_\trueG(s) + \pi^\star(a | s) (1 - \omega_\trueG(s)) \, .
    \end{align*}
    This policy can only be worse in value than $\pi^\star_\trueG$:
    \begin{lemma}\label{lem:pidag-worse}
        For all $s\in\cS$, 
        \[
        v^{\pi^\star_\trueG}(s) \ge v^{\pi^\dag_\trueG}(s) \, .
        \]
    \end{lemma}
    \begin{proof}
        \label{proof:lem:pidag-worse}
        We prove by induction for $h=H+1,H,\dots,1$ that for all $s\in\cS_h$,
        $v^{\pi^\star_\trueG}(s) \ge v^{\pi^\dag_\trueG}(s)$. The base case of $h=H+1$ is immediately true by definition as $v$-values are $0$ on $\termstate$, regardless of the policy.
        Assuming the inductive hypothesis holds for $h+1$, we continue by proving it for $h$.
        Let $(s,a)\in\cS_h\times\cA$ be arbitrary.
        Notice that 
        \[
        q^{\pi^\star_\trueG}(s,a) - q^{\pi^\dag_\trueG}(s,a) = \bigE_{S' \sim P(s,a)} v^{\pi^\star_\trueG}(S') - v^{\pi^\dag_\trueG}(S') \ge 0 \, ,
        \]
        where the inequality is due to the inductive hypothesis.
        Next, for any $s\in\cS_h$, by the above and the definition of the policies,
        \begin{align*}
        v^{\pi^\star_\trueG}(s)
        &=
        \omega_\trueG(s) \bigE_{A\sim \behavepi(s)} q^{\pi^\star_\trueG}(s,A) + (1-\omega_\trueG(s)) \max_{a\in\cA} q^{\pi^\star_\trueG}(s,a) \\
        &\ge \omega_\trueG(s) \bigE_{A\sim \behavepi(s)} q^{\pi^\dag_\trueG}(s,A) + (1-\omega_\trueG(s)) \max_{a\in\cA} q^{\pi^\dag_\trueG}(s,a) \\
        &\ge \omega_\trueG(s) \bigE_{A\sim \behavepi(s)} q^{\pi^\dag_\trueG}(s,A) + (1-\omega_\trueG(s)) \bigE_{A\sim \pi^\star_\trueG(s)} q^{\pi^\dag_\trueG}(s,A) 
        = 
        v^{\pi^\dag_\trueG}(s)\,,
        \end{align*}
        finishing the induction.
    \end{proof}
    Due to \cref{lem:pidag-worse}, $v^\star(\startstate) - v^{\pi^\star_\trueG}(\startstate) \le v^\star(\startstate) - v^{\pi^\dag_\trueG}(\startstate)$. We continue by bounding $v^\star(\startstate) - v^{\pi^\dag_\trueG}(\startstate)$.    
    We first decompose it using the performance difference lemma (\cref{lem:performance difference lemma}) to get that
    \begin{align}
        v^\star(\startstate) - v^{\pi^\dag_\trueG}(\startstate)
        = \sum_{h=2}^H \bigE_{(S_h, A_h) \sim \bPmarg{h}_{\pi^\star, \startstate}} \left(q^{\pi^\dag_\trueG}(S_h, A_h) - v^{\pi^\dag_\trueG}(S)\right) \, ,
        \label{eq:pdl decomposition}
    \end{align}
    where we only have the sum from $h=2$ to $H$, since for $h=1$ and $h=H+1$, for any $s\in\cS_h$, $\omega_\trueG(s)=0$ and therefore $\pi^\dag_\trueG(s)=\pi^\star(s)$.
    By the definition of $\omega_\trueG$ (\cref{eq:omega}), we can see that for $s \in \cS, \omega_\trueG(s) \neq 0$ only when $\range^\trueG(s) \le 2\alpha/\sqrt{2d}$ 
    Thus, the policies $\pi^\star(s)$ and $\pi^\dag_\trueG(s)$ are equal for all $s \in \cS$ that satisfy $\range^\trueG(s) \ge 2\alpha/\sqrt{2d}$.
    Making use of this result in \cref{eq:pdl decomposition},
    we get that
    \begin{align*}
        &\sum_{h=2}^H \bigE_{(S_h, A_h) \sim \bPmarg{h}_{\pi^\star, \startstate}} \left(q^{\pi^\dag_\trueG}(S_h, A_h) - v^{\pi^\dag_\trueG}(S)\right) \\
        &= \sum_{h=2}^H \bigE_{(S_h, A_h) \sim \bPmarg{h}_{\pi^\star, \startstate}} \left[\I{\range^\trueG(S_h) \le 2\alpha/\sqrt{2d}} \left(q^{\pi^\dag_\trueG}(S_h, A_h) - v^{\pi^\dag_\trueG}(S_h)\right)\right] \, .
    \end{align*}
    By \cref{eq:vstar-vs-range}, we know that
    \begin{align*}
        q^{\pi^\dag_\trueG}(s, a) - v^{\pi^\dag_\trueG}(s) 
        \le \range(s) + 2\eta \, .
    \end{align*}
    Then, we can use \cref{lem:range-G-accurate} to get that for all $h\in[2:H]$ and $(s, a) \in \cS_h \times \cA$ 
    \begin{align*}
        q^{\pi^\dag_\trueG}(s, a) - v^{\pi^\dag_\trueG}(s) 
        \le \sqrt{2d} \cdot \range^\trueG(s) + 2\eta \, .
    \end{align*}
    Putting things together we get the following bound.
    \begin{align*}
        (\RN{1})
        &\le v^\star(\startstate) - v^{\pi^\star_\trueG}(\startstate) \\
        &\le v^\star(\startstate) - v^{\pi^\dag_\trueG}(\startstate) \\
        &= \sum_{h=2}^H \bigE_{(S_h, A_h) \sim \bPmarg{h}_{\pi^\star, \startstate}} \left[\I{\range^\trueG(S) \le 2\alpha/\sqrt{2d}} \left(q^{\pi^\dag_\trueG}(S, A) - v^{\pi^\dag_\trueG}(S)\right)\right] \nonumber \\
        &\le H(2 \alpha + 2 \eta) \, .
    \end{align*}

\subsection{Bounding term $(\RN{2})$} \label{ss:bounding term 2}
    To bound $v^{\pi^\star_\trueG}(\startstate) - \bar v_{\opttheta{1}}(\startstate)$ we decompose it into the following error terms
    \begin{align}
        v^{\pi^\star_\trueG}(\startstate) - \bar v_{\opttheta{1}}(\startstate)
        &= v^{\pi^\star_\trueG}(\startstate) - \bar v_{\psi_1(\pi^\star_\trueG)}(\startstate) 
            + \bar v_{\psi_1(\pi^\star_\trueG)}(\startstate) - \bar v_{\opttheta{1}}(\startstate) \nonumber \\
        &\le \bar v_{\psi_1(\pi^\star_\trueG)}(\startstate) - \bar v_{\opttheta{1}}(\startstate)
            + \eta \, ,
        \label{eq:term 2 decomposition}
    \end{align}
    where the inequality holds since we have approximate linear $q^\pi$-realizability (\cref{ass:approximate q-pi realizability}), which implies that 
    \begin{align*}
        v^{\pi^\star_\trueG}(\startstate) - \bar v_{\psi_1(\pi^\star_\trueG)}(\startstate)
        \le \max_{a \in \cA} \left(q^{\pi^\star_\trueG}(\startstate, a) - \bar q_{\psi_1(\pi^\star_\trueG)}(\startstate, a)\right)
        \le \eta \, .
    \end{align*}
    To help us bound $\bar v_{\psi_1(\pi^\star_\trueG)}(\startstate) - \bar v_{\opttheta{1}}(\startstate)$ in \cref{eq:term 2 decomposition} we make use of two lemmas.
    The first is the following (proof in \cref{proof:lem:pi*G in Theta}).
    \begin{lemma} \label{lem:pi*G in Theta}
        There is an event $\cE_1$, which occurs with probability at least $1- \delta/3$, such that under $\cE_1$, for all $G \in \bbG$ and $h \in [H+1]$, it holds that $\psi_h(\pi^\star_G) \in \Theta_{G, h}$.
    \end{lemma}
    \cref{lem:pi*G in Theta} tells us that under event $\cE_1$, $\psi_1(\pi^\star_\trueG) \in \Theta_{\trueG, 1}$.
    The second lemma is the following (proof in \cref{proof:lem:true guess is feasible}). 
    \begin{lemma}[Feasibility] \label{lem:true guess is feasible}
        There is an event $\cE_1 \cap \cE_2 \cap \cE_3$, which occurs with probability at least $1 - \delta$, such that under event $\cE_1 \cap \cE_2 \cap \cE_3$, the true guess $\trueG$ is a feasible solution to \cref{opt:bpse start state}.
    \end{lemma} 
    Notice that since $(\optG, \opttheta{1:H+1})$ is the solution to \cref{opt:bpse start state}, it holds that $\bar v_{\opttheta{1}}(\startstate) \ge \bar v_{\theta}(\startstate)$ for any $\theta \in \Theta_{G, 1}$ where $G$ is a feasible solution to \cref{opt:bpse start state}.
    Thus, we get that under event $\cE_1 \cap \cE_2 \cap \cE_3$, since $\psi_1(\pi^\star_\trueG) \in \Theta_{\trueG, 1}$ (by \cref{lem:true guess is feasible}),
    and $\trueG$ is a feasible solution to \cref{opt:bpse start state} (by \cref{lem:true guess is feasible}),
    it holds that 
    \begin{align*}
        \bar v_{\psi_1(\pi^\star_\trueG)}(\startstate) - \bar v_{\opttheta{1}}(\startstate)
        \le \eta \, ,
    \end{align*}
    which together with \cref{eq:term 2 decomposition}, implies that
    \begin{align*}
        (\RN{2})
        = v^{\pi^\star_\trueG}(\startstate) - \bar v_{\opttheta{1}}(\startstate)
        \le \eta \, .
    \end{align*}

\subsection{Proof of \cref{lem:term 4 inductive result}}

\begin{proof} \label{proof:lem:term 4 inductive result}
Recall that $(\optG, \opttheta{1:H+1})$ is the solution to \cref{opt:bpse start state}.
    We aim to show that under event $\cE_1 \cap \cE_2$, for any $h \in [2:H+1]$, it holds that
        \begin{align}
            \bigE_{\trajectoryrand \sim \bP_{\optpi, \startstate}} \left[v^{\pi^\star_\optG}(S_h) - v^{\optpi}(S_h)\right] 
            \le 2(H-h+2) (\eta + \tilde \eps) \, .
            \label{eq:term 4 inductive hypothesis}
        \end{align}       
    To prove \cref{eq:term 4 inductive hypothesis} we will use induction.
    The base case is when $h = H+1$, which trivially holds, since $v^\pi(\termstate) = r_{H+1}(s, a) = 0$ for all $\pi \in \Pi, \termstate \in \termstatespace, a \in \cA$.
    
    Now, we show the inductive step. 
    Let $h \in [2:H]$ be arbitrary.
    Assume that \cref{eq:term 4 inductive hypothesis} holds for any $t \in [h+1:H+1]$.
    We prove that \cref{eq:term 4 inductive hypothesis} also holds for $h$.
    For any $(s, a) \in (\cS \backslash \cS_1) \times \cA$, let
    \begin{align*}
        \tilde \pi_h(a|s) = 
        \begin{cases}
            \pi^\star_\optG(a|s) & \text{if } \stage(s) = h \\
           \optpi(a|s) & \text{if } \stage(s) \neq h \, .
        \end{cases}
    \end{align*}
    Then
    \begin{align}
        &\bigE_{\trajectoryrand \sim \bP_{\optpi, \startstate}} \left[v^{\pi^\star_\optG}(S_h) - v^{\optpi}(S_h)\right] \nonumber \\
        &= \bigE_{\trajectoryrand \sim \bP_{\tilde \pi_h, \startstate}} q^{\pi^\star_\optG}(S_h, A_h) - \bigE_{\trajectoryrand \sim \bP_{\optpi, \startstate}} q^{\optpi}(S_h, A_h) \nonumber \\
        &= \bigE_{\trajectoryrand \sim \bP_{\tilde \pi_h, \startstate}} \left[q^{\pi^\star_\optG}(S_h, A_h) - \bar q_{\opttheta{h}}(S_h, A_h)\right] 
            + \bigE_{\trajectoryrand \sim \bP_{\tilde \pi_h, \startstate}} \bar q_{\opttheta{h}}(S_h, A_h) - \bigE_{\trajectoryrand \sim \bP_{\optpi, \startstate}} \bar q_{\opttheta{h}}(S_h, A_h) \nonumber \\ 
        &\quad+ \bigE_{\trajectoryrand \sim \bP_{\optpi, \startstate}} \left[\bar q_{\opttheta{h}}(S_h, A_h) - q^{\pi^\star_\optG}(S_h, A_h)\right] 
            + \bigE_{\trajectoryrand \sim \bP_{\optpi, \startstate}} \left[q^{\pi^\star_\optG}(S_h, A_h) - q^{\optpi}(S_h, A_h)\right] \nonumber \\
        &\le \bigE_{\trajectoryrand \sim \bP_{\tilde \pi_h, \startstate}} \left[\bar q_{\psi_h(\pi^\star_\optG)}(S_h, A_h) - \bar q_{\opttheta{h}}(S_h, A_h)\right] 
            + \bigE_{\trajectoryrand \sim \bP_{\tilde \pi_h, \startstate}} \bar q_{\opttheta{h}}(S_h, A_h) - \bigE_{\trajectoryrand \sim \bP_{\optpi, \startstate}} \bar q_{\opttheta{h}}(S_h, A_h) \nonumber \\ 
        &\quad+ \bigE_{\trajectoryrand \sim \bP_{\optpi, \startstate}} \left[\bar q_{\opttheta{h}}(S_h, A_h) - \bar q_{\psi_h(\pi^\star_\optG)}(S_h, A_h)\right] 
            + \bigE_{\trajectoryrand \sim \bP_{\optpi, \startstate}} \left[q^{\pi^\star_\optG}(S_h, A_h) - q^{\optpi}(S_h, A_h)\right]
            + 2\eta \, ,
            \label{eq:term 4 induction decomposition}
    \end{align}
    where the inequality holds since we have approximate linear $q^\pi$-realizability (\cref{ass:approximate q-pi realizability}).
    To bound the first and third error terms above notice that under event $\cE_1$, by \cref{lem:pi*G in Theta}, we know that $\psi_h(\pi^\star_\optG) \in \Theta_{\optG, h}$.
    We also know that $\opttheta{h} \in \Theta_{\optG, h}$, by definition.
    Let $\tilde \nu_u(s_u, a_u) = \bPmarg{u}_{\tilde \pi_h, \startstate}(s_u, a_u)$ and $\nu^\prime_u(s_u, a_u) = \bPmarg{u}_{\optpi, \startstate}(s_u, a_u)$ for all $u \in [H], (s_u, a_u) \in \cS_u \times \cA$. 
    Clearly, $\tilde \nu = (\tilde \nu_u)_{u \in [H]}$ and $\nu^\prime = (\nu^\prime_u)_{u \in [H]}$ are admissible distributions, by \cref{def:admissible dist}.
    Notice that we have satisfied all the conditions to make use of \cref{lem:Theta elements are close}.
    Thus, under event $\cE_1 \cap \cE_2$ it holds that
    \begin{align*}
        \bigE_{\trajectoryrand \sim \bP_{\tilde \pi_h, \startstate}} \left[\bar q_{\psi_h(\pi^\star_\optG)}(S_h, A_h) - \bar q_{\opttheta{h}}(S_h, A_h)\right]
        &= \bigE_{(S_h, A_h) \sim \tilde \nu_h} \left[\bar q_{\psi_h(\pi^\star_\optG)}(S_h, A_h) - \bar q_{\opttheta{h}}(S_h, A_h)\right] 
        \le \tilde \eps \, ,
    \end{align*}
    and
    \begin{align*}
        \bigE_{\trajectoryrand \sim \bP_{\optpi, \startstate}} \left[\bar q_{\opttheta{h}}(S_h, A_h) - \bar q_{\psi_h(\pi^\star_\optG)}(S_h, A_h)\right] 
        &= \bigE_{(S_h, A_h) \sim \nu^\prime_h} \left[\bar q_{\opttheta{h}}(S_h, A_h) - \bar q_{\psi_h(\pi^\star_\optG)}(S_h, A_h)\right]  
        \le \tilde \eps \, .
    \end{align*}
    The term $\bigE_{\trajectoryrand \sim \bP_{\tilde \pi_h, \startstate}} \bar q_{\opttheta{h}}(S_h, A_h) - \bigE_{\trajectoryrand \sim \bP_{\optpi, \startstate}} \bar q_{\opttheta{h}}(S_h, A_h)$ in \cref{eq:term 4 induction decomposition} can be bounded by recalling the definition of $\optpi(s)$ (\cref{eq:opttheta and optpi defn}), to get that
    \begin{align*}
        \bigE_{\trajectoryrand \sim \bP_{\tilde \pi_h, \startstate}} \bar q_{\opttheta{h}}(S_h, A_h) 
        \le \max_{a \in \cA} \bar q_{\opttheta{h}}(S_h, a) 
        = \bigE_{\trajectoryrand \sim \bP_{\optpi, \startstate}} \bar q_{\opttheta{h}}(S_h, A_h) \, .
    \end{align*}
    Under event $\cE_1 \cap \cE_2$, after plugging the above bounds into \cref{eq:term 4 induction decomposition}, we have that
    \begin{align*}
        \bigE_{\trajectoryrand \sim \bP_{\optpi, \startstate}} \left[v^{\pi^\star_\optG}(S_h) - v^{\optpi}(S_h)\right] 
        &\le \bigE_{\trajectoryrand \sim \bP_{\optpi, \startstate}} \left[q^{\pi^\star_\optG}(S_h, A_h) - q^{\optpi}(S_h, A_h)\right] + 2\eta + 2\tilde \eps  \\
        &= \bigE_{\trajectoryrand \sim \bP_{\optpi, \startstate}} \left[v^{\pi^\star_\optG}(S_{h+1}) - v^{\optpi}(S_{h+1})\right] + 2\eta + 2\tilde \eps \\
        &\le 2(H-h+2) (\eta + \tilde \eps) \, ,
    \end{align*}
    where the last inequality holds by the inductive hypothesis for $h+1$ (\cref{eq:term 4 inductive hypothesis}), completing the proof of the claim.
\end{proof}

\newpage
\section{Proof of \cref{lem:pi*G in Theta}} 

\begin{proof} \label{proof:lem:pi*G in Theta}
    
    We will prove the claim using induction.
    The base case is when $h = H+1$, for which $\psi_{H+1}(\pi) = \vec 0$ for all $\pi \in \Pi$ by definition. %
    Thus, for all $G \in \bbG$, it holds that $\psi_{H+1}(\pi^\star_G) \in \Theta_{G, H+1} = \{\vec 0\}$.

    Now, we show the inductive step.
    Let $h \in [H]$ be arbitrary.
    Assume \cref{lem:pi*G in Theta} holds for any $t \in [h+1:H+1]$.
    We prove that it also holds for $h$.
    Define 
    \begin{align*}
        \hat \psi_h(\pi^\star_G)
        = X_h^{-1} \sum_{j \in [\nh]} \phi^j_h \bigE_{\tau \sim F^j_{G, h+1}}\left[r^j_{h:\tau-1} + \bar v_{\psi_{h+1:H+1}(\pi^\star_G)} \left(s^j_{\tau}\right)\right] \, .
    \end{align*}
    By the inductive hypothesis we know that for any $G \in \bbG$
    \begin{align*}
        \psi_{h+1:H+1}(\pi^\star_G) 
        \in \Theta_{G, h+1} \times \dots \times \Theta_{G, H+1}
    \end{align*}
    Thus, $\hat \psi_h(\pi^\star_G) \in \hat \Theta_{G, h}$.
    It is left to show that $\norm{\hat \psi_h(\pi^\star_G) - \psi_h(\pi^\star_G)}_{X_h} \le \beta$, which together with the fact that $\psi_h(\pi^\star_G) \in \cB(\thetabound)$, implies the desired result, $\psi_h(\pi^\star_G) \in \Theta_{G, h}$. 

    We would like to make use of \cref{lem:least squared error decomposition} to bound $\norm{\hat \psi_h(\pi^\star_G) - \psi_h(\pi^\star_G)}_{X_h}$.
    To do so, we map the terms used in the \cref{lem:least squared error decomposition} to our terms as follows
    \begin{align*}
        &n = \nh, \lambda = \lambda, \theta_{\star} = \psi_h(\pi^\star_G), V = X_h, \hat \theta = \hat \psi_h(\pi^\star_G), A = \left(\phi_h^j\right)_{j \in [\nh]} \, , \\
        &Y = \tilde Y + \Delta = \left(\ip{\phi_h^j, \psi_h(\pi^\star_G)}\right)_{j \in [\nh]} + \gamma + \Delta = \left(\bigE_{\tau \sim F^j_{G, h+1}}\left[r^j_{h:\tau-1} + \bar v_{\psi_{h+1:H+1}(\pi^\star_G)} \left(s^j_{\tau}\right)\right]\right)_{j \in [\nh]} \, , \\
        &\gamma = \left(\bigE_{\tau \sim F^j_{G, h+1}}\left[r^j_{h:\tau-1} + \bar v_{\psi_{h+1:H+1}(\pi^\star_G)} \left(s^j_{\tau}\right)\right] 
            -  \bigE_{\trajectoryrand \sim \bP_{\behavepi, s_h^j, a_h^j}} \bigE_{\tau \sim F_{G, \trajectoryrand, h+1}} \left[R_{h:\tau-1} + \bar v_{\psi_{h+1:H+1}(\pi^\star_G)} (S_\tau)\right]\right)_{j \in [\nh]} \, , \\
        &\Delta = \left(\bigE_{\trajectoryrand \sim \bP_{\behavepi, s_h^j, a_h^j}} \bigE_{\tau \sim F_{G, \trajectoryrand, h+1}} \left[R_{h:\tau-1} + \bar v_{\psi_{h+1:H+1}(\pi^\star_G)} (S_\tau)\right] 
            - \ip{\phi_h^j, \psi_h(\pi^\star_G)}\right)_{j \in [\nh]}  \, , \\ 
        &\iota 
        = \sum_{j \in [\nh]} \phi_h^j \left(\bigE_{\tau \sim F^j_{G, h+1}}\left[r^j_{h:\tau-1} + \bar v_{\psi_{h+1:H+1}(\pi^\star_G)} \left(s^j_{\tau}\right)\right] 
            -  \bigE_{\trajectoryrand \sim \bP_{\behavepi, s_h^j, a_h^j}} \bigE_{\tau \sim F_{G, \trajectoryrand, h+1}} \left[R_{h:\tau-1} + \bar v_{\psi_{h+1:H+1}(\pi^\star_G)} (S_\tau)\right]\right) \, .
    \end{align*}       
    With the definitions as above, applying \cref{lem:least squared error decomposition} we get
    \begin{align}
        \norm{\hat \psi_h(\pi^\star_G) - \psi_h(\pi^\star_G)}_{X_h}
        &\le \sqrt{\lambda}\norm{\psi_h(\pi^\star_G)}_2 + \norm{\Delta}_\infty \sqrt{\nh} + \norm{\iota}_{X_h^{-1}} \, .
        \label{eq:least squares bound}
    \end{align}    
    The first term in \cref{eq:least squares bound} can be bounded by $H^{3/2} d$ by recalling that $\sqrt{\lambda} = H^{3/2} d/\modthetabound$ (\cref{def:lambda}) and $\norm{\psi_h(\pi^\star_G)}_2 \le \thetabound \le \modthetabound$.
    
    The $\norm{\Delta}_\infty$ in the second term in \cref{eq:least squares bound} can be bounded by first decomposing the error, and using a triangle inequality as follows. 
    \begin{align}
        \norm{\Delta}_\infty
        &\le \norm{\left(\bigE_{\trajectoryrand \sim \bP_{\behavepi, s_h^j, a_h^j}} \bigE_{\tau \sim F_{G, \trajectoryrand, h+1}} \left[\bar v_{\psi_{h+1:H+1}(\pi^\star_G)}(S_\tau) 
            - \max_{a' \in \cA}q^{\pi^\star_G}(S_\tau, a')\right]\right)_{j \in [\nh]}}_\infty \nonumber \\
        &\qquad+ \norm{\left(\bigE_{\trajectoryrand \sim \bP_{\behavepi, s_h^j, a_h^j}} \bigE_{\tau \sim F_{G, \trajectoryrand, h+1}} \left[R_{h:\tau-1} + \max_{a' \in \cA}q^{\pi^\star_G} (S_\tau, a')\right] 
            - \ip{\phi_h^j, \psi_h(\pi^\star_G)}\right)_{j \in [\nh]}}_\infty \, .
        \label{eq:Delta psi bound}
    \end{align}
    For the first term in \cref{eq:Delta psi bound}, notice that 
    \begin{align*}
        \bar v_{\psi_{h+1:H+1}(\pi^\star_G)}(S_\tau) 
            - \max_{a' \in \cA}q^{\pi^\star_G}(S_\tau, a')
        \le \max_{a' \in \cA} \left(\ip{\phi(S_\tau, a'), \psi_{\stage(S_\tau)}(\pi^\star_G)} 
            - q^{\pi^\star_G}(S_\tau, a')\right)
        \le \eta \, ,
    \end{align*}
    where the last inequality holds since we have approximate linear $q^\pi$-realizability (\cref{ass:approximate q-pi realizability}).
    To bound the second term in \cref{eq:Delta psi bound} notice that by the definition of $\pi^\star_G$ (\cref{eq:pi*G definition}), $\argmax_{a' \in \cA} q^{\pi^\star_G}(S_\tau, a')$ is exactly the action $\pi^\star_G$ would take at the stopping stage $\tau$,
    and that the distribution of $\trajectoryrand$ under policy $\pi^0$ until stopping stage $\tau$ is same as the distribution of $\trajectoryrand$ under policy $\pi^\star_G$ until stopping stage $\tau$.
    This implies that
    \begin{align*}
        &\norm{\left(\bigE_{\trajectoryrand \sim \bP_{\behavepi, s_h^j, a_h^j}} \bigE_{\tau \sim F_{G, \trajectoryrand, h+1}} \left[R_{h:\tau-1} + \max_{a' \in \cA}q^{\pi^\star_G} (S_\tau, a')\right] 
            - \ip{\phi_h^j, \psi_h(\pi^\star_G)}\right)_{j \in [\nh]}}_\infty \\
        &= \norm{\left(q^{\pi^\star_G}(s_h^j, a_h^j)
            - \ip{\phi_h^j, \psi_h(\pi^\star_G)}\right)_{j \in [\nh]}}_\infty     
        \le \eta \, ,
    \end{align*}
    where the last inequality holds since we have approximate linear $q^\pi$-realizability (\cref{ass:approximate q-pi realizability}).
    Plugging the above bounds into \cref{eq:Delta psi bound}, we get that $\norm{\Delta}_\infty \le 2\eta$.
    
    To bound the third term in \cref{eq:least squares bound}, let the event $\cE_1$ be as defined in the proof of \cref{lem:least squares noise cover bound}, which occurs with probability at least $1 - \delta/3$.
    Then, under event $\cE_1$, the third term in \cref{eq:least squares bound} can be bounded by $\bar \beta$ (\cref{def:bar-beta}), by applying \cref{lem:least squares noise cover bound} since $\psi_{h+1:H+1}(\pi^\star_G) \in \cB(\thetabound)^{H-h+1} \subset \cB(\modthetabound)^{H-h+1}$.
    
    Plugging the three bounds above back into \cref{eq:least squares bound} we get that, under event $\cE_1$, it holds that
    \begin{align*}
        &\norm{\hat \psi_h(\pi^\star_G) - \psi_h(\pi^\star_G)}_{X_h}
        \le H^{3/2}d + 2\eta \sqrt{\nh} + \bar \beta
        \le \beta \, ,
    \end{align*}
    where $\beta$ is defined in \cref{def:beta}, and the last inequality can be seen to hold by plugging in parameter values according to \cref{sec:parameter settings}.
    Thus, under event $\cE_1$, which occurs with probability at least $1 - \delta/3$, for any $G \in \bbG$ and $h \in [H]$ it holds that $\psi_h(\pi^\star_G) \in \Theta_{G, h}$, which completes the proof.
\end{proof}

\newpage
\section{Proof of \cref{lem:Theta elements are close}} 

\begin{proof} \label{proof:lem:Theta elements are close}
    Let $G \in \bbG$ be a feasible solution to \cref{opt:bpse start state}.
    By \cref{lem:opt prob cond concentrates},
    there is an event $\cE_2$, 
    that occurs with probability at least $1 - \delta/3$, 
    such that under event $\cE_2$, 
    for all $G \in \bbG$,
    and for all $h \in [H]$, 
    it holds that 
    \begin{align*}
        &\abs{\bigE_{(S, A) \sim \mu_h} \left[\max_{\theta \in \Theta_{G, h}} \bar q_{\theta}(S, A) - \min_{\theta \in \Theta_{G, h}} \bar q_{\theta}(S, A)\right]
            - \frac{1}{\nh} \sum_{i \in [\nh]} \left(\max_{\theta \in \Theta_{G, h}} \bar q_{\theta}(s_h^i, a_h^i) - \min_{\theta \in \Theta_{G, h}} \bar q_{\theta}(s_h^i, a_h^i)\right)} \\
        &\le \frac{H}{\sqrt{\nh}} \sqrt{\log\left(\frac{6H|C_\xi^\bbG|}{\delta}\right)} 
            + 2\modxi \, . 
    \end{align*}
    where $|C_\xi^\bbG|, \modxi$, are defined in \cref{def:G cover size,def:modxi}.
    Let $h \in [H]$.
    Recall that since $G$ is a feasible solution to \cref{opt:bpse start state} we know that \cref{eq:opt prob cond} passed for $h$.
    Thus,
    \begin{align*}
        \frac{1}{\nh} \sum_{i \in [\nh]} \left(\max_{\theta \in \Theta_{G, h}} \bar q_{\theta}(s_h^i, a_h^i) - \min_{\theta \in \Theta_{G, h}} \bar q_{\theta}(s_h^i, a_h^i)\right) 
        \le \bar \eps \, . 
    \end{align*}
    Combining the above two results,
    we have that under event $\cE_2$, 
    for any $h \in [H]$, 
    it holds that
    \begin{align*}
        \bigE_{(S, A) \sim \mu_h} \left[\max_{\theta \in \Theta_{G, h}} \bar q_{\theta}(S, A) - \min_{\theta \in \Theta_{G, h}} \bar q_{\theta}(S, A)\right]
        \le \frac{H}{\sqrt{\nh}} \sqrt{\log\left(\frac{6H|C_\xi^\bbG|}{\delta}\right)} 
            + 2\modxi 
            + \bar \eps \, .
    \end{align*}
 
   Now we will relate the data collecting distribution $\mu$ to any admissible distribution $\nu$.
   Notice that by the definition of $\max$ and $\min$, for any $(s, a) \in \cS \times \cA$, it holds that 
    \begin{align*}
        \max_{\theta \in \Theta_{G, h}} \bar q_{\theta}(s, a) - \min_{\theta \in \Theta_{G, h}} \bar q_{\theta}(s, a) \ge 0 \, .
    \end{align*}
    Thus, we can apply \cref{lem:conc error bound} to get that under event $\cE_2$, 
    for all $h \in [H]$,
    and admissible distribution $\nu = (\nu_t)_{t \in [H]}$, 
    it holds that
    \begin{align*}
        \bigE_{(S, A) \sim \nu_h} \left[\max_{\theta \in \Theta_{G, h}} \bar q_{\theta}(S, A) - \min_{\theta \in \Theta_{G, h}} \bar q_{\theta}(S, A)\right]
        \le \conc \left(\frac{H}{\sqrt{\nh}} \sqrt{\log\left(\frac{6H|C_\xi^\bbG|}{\delta}\right)} 
            + 2\modxi 
            + \bar \eps\right) \, .
    \end{align*}
    To conclude, we have that under event $\cE_2$, 
    for any $G \in \bbG$ that is a feasible solution to \cref{opt:bpse start state}, 
    for any $h \in [H]$, for any
        $(\theta_{s,a})_{(s,a)\in\cS_h\times\cA}$
        and $(\check\theta_{s,a})_{(s,a)\in\cS_h\times\cA}$
        with $\theta_{s,a},\check\theta_{s,a}\in\Theta_{G, h}$ for all $(s,a)\in\cS_h\times\cA$,
    and for any admissible distribution $\nu = (\nu_t)_{t \in [H]}$, 
    it holds that
    \begin{align}
        &\bigE_{(S, A) \sim \nu_h} \left[\bar q_{\theta_{S, A}}(S, A) - \bar q_{\check \theta_{S, A}}(S, A)\right] \nonumber \\
        &\le \bigE_{(S, A) \sim \nu_h} \left[\max_{\theta \in \Theta_{G, h}} \bar q_{\theta}(S, A) - \min_{\theta \in \Theta_{G, h}} \bar q_{\theta}(S, A)\right] \nonumber \\
        &\le \conc \left(\frac{H}{\sqrt{\nh}} \sqrt{\log\left(\frac{6H|C_\xi^\bbG|}{\delta}\right)} 
            + 2\modxi 
            + \bar \eps\right) \nonumber \\
        &\le \conc \left(\frac{H}{\sqrt{\nh}} \sqrt{\log\left(\frac{6H(1 + 2\usedtobesqrtdoneplusone/\xi))^{dHd_0}}{\delta}\right)} 
            + 24 \sqrt{2d} H^2 \featurebound \xi \alpha^{-1} \left(2\sqrt{\nh}\featurebound\modthetabound / (H^{3/2}d)\right)^{H} 
            + \bar \eps\right) \nonumber \\
        &\le \conc \left(\frac{H}{\sqrt{\nh}} \sqrt{ dH^2d_0\log\left(1 + 96\sqrt{2d}H^2\featurebound\usedtobesqrtdoneplusone\alpha^{-1}\sqrt{\nh}\featurebound\modthetabound / (H^{3/2}d)\right) + \log\left(\frac{6H}{\delta}\right)} 
            + \frac{1}{\sqrt{n}} 
            + \bar \eps\right) \nonumber \\
        &= \tilde \eps 
        = \tilde \cO\left(\frac{\conc H^2 d}{\sqrt{n}} + \frac{\conc}{\sqrt{n}} + \frac{\conc^2 H^{5/2} d^2}{\sqrt{n}}\right)
        = \tilde \cO\left(\frac{\conc^2 H^{5/2} d^2}{\sqrt{n}}\right) \, .
        \label{eq:tilde-eps definition}
    \end{align}       
    The third inequality holds by plugging in the values of $|C_\xi^\bbG|, \modxi$, as defined in \cref{def:G cover size,def:modxi}.
    The last inequality holds by setting $\xi^{-1} = 24 \sqrt{2d} \sqrt{n} H^2 \featurebound \alpha^{-1} \left(2\sqrt{\nh}\featurebound\modthetabound / (H^{3/2}d)\right)^{H} $.
    The last two equalities hold by plugging in parameter values according to \cref{sec:parameter settings}.
\end{proof}

\newpage
\section{Proof of \cref{lem:true guess is feasible}} 

\begin{proof} \label{proof:lem:true guess is feasible}
    
    To show that $\trueG$ is a feasible solution we need to show that \cref{eq:opt prob cond} is satisfied for all $h \in [H]$.

    By \cref{lem:opt prob cond concentrates},
    there is an event $\cE_2$, 
    that occurs with probability at least $1 - \delta/3$, 
    such that under event $\cE_2$, 
    for all $h \in [H]$, 
    it holds that 
    \begin{align}
        &\frac{1}{\nh} \sum_{i \in [\nh]} \left(\max_{\theta \in \Theta_{\trueG, h}} \bar q_{\theta}(s_h^i, a_h^i) - \min_{\theta \in \Theta_{\trueG, h}} \bar q_{\theta}(s_h^i, a_h^i)\right)
            - \bigE_{(S, A) \sim \mu_h} \left[\max_{\theta \in \Theta_{\trueG, h}} \bar q_{\theta}(S, A) - \min_{\theta \in \Theta_{\trueG, h}} \bar q_{\theta}(S, A)\right] \nonumber \\
        &\le \frac{H}{\sqrt{\nh}} \sqrt{\log\left(\frac{6H|C_\xi^\bbG|}{\delta}\right)} 
            + 2\modxi \, . 
        \label{eq:true guess event 2 bound}
    \end{align}
    where $|C_\xi^\bbG|, \modxi$, are defined in \cref{def:G cover size,def:modxi}.
    Let $h \in [H]$ be arbitrary for the remainder of the proof.
    We focus on bounding $\bigE_{(S, A) \sim \mu_h} \left[\max_{\theta \in \Theta_{\trueG, h}} \bar q_{\theta}(S, A) - \min_{\theta \in \Theta_{\trueG, h}} \bar q_{\theta}(S, A)\right]$ for the remainder of the proof.
    The following lemma will be helpful (proof in \cref{proof:lem:sum of elliptical terms bound}).
    \begin{lemma} \label{lem:sum of elliptical terms bound}
        There is an event $\cE_1$, that occurs with probability at least $1 - \delta/3$, such that under event $\cE_1$,
        for any $h \in [H], 
        (s, a) \in \cS_h \times \cA, 
        \theta_h \in \Theta_{\trueG, h}$, 
        it holds that
        \begin{align}
            \abs{\bar q_{\theta_h}(s, a) - q^{\pi^\star_{\trueG}}(s, a)} 
            \le 2\beta\bigE_{\trajectoryrand \sim \bP_{\bar \pi, s, a}} \sum_{t = h}^{H} \min\set{1, \norm{\phi(S_t, A_t)}_{X_t^{-1}}} + (H-h+1)\modmisspec \, ,
            \label{eq:q close inductive hypothesis}
        \end{align}   
        where, for any $(s', a') \in \cS \times \cA$,
        \begin{align}
            \bar \pi(a'|s') 
            &= \behavepi(a'|s') \omega_\trueG(s') + \I{\argmax_{a'' \in \cA} g^{\bar \pi}(s', a'') = a'} (1 - \omega_\trueG(s')) \, ,
            \label{eq:optpi elliptical}
        \end{align}
        and $g^{\bar \pi}$ is a state-action value function of policy $\bar \pi$ (similar to $q^{\bar \pi}$), 
        except in the alternative MDP that has the same state and action spaces, and transition distributions as the original MDP under consideration, but with a reward function modified as follows. 
        For all $(s', a') \in \cS \times \cA$, the reward in this alternative MDP is deterministically $\min\set{1, \norm{\phi(s', a')}_{X_h^{-1}}}$ $\braces{\text{i.e. }\cR(s', a') = \I{\min\set{1, \norm{\phi(s', a')}_{X_h^{-1}}}}}$.
        In particular for any $h' \in [H], (s', a') \in \cS_{h'} \times \cA$
        \begin{align}
            g^{\bar \pi}(s', a') = \bigE_{\trajectoryrand \sim \bP_{\bar \pi, s', a'}} \sum_{t = h'}^{H} \min\set{1, \norm{\phi(S_t, A_t)}_{X_t^{-1}}} \, .
            \label{eq:gpi definition}
        \end{align}
        The recursive definition of $\bar \pi$ can be interpreted in the same way as described below \cref{eq:pi*G definition}.
    \end{lemma}
    Let $\bar \pi$ be as defined in \cref{lem:sum of elliptical terms bound}.
    Then, by \cref{lem:sum of elliptical terms bound}, under event $\cE_1$, it holds that
    \begin{align}
        &\bigE_{(S_h, A_h) \sim \mu_h} \left[\max_{\theta \in \Theta_{\trueG, h}} \bar q_\theta(S_h, A_h) - \min_{\theta \in \Theta_{\trueG, h}} \bar q_\theta(S_h, A_h)\right] \nonumber \\
        &= \bigE_{(S_h, A_h) \sim \mu_h} \left[\max_{\theta \in \Theta_{\trueG, h}} \bar q_\theta(S_h, A_h) - v^{\pi^\star_\trueG}(S_h) + v^{\pi^\star_\trueG}(S_h) - \min_{\theta \in \Theta_{\trueG, h}} \bar q_\theta(S_h, A_h)\right] \nonumber \\
        &\le \bigE_{(S_h, A_h) \sim \mu_h} \sqbraces{4\beta\bigE_{\trajectoryrand \sim \bP_{\bar \pi, S_h, A_h}} \sum_{t = h}^{H} \min\set{1, \norm{\phi(S_t, A_t)}_{X_t^{-1}}} + 2(H-h+1)\modmisspec} \nonumber \\
        &= 4\beta \bigE_{(S_h, A_h) \sim \mu_h} \min\set{1, \norm{\phi(S_h, A_h)}_{X_h^{-1}}} + 4\beta \sum_{t = h+1}^{H} \bigE_{(S_h, A_h) \sim \mu_h} \bigE_{(S_t, A_t) \sim \bPmarg{t}_{\bar \pi, S_h, A_h}} \min\set{1, \norm{\phi(S_{t}, A_{t})}_{X_{t}^{-1}}} \nonumber \\
            &\qquad + (H-h+1)\modmisspec \, .
        \label{eq:sum of elliptical term under bar-pi}
    \end{align}
    The inequality used \cref{lem:sum of elliptical terms bound}.
    As we will show shortly, the first term can be bounded by \cref{lem:expected elliptical potential bound}, since its expectation is taken w.r.t. the data collecting distribution $\mu$.
    Thus, the approach we take to bounding the second term is to relate its nested expectations to just be a single expectation taken w.r.t. the distribution $\mu$ (similar to the first term, which we claim we know how to bound).
    To this end,
    for any $(s', a') \in \cS \times \cA$, define the policy $\check \pi_h$ as
    \begin{align*}
        \check \pi_h(a'|s')
        = \begin{cases}
            \behavepi(a'|s') & \text{if } \stage(s') \le h \\
            \bar \pi(a'|s') & \text{if } \stage(s') > h \, . 
        \end{cases}
    \end{align*}
    Notice that for any $t \in [H], u \in [t+1:H+1]$, $\bigE_{(S_t, A_t) \sim \mu_t} \bigE_{(S_u, A_u) \sim \bPmarg{u}_{\bar \pi, S_t, A_t}} = \bigE_{(S_u, A_u) \sim \bPmarg{u}_{\check \pi_t, s}}$.
    Let $\check \nu_u(s_u, a_u) = \bPmarg{u}_{\check \pi_h, \startstate}(s_u, a_u)$ for all $u \in [H], (s_u, a_u) \in \cS_u \times \cA$. 
    Clearly, $\check \nu = (\check \nu_u)_{u \in [H]}$ is an admissible distribution by \cref{def:admissible dist}.
    Thus, with the definition of $\check \nu$ and using \cref{lem:conc error bound}, we get that \cref{eq:sum of elliptical term under bar-pi} is
    \begin{align*}
        &= 4\beta \bigE_{(S_h, A_h) \sim \mu_h} \min\set{1, \norm{\phi(S_h, A_h)}_{X_h^{-1}}} + 4\beta \sum_{t = h+1}^{H} \bigE_{(S_t, A_t) \sim \check \nu_h} \min\set{1, \norm{\phi(S_{t}, A_{t})}_{X_{t}^{-1}}} + (H-h+1)\modmisspec \\
        &\le 4\beta \bigE_{(S_h, A_h) \sim \mu_h} \min\set{1, \norm{\phi(S_h, A_h)}_{X_h^{-1}}} + 4\conc\beta \sum_{t = h+1}^{H} \bigE_{(S_t, A_t) \sim \mu_t} \min\set{1, \norm{\phi(S_{t}, A_{t})}_{X_{t}^{-1}}} + (H-h+1)\modmisspec \\
        &\le 4\conc\beta \sum_{t = h}^{H} \bigE_{(S_t, A_t) \sim \mu_t} \min\set{1, \norm{\phi(S_{t}, A_{t})}_{X_{t}^{-1}}} + (H-h+1)\modmisspec \, . 
    \end{align*}
    The last inequality used that $\conc \ge 1$.
    Finally, we can apply \cref{lem:expected elliptical potential bound} to bound $\bigE_{(S_t, A_t) \sim \mu_t} \min\set{1, \norm{\phi(S_{t}, A_{t})}_{X_t^{-1}}}$.
    In particular, let $\cE_3$ be as defined in the proof of \cref{lem:expected elliptical potential bound}.
    Then, by \cref{lem:expected elliptical potential bound}, under event $\cE_3$, it holds that
    \begin{align*}
        4\conc\beta \sum_{t = h}^{H} \bigE_{(S_t, A_t) \sim \mu_t} \min\set{1, \norm{\phi(S_{t}, A_{t})}_{X_t^{-1}}}
        \le 4\conc\beta \sum_{t = h}^{H} \check \eps
        \le 4H \conc \check \eps \beta \, .
    \end{align*}
    Putting all of the bound after \cref{eq:sum of elliptical term under bar-pi} together and plugging them into \cref{eq:sum of elliptical term under bar-pi}, we have that,
    under event $\cE_1 \cap \cE_2$, 
    it holds that
    \begin{align}
        \bigE_{(S_h, A_h) \sim \mu_h} \left[\max_{\theta \in \Theta_{\trueG, h}} \bar q_\theta(S_h, A_h) - \min_{\theta \in \Theta_{\trueG, h}} \bar q_\theta(S_h, A_h)\right]
        \le (H-h+1) \modmisspec + 4H \conc \check\eps \beta \, .
        \label{eq:true guess event 1 and 3 bound}
    \end{align}
    We are now ready to state the final result.
    Noting that $h \in [H]$ was arbitrary,
    by combining \cref{eq:true guess event 2 bound,eq:true guess event 1 and 3 bound}, we have that,
    under event $\cE_1 \cap \cE_2 \cap \cE_3$,
    for all $h \in [H]$,
    it holds that
    \begin{align}
        &\frac{1}{\nh} \sum_{i \in [\nh]} \left(\max_{\theta \in \Theta_{\trueG, h}} \bar q_{\theta}(s_h^i, a_h^i) - \min_{\theta \in \Theta_{\trueG, h}} \bar q_{\theta}(s_h^i, a_h^i)\right) \nonumber \\
        &\le \frac{H}{\sqrt{\nh}} \sqrt{\log\left(\frac{6H|C_\xi^\bbG|}{\delta}\right)} 
            + 2\modxi 
            + (H-h+1) \modmisspec 
            + 4H \conc \check\eps \beta \nonumber \\
        &= \frac{H}{\sqrt{\nh}} \sqrt{\log\left(\frac{6H(1 + 2\usedtobesqrtdoneplusone/\xi))^{dHd_0}}{\delta}\right)} 
            + 24 \sqrt{2d} H^2 \featurebound \xi \alpha^{-1} \left(2\sqrt{\nh}\featurebound\modthetabound / (H^{3/2}d)\right)^{H} \nonumber \\
            &\qquad + (H-h+1) \modmisspec 
            + 4H \conc \check\eps \beta \nonumber \nonumber \\    
        &\le \frac{H}{\sqrt{\nh}} \sqrt{dH^2d_0\log\left(1 + 96 \sqrt{n} \sqrt{2d} H^2 \featurebound \usedtobesqrtdoneplusone \alpha^{-1} \sqrt{\nh}\featurebound\modthetabound / (H^{3/2}d)\right) + \log\left(\frac{6H}{\delta}\right)} 
            + \frac{1}{\sqrt{n}} \nonumber \\
            &\qquad + (H-h+1) \modmisspec 
            + 4H \conc \check\eps \beta \nonumber \\
        &= \bar \eps
        = \tilde \cO\left(\frac{H^2 d}{\sqrt{n}} + \frac{1}{\sqrt{n}} + \frac{H^{5/2} d}{\sqrt{n}} + \frac{\conc H^{5/2} d^2}{\sqrt{n}}\right)
        = \tilde \cO\left(\frac{\conc H^{5/2} d^2}{\sqrt{n}}\right) \, .
        \label{eq:bar-eps definition}
    \end{align}
    The first equality holds by plugging in the values of $|C_\xi^\bbG|, \modxi$, as defined in \cref{def:G cover size,def:modxi}. 
    The second equality holds by setting $\xi^{-1} = 24 \sqrt{n} \sqrt{2d} H^2 \featurebound \alpha^{-1} \left(2\sqrt{\nh}\featurebound\modthetabound / (H^{3/2}d)\right)^{H}$.
    The last two equalities hold by plugging in parameter values according to \cref{sec:parameter settings}.
    
    Noticing that this is exactly the condition (\cref{eq:opt prob cond}) in \cref{opt:bpse start state} that needs to be satisfied by any feasible solution,
    we conclude that,
    under event $\cE_1 \cap \cE_2 \cap \cE_3$,
    the true guess $\trueG$ is a feasible solution to \cref{opt:bpse start state}.

\end{proof}

\subsection{Proof of \cref{lem:sum of elliptical terms bound}} 
\begin{proof} \label{proof:lem:sum of elliptical terms bound}
    To prove \cref{eq:q close inductive hypothesis} we will use induction.
    The base case is when $h = H+1$, for which $\bar v_{\theta_{H+1}}(s) = v^{\pi^\star_\trueG}(s)$ for all $s \in \termstatespace$.
    This holds, since for all $(s, a) \in \termstatespace \times \cA$, $\bar q_{\theta_{H+1}}(s, a) = 0$ for all $\theta_{H+1} \in \Theta_{\trueG, H+1}$, by the definition of $\Theta_{\trueG, H+1}$ (\cref{eq:Theta definition}),
    and $v^{\pi^\star_\trueG}(s) = 0$, by \cref{eq:v and q at term state}, since $\termstatespace = \set{\termstate}$.

    Now, we show the inductive step.
    Let $h \in [H]$ be arbitrary.
    Assume that \cref{eq:q close inductive hypothesis} holds for any $t \in [h+1:H+1]$.
    We prove that \cref{eq:q close inductive hypothesis} holds for $h$.
    Let $(s, a) \in \cS_h \times \cA, \theta_h \in \Theta_{\trueG, h}$.
    Then,
    \begin{align}
        \abs{\bar q_{\theta_h}(s, a) - q^{\pi^\star_{\trueG}}(s, a)} 
        &\le \abs{\min\set{H, \ip{\phi(s, a), \theta_h} -  q^{\pi^\star_{\trueG}}(s, a)}} \nonumber \\
        &= \min\set{H, \abs{\ip{\phi(s, a), \theta_h} -  q^{\pi^\star_{\trueG}}(s, a)}} \, ,
        \label{eq:clipped q bound}
    \end{align}
    where the last equality holds since $\left(\ip{\phi(s, a), \theta_h} -  q^{\pi^\star_{\trueG}}(s, a)\right) \in [-H, H]$.
    We will focus on bounding $\abs{\ip{\phi(s, a), \theta_h} -  q^{\pi^\star_{\trueG}}(s, a)}$. 
    By the definition of the set $\Theta_{\trueG, h}$ (\cref{eq:Theta definition}) we know that there exists a $\hat \theta_h \in \hat \Theta_{\trueG, h}$ such that $\norm{\theta_h - \hat \theta_h}_{X_h} \le \beta$.
    Thus, by the Cauchy-Schwarz inequality, we have that
    \begin{align*}
        \abs{\ip{\phi(s, a), \theta_h - \hat \theta_h}}
        \le \norm{\phi(s, a)}_{X_h^{-1}} \norm{\theta_h - \hat \theta_h}_{X_h} 
        \le \beta \norm{\phi(s, a)}_{X_h^{-1}} \, . 
    \end{align*}
    This implies that 
    \begin{align}
        \abs{\ip{\phi(s, a), \theta_h} -  q^{\pi^\star_{\trueG}}(s, a)}
        \le \abs{\ip{\phi(s, a), \hat \theta_h} -  q^{\pi^\star_{\trueG}}(s, a)} + \beta \norm{\phi(s, a)}_{X_h^{-1}} \, .  
        \label{eq:theta true guess bound}
    \end{align}
    Since we know $\hat \theta_h \in \hat \Theta_{\trueG, h}$, by the definition of $\hat \Theta_{\trueG, h}$ (\cref{eq:hat-Theta defintion}), there exists a $\theta_{h+1:H+1} \in \Theta_{\trueG, h+1} \times \dots \times \Theta_{\trueG, H+1}$, 
    such that
    \begin{align*}
        \hat \theta_h
        = X_h^{-1} \sum_{j \in [\nh]} \phi^j_h \bigE_{\tau \sim F^j_{\trueG, h+1}}\left[r^j_{h:\tau-1} + \bar v_{\theta_{h+1:H+1}} \left(s^j_{\tau}\right)\right] \, .
    \end{align*}
    By \cref{lem:approx linear MDP},
    we know there exists a parameter $\rho_h^\behavepi(f) \in \cB(\modthetabound)$,
    such that for all $(s, a) \in \cS_h \times \cA$,
    \begin{align}
        &\abs{\bigE_{\trajectoryrand \sim \bP_{\behavepi, s, a}} \bigE_{\tau \sim F_{\trueG, \trajectoryrand, h+1}} \left[R_{h:\tau-1} + \bar v_{\theta_{h+1:H+1}}(S_\tau)\right] 
            - \ip{\phi(s,a),\rho_h^\behavepi\left(\bar v_{\theta_{h+1:H+1}}\right)}} 
        \le \modmisspec \, . 
        \label{eq:theta modmisspec bound}
    \end{align}
    Let $\rho_h^\behavepi$ be as defined above.
    Next, we will show a bound on $\abs{\ip{\phi(s, a), \hat \theta_h - \rho_h^\behavepi\left(\bar v_{\theta_{h+1:H+1}}\right)}}$ and $\abs{\ip{\phi(s, a), \rho_h^\behavepi\left(\bar v_{\theta_{h+1:H+1}}\right)} -  q^{\pi^\star_{\trueG}}(s, a)}$, which together will give us a bound on $\abs{\ip{\phi(s, a), \hat \theta_h} - q^{\pi^\star_{\trueG}}(s, a)}$, as desired.
    The following result gives us a bound on $\norm{\hat \theta_h - \rho_h^\behavepi\left(\bar v_{\theta_{h+1:H+1}}\right)}_{X_h}$.
    \begin{lemma} \label{lem:theta-hat least squares bound}
        There is an event $\cE_1$, which occurs with probability at least $1 - \delta/3$, such that under event $\cE_1$, 
        for all $G \in \bbG$
        for all $h \in [H]$,
        and for all $\hat \theta_h \in \hat \Theta_{G, h}$,
        it holds that
        \begin{align*}
            \norm{\hat \theta_h - \theta_h^\star}_{X_h}
            \le \beta \, , 
        \end{align*} 
        where
        \begin{align*}
            \hat \theta_h
            = X_h^{-1} \sum_{j \in [\nh]} \phi^j_h \bigE_{\tau \sim F^j_{G, h+1}}\left[r^j_{h:\tau-1} + \bar v_{\theta_{h+1:H+1}} \left(s^j_{\tau}\right)\right]
            \quad \text{for some } \theta_{h+1:H+1} \in \Theta_{G, h+1} \times \dots \times \Theta_{G, H+1} \, ,
        \end{align*}
        and $\theta_h^\star \in \cB(\modthetabound)$ is such that,
        for all $(s, a) \in \cS_h \times \cA$,
        it satisfies
        \begin{align*}
            &\abs{\bigE_{\trajectoryrand \sim \bP_{\behavepi, s, a}} \bigE_{\tau \sim F_{G, \trajectoryrand, h+1}} \left[R_{h:\tau-1} + \bar v_{\theta_{h+1:H+1}}(S_\tau)\right] 
                - \ip{\phi(s,a), \theta_h^\star}} 
            \le \modmisspec \, . 
        \end{align*}
    \end{lemma}
    \begin{proof} \label{proof:lem:theta-hat least squares bound}
        Fix $G \in \bbG, h \in [H]$, and $\hat \theta_h \in \hat \Theta_{G, h}$, such that 
        \begin{align*}
            \hat \theta_h
            = X_h^{-1} \sum_{j \in [\nh]} \phi^j_h \bigE_{\tau \sim F^j_{G, h+1}}\left[r^j_{h:\tau-1} + \bar v_{\theta_{h+1:H+1}} \left(s^j_{\tau}\right)\right]
            \quad \text{for some } \theta_{h+1:H+1} \in \Theta_{G, h+1} \times \dots \times \Theta_{G, H+1} \, .
        \end{align*}
        Fix $\theta_h^\star \in \cB(\modthetabound)$, such that,
        for all $(s, a) \in \cS_h \times \cA$,
        it satisfies
        \begin{align}
            &\abs{\bigE_{\trajectoryrand \sim \bP_{\behavepi, s, a}} \bigE_{\tau \sim F_{G, \trajectoryrand, h+1}} \left[R_{h:\tau-1} + \bar v_{\theta_{h+1:H+1}}(S_\tau)\right] 
                - \ip{\phi(s,a), \theta_h^\star}} 
            \le \modmisspec \, . 
            \label{eq:Delta modmisspec bound}
        \end{align}
        We would like to make use of \cref{lem:least squared error decomposition} to bound $\norm{\hat \theta_h - \theta_h^\star}_{X_h}$.
        To do so, we map the terms used in the \cref{lem:least squared error decomposition} to our terms as follows
        \begin{align*}
            &n = \nh, \lambda = \lambda, \theta_{\star} = \theta_h^\star, V = X_h, \hat \theta = \hat \theta_h, A = \left(\phi_h^j\right)_{j \in [\nh]} \, , \\
            &Y = \tilde Y + \Delta = \left(\ip{\phi_h^j, \theta_h^\star}\right)_{j \in [\nh]} + \gamma + \Delta = \left(\bigE_{\tau \sim F^j_{G, h+1}}\left[r^j_{h:\tau-1} + \bar v_{\theta_{h+1:H+1}} \left(s^j_{\tau}\right)\right]\right)_{j \in [\nh]} \, , \\
            &\gamma = \left(\bigE_{\tau \sim F^j_{G, h+1}}\left[r^j_{h:\tau-1} + \bar v_{\theta_{h+1:H+1}} \left(s^j_{\tau}\right)\right] 
                -  \bigE_{\trajectoryrand \sim \bP_{\behavepi, s_h^j, a_h^j}} \bigE_{\tau \sim F_{G, \trajectoryrand, h+1}} \left[R_{h:\tau-1} + \bar v_{\theta_{h+1:H+1}} (S_\tau)\right]\right)_{j \in [\nh]} \, , \\
            &\Delta = \left(\bigE_{\trajectoryrand \sim \bP_{\behavepi, s_h^j, a_h^j}} \bigE_{\tau \sim F_{G, \trajectoryrand, h+1}} \left[R_{h:\tau-1} + \bar v_{\theta_{h+1:H+1}} (S_\tau)\right] 
                - \ip{\phi_h^j, \theta^\star_h}\right)_{j \in [\nh]} \, , \\ 
            &\iota 
            = \sum_{j \in [\nh]} \phi_h^j \left(\bigE_{\tau \sim F^j_{G, h+1}}\left[r^j_{h:\tau-1} + \bar v_{\theta_{h+1:H+1}} \left(s^j_{\tau}\right)\right] 
                -  \bigE_{\trajectoryrand \sim \bP_{\behavepi, s_h^j, a_h^j}} \bigE_{\tau \sim F_{G, \trajectoryrand, h+1}} \left[R_{h:\tau-1} + \bar v_{\theta_{h+1:H+1}} (S_\tau)\right]\right) \, .
        \end{align*}       
        With the definitions as above, applying \cref{lem:least squared error decomposition} we get
        \begin{align}
            \norm{\hat \theta_h - \theta_h^\star}_{X_h}
            &\le \sqrt{\lambda}\norm{\theta_h^\star}_2 + \norm{\Delta}_\infty \sqrt{\nh} + \norm{\iota}_{X_h^{-1}} \, .
            \label{eq:true guess least squares bound}
        \end{align}           
        The first term can be bounded by $H^{3/2} d$, by noting that $\norm{\theta_h^\star}_2 \le \modthetabound$, and setting 
        \begin{align}
            \sqrt{\lambda} = H^{3/2} d/\modthetabound \, .
            \label{eq:lambda definition}
        \end{align}
        The second term can be bounded by $\modmisspec \sqrt{n}$, by using \cref{eq:Delta modmisspec bound}.
        
        To bound the third term let the event $\cE_1$ be as defined in the proof of \cref{lem:least squares noise cover bound}, which occurs with probability at least $1 - \delta/3$.
        Then, under event $\cE_1$, the third term in \cref{eq:true guess least squares bound} can be bounded by $\bar \beta$ (\cref{def:bar-beta}), by applying \cref{lem:least squares noise cover bound}, since $\theta_{h+1:H+1} \in \Theta_{G, h+1} \times \dots \times \Theta_{G, H+1} \subset \cB(\modthetabound)^{H-h+1}$.
        
        Plugging the three bounds above back into \cref{eq:true guess least squares bound} we get that under event $\cE_1$, it holds that
        \begin{align}
            \norm{\hat \theta_h - \theta_h^\star}_{X_h}
            &\le H^{3/2} d + \modmisspec \sqrt{\nh} + \bar \beta \nonumber \\
            &= \beta
            = \tilde \cO\left(H^{3/2} d\right) \, .
            \label{eq:beta definition}
        \end{align}
        The last equality holds by setting
        \begin{align}
            \misspec
            \le \frac{H^{3/2} d}{\sqrt{n} \left(10 H^2 d_0/\alpha + 1\right)}
            \implies
            \modmisspec
            \le H^{3/2}d / \sqrt{n} \, ,
            \label{eq:eta definition}
        \end{align}
        and the values of $n, \bar \beta$ are set according to \cref{sec:parameter settings}.
    \end{proof}
    
    We return back to the proof of \cref{lem:sum of elliptical terms bound}. 
    Let $\cE_1$ be as defined in the proof of \cref{lem:theta-hat least squares bound}.
    For the remainder of the proof, operate under event $\cE_1$.
    By \cref{lem:theta-hat least squares bound} (with $\rho_h^\behavepi\left(\bar v_{\theta_{h+1:H+1}}\right) = \theta_h^\star$), we have that
    \begin{align}
        \norm{\hat \theta_h - \rho_h^\behavepi\left(\bar v_{\theta_{h+1:H+1}}\right)}_{X_h}
        \le \beta \, . 
        \label{eq:true guess ellipsoid bounds}
    \end{align} 
    Returning to $\abs{\ip{\phi(s, a), \hat \theta_h - \rho_h^\behavepi\left(\bar v_{\theta_{h+1:H+1}}\right)}}$, we can now bound it as follows.
    \begin{align}
        \abs{\ip{\phi(s, a), \hat \theta_h - \rho_h^\behavepi\left(\bar v_{\theta_{h+1:H+1}}\right)}}
        \le \norm{\phi(s, a)}_{X_h^{-1}}\norm{\hat \theta_h - \rho_h^\behavepi\left(\bar v_{\theta_{h+1:H+1}}\right)}_{X_h}
        \le \beta \norm{\phi(s, a)}_{X_h^{-1}} \, .
        \label{eq:hat-theta true guess bound}
    \end{align}
    The first inequality used the Cauchy-Schwarz inequality.
    The second inequality used \cref{eq:true guess ellipsoid bounds}.
    Now, we bound $\abs{\ip{\phi(s, a), \rho_h^\behavepi\left(\bar v_{\theta_{h+1:H+1}}\right)} -  q^{\pi^\star_{\trueG}}(s, a)}$, by making use of \cref{eq:theta modmisspec bound}, to get that
    \begin{align}
        &\abs{\ip{\phi(s, a), \rho_h^\behavepi\left(\bar v_{\theta_{h+1:H+1}}\right)} -  q^{\pi^\star_{\trueG}}(s, a)} \nonumber \\
        &\le \abs{\bigE_{\trajectoryrand \sim \bP_{\behavepi, s, a}} \bigE_{\tau \sim F_{\trueG, \trajectoryrand, h+1}} \left[R_{h:\tau-1} + \bar v_{\theta_{h+1:H+1}}(S_\tau)\right] - q^{\pi^\star_{\trueG}}(s, a)} + \modmisspec \nonumber \\
        &= \abs{\bigE_{\trajectoryrand \sim \bP_{\behavepi, s, a}} \bigE_{\tau \sim F_{\trueG, \trajectoryrand, h+1}} \left[R_{h:\tau-1} + \max_{a' \in \cA}\bar q_{\theta_{\stage(S_\tau)}}(S_\tau, a') - \max_{a' \in \cA} q^{\pi^\star_\trueG}(S_\tau, a') + \max_{a' \in \cA} q^{\pi^\star_\trueG}(S_\tau, a')\right] - q^{\pi^\star_{\trueG}}(s, a)} + \modmisspec \nonumber \\
        &\le \abs{\bigE_{\trajectoryrand \sim \bP_{\behavepi, s, a}} \bigE_{\tau \sim F_{\trueG, \trajectoryrand, h+1}} \left[R_{h:\tau-1} + \max_{a' \in \cA} q^{\pi^\star_\trueG}(S_\tau, a')\right] - q^{\pi^\star_{\trueG}}(s, a)} \\
            &\qquad + \abs{\bigE_{\trajectoryrand \sim \bP_{\behavepi, s, a}} \bigE_{\tau \sim F_{\trueG, \trajectoryrand, h+1}} \max_{a' \in \cA}\braces{\bar q_{\theta_{\stage(S_\tau)}}(S_\tau, a') - q^{\pi^\star_\trueG}(S_\tau, a')}} + \modmisspec \, .        
        \label{eq:rho q*G bound}
    \end{align}
    The equality used the definition of $\bar v$ (\cref{eq:bar-v_theta definition}).
    To bound the first term in \cref{eq:rho q*G bound} notice that by the definition of $\pi^\star_\trueG$ (\cref{eq:pi*G definition}), $\argmax_{a' \in \cA} q^{\pi^\star_\trueG}(S_\tau, a')$ is exactly the action $\pi^\star_\trueG$ would take at the stopping stage $\tau$,
    and that the distribution of $\trajectoryrand$ under policy $\pi^0$ until stopping stage $\tau$ is same as the distribution of $\trajectoryrand$ under policy $\pi^\star_\trueG$ until stopping stage $\tau$.
    This gives that
    \begin{align*}
        \abs{\bigE_{\trajectoryrand \sim \bP_{\behavepi, s, a}} \bigE_{\tau \sim F_{\trueG, \trajectoryrand, h+1}} \left[R_{h:\tau-1} + \max_{a' \in \cA} q^{\pi^\star_\trueG}(S_\tau, a')\right] - q^{\pi^\star_{\trueG}}(s, a)}
        = \abs{q^{\pi^\star_{\trueG}}(s, a) - q^{\pi^\star_{\trueG}}(s, a)}
        = 0 \, .
    \end{align*}
    To bound the second term in \cref{eq:rho q*G bound} we can use the inductive hypothesis (\cref{eq:q close inductive hypothesis}).
    Defining notation that will be needed for the below display, for any $(s', a') \in \cS \times \cA$ we will write $\trajectoryrand' \sim \bP_{\check \pi, s, a}$ to have the usual definition $\trajectoryrand' = (s', a', R_h', \dots, S_{H+1}', A_{H+1}', R_{H+1}')$, except with a superscript $(\cdot)'$ added to all of the random elements.
    Then, by letting $a_\tau = \argmax_{a' \in \cA}\braces{\bar q_{\theta_{\stage(S_\tau)}}(S_\tau, a') - q^{\pi^\star_\trueG}(S_\tau, a')}$, applying the inductive hypothesis, and a triangle inequality,
    we have that
    \begin{align*}
        &\abs{\bigE_{\trajectoryrand \sim \bP_{\behavepi, s, a}} \bigE_{\tau \sim F_{\trueG, \trajectoryrand, h+1}} \max_{a' \in \cA}\braces{\bar q_{\theta_{\stage(S_\tau)}}(S_\tau, a') - q^{\pi^\star_\trueG}(S_\tau, a')}} \\
        &\le \abs{\bigE_{\trajectoryrand \sim \bP_{\behavepi, s, a}} \bigE_{\tau \sim F_{\trueG, \trajectoryrand, h+1}} \bigg[2\beta\bigE_{\trajectoryrand' \sim \bP_{\bar \pi, S_\tau, a_\tau}} \sum_{t = \tau}^{H} \min\set{1, \norm{\phi(S_t', A_t')}_{X_t^{-1}}}} + 
            \abs{\bigE_{\trajectoryrand \sim \bP_{\behavepi, s, a}} \bigE_{\tau \sim F_{\trueG, \trajectoryrand, h+1}} (H-\tau+1)\modmisspec\bigg]} \, .
    \end{align*}
    We bound each of the terms in the expectation separately. 
    For the first term, we first recall the definition of $g^{\bar \pi}$ (\cref{eq:gpi definition}) and upper bound the term inside the expectation as follows, which will help us relate things to $\bar \pi$ as we shall see soon. 
    \begin{align*}
        \bigE_{\trajectoryrand' \sim \bP_{\bar \pi, S_\tau, a_\tau}} \sum_{t = \tau}^{H} \min\set{1, \norm{\phi(S_t', A_t')}_{X_t^{-1}}}
        = g^{\bar \pi}(S_\tau, a_\tau)
        \le \max_{a' \in \cA} g^{\bar \pi}(S_\tau, a') \, .
    \end{align*}
    Along with the above result, notice that by the definition of $\bar \pi$ (\cref{eq:optpi elliptical}), $\argmax_{a' \in \cA} g^{\bar \pi}(S_\tau, a')$ is exactly the action $\bar \pi$ would take at the stopping stage $\tau$, 
    and that the distribution of $\trajectoryrand$ under policy $\pi^0$ until stopping stage $\tau$ is same as the distribution of $\trajectoryrand$ under policy $\bar \pi$ until stopping stage $\tau$.
    Thus, 
    \begin{align*}
        &2\beta \abs{\bigE_{\trajectoryrand \sim \bP_{\behavepi, s, a}} \bigE_{\tau \sim F_{\trueG, \trajectoryrand, h+1}} \max_{a' \in \cA} g^{\bar \pi}(S_\tau, a')} \\
        &\le 2\beta \bigE_{\trajectoryrand \sim \bP_{\behavepi, s, a}} \bigE_{\tau \sim F_{\trueG, \trajectoryrand, h+1}} \left[\sum_{t = h+1}^{\tau-1} \min\set{1, \norm{\phi(S_t, A_t)}_{X_t^{-1}}} + \max_{a' \in \cA} g^{\bar \pi}(S_\tau, a')\right] \\
        &= 2 \beta g^{\bar \pi}(s, a)
        = 2\beta \bigE_{\trajectoryrand \sim \bP_{\bar \pi, s, a}} \sum_{t = h+1}^{H} \min\set{1, \norm{\phi(S_t, A_t)}_{X_t^{-1}}} \, .
    \end{align*}
    For the second term, since $\tau \ge h+1$ where $\tau \sim F_{\trueG, \trajectoryrand, h+1}$, it holds that
    \begin{align*}
        \abs{\bigE_{\trajectoryrand \sim \bP_{\behavepi, s, a}} \bigE_{\tau \sim F_{\trueG, \trajectoryrand, h+1}}(H-\tau+1)\modmisspec}
        \le (H-h)\modmisspec \, .             
    \end{align*}
    Plugging the above two bounds into \cref{eq:rho q*G bound}, we get that
    \begin{align}
        \abs{\ip{\phi(s, a), \rho_h^\behavepi\left(\bar v_{\theta_{h+1:H+1}}\right)} -  q^{\pi^\star_{\trueG}}(s, a)} 
        \le 2\beta \bigE_{\trajectoryrand \sim \bP_{\bar \pi, s, a}} \sum_{t = h+1}^{H} \min\set{1, \norm{\phi(S_t, A_t)}_{X_t^{-1}}} + (H-h+1)\modmisspec \, .
        \label{eq:true theta true guess bound}
    \end{align}
    Combining \cref{eq:clipped q bound,eq:theta true guess bound,eq:hat-theta true guess bound,eq:true theta true guess bound}, we get that
    \begin{align*}
        & \abs{\bar q_{\theta_h}(s, a) - q^{\pi^\star_{\trueG}}(s, a)} \\
        &\le \min\set{H, 2\beta \norm{\phi(s, a)}_{X_h^{-1}}
            + 2\beta \bigE_{\trajectoryrand \sim \bP_{\bar \pi, s, a}} \sum_{t = h+1}^{H} \min\set{1, \norm{\phi(S_t, A_t)}_{X_t^{-1}}} + (H-h+1)\modmisspec} \\
        &\le 2\beta \min\set{1, \norm{\phi(s, a)}_{X_h^{-1}}}
            + 2\beta \bigE_{\trajectoryrand \sim \bP_{\bar \pi, s, a}} \sum_{t = h+1}^{H} \min\set{1, \norm{\phi(S_t, A_t)}_{X_t^{-1}}} + (H-h+1)\modmisspec \\
        &\le 2\beta \bigE_{\trajectoryrand \sim \bP_{\bar \pi, s, a}} \sum_{t = h}^{H} \min\set{1, \norm{\phi(S_t, A_t)}_{X_t^{-1}}} + (H-h+1)\modmisspec \, .
    \end{align*}
    The second inequality used that $\beta \ge H$ and that $\min(a, b+c) \le \min(a, b) + c$ for $a, b, c \ge 0$.
\end{proof}

\begin{lemma} \label{lem:expected elliptical potential bound}
    There is an event $\cE_3$, that occurs with probability at least $1 - \delta/3$, such that under event $\cE_3$, for all $h \in [H]$, it holds that
    \begin{align*}
        \bigE_{(S, A) \sim \mu_h} \min\set{1, \norm{\phi(S, A)}_{X_h^{-1}}} \le \check \eps \, .
    \end{align*}
    where $\check \eps$ is defined in \cref{def:check-eps}. 
\end{lemma}
\begin{proof} 
    First we will show two useful results (namely \cref{eq:elliptical potential difference bound} and \cref{eq:high probability expected elliptical bound}) that are needed in the proof. 
    Let 
    \begin{align*}
        \bX = \left\{M \in \bR^{d \times d}: M \text{ is positive semi-definite, and } \lambdamax(M) \le 1/\lambda\right\} \, .
    \end{align*}   
    Notice that any $X \in \bX$ can be written as $X = \sum_{i=1}^d x_i x_i^\top$ with $x_i \in  \cB(1/\lambda)$ for all $i \in [d]$.
    By \cref{lem:cover number ball}, we know there exists a set $C_\xi \subset \cB(a), a, \xi > 0$ with $|C_\xi| = (1 + 2a/\xi)^d$ such that for any $x \in \cB(a)$ there exists a $y \in C_\xi$ such that $\norm{x - y}_2 \le \xi$.
    Define the set
    \begin{align*}
        \bY = \left\{\sum_{i=1}^d y_i y_i^\top: y_i \in C_\xi \text{ for all } i \in [d]\right\} \, ,
    \end{align*}
    with $|\bY| = (1 + 2/(\lambda \xi))^{d^2}$.
    Then, for any $X = \sum_{i=1}^d x_i x_i^\top \in \bX$ there exists a $Y = \sum_{i=1}^d y_i y_i^\top \in \bY$, such that $\norm{x_i - y_i}_2 \le \xi$ for all $i \in [d]$.
    Let $X, Y$ be as we just defined them.
    Then, 
    writing $\opnorm{\cdot}$ for the operator norm,
    \begin{align*}
        \opnorm{X - Y}
        &= \opnorm{\sum_{i=1}^d x_i x_i^\top - y_i y_i^\top} \\
        &= \opnorm{\sum_{i=1}^d (x_i - y_i)(x_i - y_i)^\top + y_i (x_i - y_i)^\top + (x_i - y_i) y_i^\top} \\
        &\le \sum_{i=1}^d \opnorm{(x_i - y_i)(x_i - y_i)^\top} + \opnorm{y_i (x_i - y_i)^\top} + \opnorm{(x_i - y_i) y_i^\top} \\
        &\le \sum_{i=1}^d \norm{(x_i - y_i)}_2\norm{(x_i - y_i)}_2 + \norm{y_i}_2\norm{(x_i - y_i)}_2 + \norm{(x_i - y_i)}_2 \norm{y_i}_2 \\
        &\le \sum_{i=1}^d \xi^2 + \frac{2\xi}{\sqrt{\lambda}}
        = d \xi^2  + \frac{2d \xi}{\sqrt{\lambda}} \, . 
    \end{align*}
    Then, for any $u \in \cB(\featurebound)$
    \begin{align*}
        \left|\norm{u}_X^2 - \norm{u}_Y^2\right|
        = \left|u^\top (X - Y) u\right|
        \le \norm{u}_2^2 \opnorm{X-Y} 
        \le \featurebound^2 \left(d \xi^2  + \frac{2d \xi}{\sqrt{\lambda}}\right) \, ,
    \end{align*}
    which implies that (since for non-negative $a, b, \sqrt{a+b} \le \sqrt{a} + \sqrt{b}$)
    \begin{align*}
        &\norm{u}_X 
        \le \sqrt{\norm{u}_Y^2 + \featurebound^2 \left(d \xi^2  + \frac{2d \xi}{\sqrt{\lambda}}\right)}
        \le \norm{u}_Y + \sqrt{\featurebound^2 \left(d \xi^2  + \frac{2d \xi}{\sqrt{\lambda}}\right)}, \\
        &\norm{u}_Y 
        \le \sqrt{\norm{u}_X^2 + \featurebound^2 \left(d \xi^2  + \frac{2d \xi}{\sqrt{\lambda}}\right)}
        \le \norm{u}_X + \sqrt{\featurebound^2 \left(d \xi^2  + \frac{2d \xi}{\sqrt{\lambda}}\right)} \, .
    \end{align*}
    Thus, for any $u \in \cB(\featurebound)$ 
    \begin{align}
        \left|\norm{u}_X - \norm{u}_Y\right|
        \le \featurebound \sqrt{d \xi^2  + \frac{2d \xi}{\sqrt{\lambda}}} \, . 
        \label{eq:elliptical potential difference bound}
    \end{align}
    \cref{eq:elliptical potential difference bound} is the first useful result that we alluded to at the beginning of the proof.

    Now, we will show the second useful result.
    For any $Y \in \bY$ and $h \in [H]$, define the event
    \begin{align*}
        \cE_3^{Y, h}
        = \Bigg\{\abs{\bigE_{(S, A) \sim \mu_h} \min \left\{1, \norm{\phi(S, A)}_{Y}\right\} 
            - \frac{1}{\nh} \sum_{j \in [\nh]} \min \left\{1, \norm{\phi^j_h}_{Y}\right\}}
        \le \frac{1}{\sqrt{\nh}}\sqrt{\log\left(\frac{6H |\bY|}{\delta}\right)}\Bigg\} \, .
    \end{align*}
    Since $\min \left\{1, \norm{u}_{Y}\right\} \in [0, 1]$ for all $u \in \cB(\featurebound), Y \in \bY$, we can use Hoeffding's inequality (\cref{lem:hoeffdings inequality}) to get that, for any $Y \in \bY, h \in [H]$, event $\cE_3^{Y, h}$ occurs with probability at least $1 - \delta/(3H|\bY|)$.
    Let 
    \begin{align}
        \cE_3 
        = \bigcap_{Y \in \bY, h \in [H]} \cE_3^{Y, h} \, .
        \label{eq:event 3 defintion}
    \end{align}
    Then, by applying a union bound over $Y, h$ we have that the event $\cE_3$ occurs with probability at least $1 - \delta/3$, and under event $\cE_3$, for all $Y \in \bY, h \in [H]$, it holds that
    \begin{align}
        \bigE_{(S, A) \sim \mu_h} \min \left\{1, \norm{\phi(S, A)}_{Y}\right\}
        \le \frac{1}{\nh} \sum_{j \in [\nh]} \min \left\{1, \norm{\phi^j_h}_{Y}\right\}
            + \frac{1}{\sqrt{\nh}}\sqrt{\log\left(\frac{6H |\bY|}{\delta}\right)} \, .
        \label{eq:high probability expected elliptical bound}
    \end{align}
    \cref{eq:high probability expected elliptical bound} is the second useful result that we alluded to at the beginning of the proof.

    Now, we turn to proving \cref{lem:expected elliptical potential bound}.
    Let $h \in [H]$.
    Let $X \in \bX$, and select $Y \in \bY$ such that, for any $u \in \cB(\featurebound)$
    \begin{align}
        \left|\norm{u}_X - \norm{u}_Y\right| 
        \le \featurebound \sqrt{d \xi^2  + \frac{2d \xi}{\sqrt{\lambda}}} \, ,
        \label{eq:result elliptical potential difference bound}
    \end{align}
    which we know exists by \cref{eq:elliptical potential difference bound}.
    By using \cref{eq:result elliptical potential difference bound} we get that 
    \begin{align}
        \bigE_{(S, A) \sim \mu_h} \min \left\{1, \norm{\phi(S, A)}_X\right\}
        \le \bigE_{(S, A) \sim \mu_h} \min \left\{1, \norm{\phi(S, A)}_{Y}\right\} + \featurebound \sqrt{d \xi^2  + \frac{2d \xi}{\sqrt{\lambda}}} \, .
        \label{eq:result expected elliptical potential difference bound}
    \end{align}
    To bound the first term on the RHS in \cref{eq:result expected elliptical potential difference bound} we can use \cref{eq:high probability expected elliptical bound}, to get that under event $\cE_3$
    \begin{align}
        \bigE_{(S, A) \sim \mu_h} \min \left\{1, \norm{\phi(S, A)}_{Y}\right\}
        \le \frac{1}{\nh} \sum_{j \in [\nh]} \min \left\{1, \norm{\phi_h^j}_{Y}\right\} + \frac{1}{\sqrt{n}}\sqrt{\log\left(\frac{6H |\bY|}{\delta}\right)} \, . 
        \label{eq:result high probability expected elliptical bound}
    \end{align}
    We can bound the first term on the RHS of \cref{eq:result high probability expected elliptical bound}, by again using \cref{eq:result elliptical potential difference bound}, to get that
    \begin{align}
        \frac{1}{\nh} \sum_{j \in [\nh]} \min \left\{1, \norm{\phi_h^j}_{Y}\right\}
        \le \frac{1}{\nh} \sum_{j \in [\nh]} \min \left\{1, \norm{\phi_h^j}_{X}\right\} + \featurebound \sqrt{d \xi^2  + \frac{2d \xi}{\sqrt{\lambda}}} \, .
        \label{eq:result average elliptical potential difference bound}
    \end{align}
    Then, by Jensen's inequality we have that
    \begin{align}
        \frac{1}{\nh} \sum_{j \in [\nh]} \min \left\{1, \norm{\phi_h^j}_{X}\right\} 
        = \sqrt{\left(\frac{1}{\nh} \sum_{j \in [\nh]} \min \left\{1, \norm{\phi_h^j}_{X}\right\}\right)^2}   
        \le \sqrt{\frac{1}{\nh} \sum_{j \in [\nh]} \min \left\{1, \norm{\phi_h^j}_{X}^2\right\}} \, .   
        \label{eq:jensen elliptical bound}
    \end{align}
    Putting \cref{eq:result expected elliptical potential difference bound,eq:result high probability expected elliptical bound,eq:result average elliptical potential difference bound,eq:jensen elliptical bound} together and noting that $X, h$ were arbitrary, we get that, under event $\cE_3$, for any $X \in \bX, h \in [H]$, it holds that
    \begin{align*}
        \bigE_{(S, A) \sim \mu_h} \min \left\{1, \norm{\phi(S, A)}_X\right\}
        \le \sqrt{\frac{1}{\nh} \sum_{j \in [\nh]} \min \left\{1, \norm{\phi_h^j}_{X}^2\right\}} + 2\featurebound \sqrt{d \xi^2  + \frac{2d \xi}{\sqrt{\lambda}}} + \frac{1}{\sqrt{\nh}}\sqrt{\log\left(\frac{6H |\bY|}{\delta}\right)} \, .
    \end{align*}
    We can now introduce $X_h$ and make use of the above result.
    Notice that for any $h \in [H], X_h^{-1} = (\lambda I + \sum_{j \in [\nh]} \phi^j_h (\phi^j_h)^\top)^{-1}$ is such that $\lambdamax(X_h^{-1}) \le 1/\lambda$, since $\lambdamin(X_h) \ge \lambda$.
    Thus, $X_h^{-1} \in \bX$.
    For any $t \in [\nh], h \in [H]$, define $X_{t, h} = \lambda I + \sum_{j \in [t]} \phi^j_h (\phi^j_h)^\top$, and notice that $X_{\nh, h}^{-1} = X_h^{-1}$, and that $X_{t, h}^{-1} - X_{\nh, h}^{-1}$ is positive semidefinite.
    This implies that, for all $h \in [H]$
    \begin{align*}
        \frac{1}{\nh} \sum_{j \in [\nh]} \min \left\{1, \norm{\phi_h^j}^2_{X_h^{-1}}\right\}
        \le \frac{1}{\nh} \sum_{j \in [\nh]} \min \left\{1, \norm{\phi_h^j}^2_{X_{j-1, h}^{-1}}\right\} \, .
    \end{align*}
    Now, we can use the elliptical potential lemma (\cref{lem:elliptical potential}), to conclude that, for all $h \in [H]$ 
    \begin{align*}
        \frac{1}{\nh} \sum_{j \in [\nh]} \min \left\{1, \norm{\phi_h^j}^2_{X_{j-1, h}^{-1}}\right\}
        \le \frac{2d}{\nh} \log\left(\frac{d\lambda + \nh\featurebound^2}{d\lambda}\right) \, .
    \end{align*}
    Putting everything together, we get that, under event $\cE_2$, for all $h \in [H]$, it holds that
    \begin{align}
        &\bigE_{(S, A) \sim \mu_h} \min \left\{1, \norm{\phi(S, A)}_{X_h^{-1}}\right\} \nonumber \\
        &\le 2\featurebound \sqrt{d \xi^2  + \frac{2d \xi}{\sqrt{\lambda}}} 
            + \frac{1}{\sqrt{\nh}}\sqrt{\log\left(\frac{6H |\bY|}{\delta}\right)}
            + \sqrt{\frac{2d}{\nh} \log\left(\frac{d\lambda + \nh\featurebound^2}{d\lambda}\right) \nonumber} \\
        &\le 2\featurebound \sqrt{d \xi^2  + \frac{2\xi \modthetabound}{H^{3/2}}}
            + \frac{1}{\sqrt{\nh}}\sqrt{\log\left(\frac{3H (1 + 2\modthetabound^2/\xi)^{d^2}}{\delta}\right)}
            + \sqrt{\frac{2d}{\nh} \log\left(\frac{d\lambda + \nh\featurebound^2}{d\lambda}\right) \nonumber} \\
        &= \frac{\sqrt{d}}{\sqrt{n}} 
            + \frac{1}{\sqrt{\nh}}\sqrt{d^2\log\left(1 + 16 n \featurebound^2\modthetabound^3 \right) + \log\left(\frac{3H}{\delta}\right)}
            + \sqrt{\frac{2d}{\nh} \log\left(\frac{d\lambda + \nh\featurebound^2}{d\lambda}\right) \nonumber} \\
        &= \check \eps 
        = \tilde \cO\left(d/\sqrt{n}\right) \, .
        \label{eq:check-eps definition}
    \end{align}
    The second inequality used that $|\bY| = (1 + 2/(\lambda \xi))^{d^2}$. 
    The first equality holds by setting $\xi^{-1} = 8 \modthetabound \featurebound^2 n$.
    The last equality holds by plugging in parameter values according to \cref{sec:parameter settings}.
\end{proof}

\begin{lemma} \label{lem:conc error bound}
    If \cref{ass:concentrability} holds,
    then for any non-negative function $f: \cS \times \cA \to [0, \infty)$, 
    for any admissible distribution $\nu = (\nu_t)_{t \in [H]}$,
    and for any $h \in [H]$,
    it holds that
    \begin{align*}
        \bigE_{(S, A) \sim \nu_h} f(S, A)
        \le \conc \bigE_{(S, A) \sim \mu_h} f(S, A) \, . 
    \end{align*}
\end{lemma}
\begin{proof}
    Let $f: \cS \times \cA \to [0, \infty)$ be any non-negative function.
    Let $h \in [H]$ be any stage.
    Then,
    \begin{align*}
        \bigE_{(S, A) \sim \nu_h} f(S, A)
        &= \int_{z \in \cS_h \times \cA} f(z) \nu_h(z) dz \\
        &= \int_{z \in \cS_h \times \cA} f(z) \frac{\nu_h(z)}{\mu_h(z)} \mu_h(z) dz \\
        &\le \int_{z \in \cS_h \times \cA} f(z) \conc \mu_h(z) dz \\
        &= \conc \bigE_{(S, A) \sim \mu_h} f(S, A) \, ,
    \end{align*}
    where the inequality holds by applying \cref{ass:concentrability}, and noting that $f$ is non-negative.
    This implies the desired result, since $f$ and $h$ were arbitrary.
\end{proof}

\newpage
\section{Lemmas Related to Least-squares} \label{sec:lemmas related to least squares}

\begin{lemma}[Least-squares Error Decomposition] \label{lem:least squared error decomposition}
    Let $\lambda>0, \theta_\star \in \bR^d$ and $n \in \N^+$, 
    For all $k \in [n]$, let
    \begin{align*}
    &A_k \in \bR^d, \,\,
    \gamma_k \in \bR, \,\,
    \tilde Y_k = \ip{A_k, \theta_\star} + \gamma_k, \,\,
    Y_k = \tilde Y_k+\Delta_k, \\
    &V = \lambda I + \sum_{t=1}^n A_t A_t^\top, \,\,
    \hat\theta = V^{-1}\sum_{t=1}^n A_t Y_t, \,\,
    \iota = \sum_{t=1}^n A_t \gamma_t, \,\,
    \Delta = (\Delta_t)_{t \in [n]} \, .
    \end{align*} 
    Then,
    \begin{align*}
        \norm{\hat \theta - \theta_\star}_{V}
        &\le \sqrt{\lambda}\norm{\theta_\star}_2 + \norm{\Delta}_\infty \sqrt{n} + \norm{\iota}_{V^{-1}} \, .
    \end{align*}    
\end{lemma}
\begin{proof}
    We begin by decomposing the targets used in $\hat \theta$ as follows
    \begin{align*}
        \hat \theta 
        &= V^{-1} \sum_{t=1}^{n} A_t Y_t \\
        &= V^{-1} \sum_{t=1}^{n} A_t \left(\ip{A_t, \theta_\star} + \gamma_t + \Delta_t\right) \\
        &= \left(\lambda I + \sum_{t=1}^n A_t A_t^\top\right)^{-1} \left(\sum_{t=1}^{n} A_t A_t^\top \theta_\star + \lambda I \theta_\star - \lambda I \theta_\star\right) + V^{-1} \sum_{t=1}^{n} A_t \left(\gamma_t + \Delta_t\right) \\
        &= \theta_\star + \lambda V^{-1} \theta_\star + V^{-1} \sum_{t=1}^{n} A_t \gamma_t + V^{-1} \sum_{t=1}^{n} A_t \Delta_t \, .
    \end{align*}
    Then, subtracting $\theta_\star$ from both sides and taking the matrix $V$ weighted norm of both sides gives us that
    \begin{align*}
        \norm{\hat \theta - \theta_\star}_V
        &= \norm{\lambda V^{-1} \theta_\star + V^{-1} \sum_{t=1}^{n} A_t \gamma_t + V^{-1} \sum_{t=1}^{n} A_t \Delta_t}_V \\
        &\le \lambda \norm{\theta_\star}_{V^{-1}} + \norm{\sum_{t=1}^{n} A_t \gamma_t}_{V^{-1}} + \norm{\sum_{t=1}^{n} A_t \Delta_t}_{V^{-1}} \\
        &\le \frac{\lambda}{\lambdamin(V)} \norm{\theta_\star}_2 + \norm{\iota}_{V^{-1}} + \norm{\Delta}_\infty \sqrt{n} \\
        &\le \sqrt{\lambda} \norm{\theta_\star}_2 + \norm{\iota}_{V^{-1}} + \norm{\Delta}_\infty \sqrt{n} \, ,
    \end{align*}
    where the second inequality used that $\norm{V^{-1}} \le \lambdamax(V^{-1}) = 1/\lambdamin(V)$ and \cref{lem:projection bound} to bound $\norm{\sum_{t=1}^{n} A_t \Delta_t}_{V^{-1}}$.
\end{proof}

\begin{lemma}[Least-squares Noise Bound] \label{lem:least squares noise cover bound}
    There is an event $\cE_1$, which occurs with probability at least $1 - \delta/3$, such that under event $\cE_1$, 
    for all $h \in [H]$, 
    for all $G \in \bbG$,
    and $\theta_{h+1:H+1} \in \cB(\modthetabound)^{H-h+1}$,
    it holds that
    \begin{align*}
        &\norm{\sum_{j \in [\nh]} \phi_h^j \left(\bigE_{\tau \sim F^j_{G, h+1}}\left[r^j_{h:\tau-1} + \bar v_{\theta_{h+1:H+1}}\left(s^j_{\tau}\right)\right] 
            - \bigE_{\trajectoryrand \sim \bP_{\behavepi, s_h^j, a_h^j}} \bigE_{\tau \sim F_{G, \trajectoryrand, h+1}} \left[R_{h:\tau-1} + \bar v_{\theta_{h+1:H+1}}(S_\tau)\right]\right)}_{X_h^{-1}}
        \le \bar \beta \, ,
    \end{align*}
    where $\bar \beta$ is defined in \cref{def:bar-beta}.
\end{lemma}
\begin{proof} 
    We begin the proof by showing two useful results (namely \cref{eq:target cover bound} and \cref{eq:high probability least squares noise bound}), which will be needed later in the proof.
    
    By \cref{lem:cover number ball}, we know there exists a set $C_\xi \subset \cB(a), a, \xi > 0$ with $|C_\xi| = (1 + 2a/\xi)^d$ such that for any $x \in \cB(a)$ there exists a $y \in C_\xi$ such that $\norm{x - y}_2 \le \xi$.
    Define the set $C_\xi^\Theta = \bigtimes_{h \in [2:H+1]} C_\xi \subset \cB(\modthetabound)^H$ with $|C_\xi^\Theta| = (1 + 2\modthetabound/\xi))^{dH}$.
    Then, for any $\theta_{2:H+1} \in \cB(\modthetabound)^H$, there exists a $\tilde\theta_{2:H+1} \in C_\xi^\Theta$ such that
    \begin{align*}
    \norm{\theta_h-\tilde\theta_h}_2 \le \xi \quad \text{for all } h\in[2:H+1] \, . 
    \end{align*}
    Which implies that for any $h \in [H], t \in [h+1:H+1], s \in \cS_t$ and $\theta_{h+1:H+1}, \theta^\sim_{h+1:H+1}$ as defined above
    \begin{align}
        \left|\bar v_{\theta_t}(s) - \bar v_{\theta^\sim_t}(s) \right|
        &\le \left|\clip_{[0, H]} \max_{a \in \cA} \ip{\phi(s, a), \theta_t} - \clip_{[0, H]} \max_{a \in \cA} \ip{\phi(s, a), \theta^\sim_t} \right| \nonumber \\
        &\le \left|\max_{a \in \cA} \ip{\phi(s, a), \theta_t - \theta^\sim_t} \right| \nonumber \\
        &\le \max_{a \in \cA} \norm{\phi(s, a)}_2 \norm{\theta_t - \theta^\sim_t}_2 \nonumber \\
        &\le \featurebound \xi \, . 
        \label{eq:value func cover bound}
    \end{align}
    Let $\xi > 0$. 
    Combining \cref{eq:value func cover bound} with \cref{lem:G cover useful results}, we get that there exists a set $C_\xi^\bbG \times C_\xi^\Theta \subset \bbG \times \cB(\modthetabound)^H$ with $|C_\xi^\bbG \times C_\xi^\Theta| \le (1 + 2\usedtobesqrtdoneplusone\modthetabound/\xi))^{dH(d_0+1)}$ such that, 
    for any $(G, \theta_{2:H+1}) \in \bbG \times \cB(\modthetabound)^H$, there exists a $(\tilde G, \theta_{2:H+1}^\sim) \in C_\xi^\bbG \times C_\xi^\Theta$ such that,
    for any $h \in [H]$, 
    for any $u \in [h]$ and trajectory $\trajectory = (s_t, a_t, r_t)_{t \in [u, H+1]}$
    it holds that   
    \begin{align}
        &\left|\bigE_{\tau \sim F_{G, \trajectory, h+1}}\left[r_{h:\tau-1} + \bar v_{\theta_{h+1:H+1}} \left(s_{\tau}\right)\right] 
            - \bigE_{\tau \sim F_{\tilde G, \trajectory, h+1}}\left[r_{h:\tau-1} + \bar v_{\theta^\sim_{h+1:H+1}} \left(s_{\tau}\right)\right] \right| \nonumber \\
        &\le (H-h+1)6 \sqrt{2d} H \featurebound \xi/\alpha + \sum_{t=h}^{H} \featurebound \xi \nonumber \\
        &= (H-h+1)7 \sqrt{2d} H \featurebound \xi/\alpha \, .
        \label{eq:target cover bound}
    \end{align}
    \cref{eq:target cover bound} is the first useful result we alluded to at the beginning of the proof.
    
    Now, we show the second useful result, which is a bound under a high probability event.
    For any $(\tilde G, \theta_{2:H+1}^\sim) \in C_\xi^\bbG \times C_\xi^\Theta$ and $h \in [H]$ define the event
    \begin{align*}
        &\cE_1^{\tilde G, \theta^\sim, h}
        = \left\{\norm{\iota_{\tilde G, \theta^\sim, h}}_{X_h^{-1}} \le \sqrt{2 H^2 \log\left(\frac{3H |C_\xi^\bbG \times C_\xi^\Theta|}{\delta}\right) + \log\left(\sqrt{\frac{\det(X_h)}{\det(\lambda I)}}\right)}\right\} \\
        &\text{where  } \iota_{\tilde G, \theta^\sim, h} 
        = \sum_{j \in [\nh]} \phi_h^j \left(\bigE_{\tau \sim F^j_{\tilde G, h+1}}\left[r^j_{h:\tau-1} + \bar v_{\theta^\sim_{h+1:H+1}} \left(s^j_{\tau}\right)\right] 
            -  \bigE_{\trajectoryrand \sim \bP_{\behavepi, s_h^j, a_h^j}} \bigE_{\tau \sim F_{\tilde G, \trajectoryrand, h+1}} \left[R_{h:\tau-1} + \bar v_{\theta_{h+1:H+1}^\sim}(S_\tau)\right]\right) \, .
    \end{align*}
    Notice that $\iota_{\tilde G, \theta^\sim, h}$ is $H$-subgaussian.
    Thus, we can use Theorem 1 from \citep{abbasi2011improved} to get that the event $\cE_1^{\tilde G, \theta^\sim, h}$ occurs with probability at least $1- \delta/(3H |C_\xi^\bbG \times C_\xi^\Theta|)$.
    Define 
    \begin{align*}
        \cE_1 = \bigcap\limits_{(\tilde G, \theta^\sim) \in C_\xi^\bbG \times C_\xi^\Theta, h \in [H]} \cE_1^{\tilde G, \theta^\sim, h} \, .
    \end{align*}
    Then, by applying a union bound over $\tilde G, \theta^\sim, h$ we have that the event $\cE_1$ occurs with probability at least $1 - \delta/3$, and under event $\cE_1$, for all $(\tilde G, \theta^\sim) \in C_\xi^\bbG \times C_\xi^\Theta, h \in [H]$, it holds that
    \begin{align}
        &\norm{\sum_{j \in [\nh]} \phi_h^j \left(\bigE_{\tau \sim F^j_{\tilde G, h+1}}\left[r^j_{h:\tau-1} + \bar v_{\theta^\sim_{h+1:H+1}} \left(s^j_{\tau}\right)\right] 
            -  \bigE_{\trajectoryrand \sim \bP_{\behavepi, s_h^j, a_h^j}} \bigE_{\tau \sim F_{\tilde G, \trajectoryrand, h+1}} \left[R_{h:\tau-1} + \bar v_{\theta_{h+1:H+1}^\sim}(S_\tau)\right]\right)}_{X_h^{-1}} \nonumber \\
        &\le \sqrt{2 H^2 \log\left(\frac{3H (1 + 2\usedtobesqrtdoneplusone\modthetabound/\xi))^{dH(d_0+1)}}{\delta}\right) + \log\left(\sqrt{\frac{\det(X_h)}{\det(\lambda I)}}\right)} \, . 
        \label{eq:high probability least squares noise bound}
    \end{align}
    
    Now with all of the above results in hand, we turn to finally proving \cref{lem:least squares noise cover bound}.
    Let $G \in \bbG$ as in the lemma statement.
    Then, by \cref{eq:target cover bound} we know that there exists a $(\tilde G, \theta_{2:H+1}^\sim) \in C_\xi^\bbG \times C_\xi^\Theta$ such that,
    for any $h \in [H]$, 
    for any $u \in [h]$ and trajectory $\trajectory = (s_t, a_t, r_t)_{t \in [u:H+1]}$, 
    it holds that 
    \begin{align}
        \left|\bigE_{\tau \sim F_{G, \trajectory, h+1}}\left[r_{h:\tau-1} + \bar v_{\theta_{h+1:H+1}} \left(s_{\tau}\right)\right] 
            - \bigE_{\tau \sim F_{\tilde G, \trajectory, h+1}}\left[r_{h:\tau-1} + \bar v_{\theta^\sim_{h+1:H+1}} \left(s_{\tau}\right)\right] \right|
        \le 7 \sqrt{2d} H^2 \featurebound \xi/\alpha \, .
        \label{eq:target cover bound final}
    \end{align}
    Let $\tilde G, \theta_{2:H+1}^\sim$ be as defined above.
    Then, for any $h \in [H]$, by using the triangle inequality we can write
    \begin{align*}
        &\norm{\sum_{j \in [\nh]} \phi_h^j \left(\bigE_{\tau \sim F^j_{G, h+1}}\left[r^j_{h:\tau-1} + \bar v_{\theta_{h+1:H+1}}\left(s^j_{\tau}\right)\right] 
            - \bigE_{\trajectoryrand \sim \bP_{\behavepi, s_h^j, a_h^j}} \bigE_{\tau \sim F_{G, \trajectoryrand, h+1}} \left[R_{h:\tau-1} + \bar v_{\theta_{h+1:H+1}}(S_\tau)\right]\right)}_{X_h^{-1}} \\
        &\le \norm{\sum_{j \in [\nh]} \phi_h^j \left(\bigE_{\tau \sim F^j_{\tilde G, h+1}}\left[r^j_{h:\tau-1} + \bar v_{\theta^\sim_{h+1:H+1}}\left(s^j_{\tau}\right)\right] 
            - \bigE_{\trajectoryrand \sim \bP_{\behavepi, s_h^j, a_h^j}} \bigE_{\tau \sim F_{\tilde G, \trajectoryrand, h+1}} \left[R_{h:\tau-1} + \bar v_{\theta_{h+1:H+1}^\sim}(S_\tau)\right]\right)}_{X_h^{-1}} \\
        &\quad + \norm{\sum_{j \in [\nh]} \phi_h^j \left(\bigE_{\tau \sim F^j_{G, h+1}}\left[r^j_{h:\tau-1} + \bar v_{\theta_{h+1:H+1}} \left(s^j_{\tau}\right)\right] 
            - \bigE_{\tau \sim F^j_{\tilde G, h+1}}\left[r^j_{h:\tau-1} + \bar v_{\theta^\sim_{h+1:H+1}} \left(s^j_{\tau}\right)\right]\right)}_{X_h^{-1}} \\
        &\quad + \norm{\sum_{j \in [\nh]} \phi_h^j \left(\bigE_{\trajectoryrand \sim \bP_{\behavepi, s_h^j, a_h^j}} \sqbraces{\bigE_{\tau \sim F_{\tilde G, \trajectoryrand, h+1}} \left[R_{h:\tau-1} + \bar v_{\theta_{h+1:H+1}^\sim}(S_\tau)\right] 
            - \bigE_{\tau \sim F_{G, \trajectoryrand, h+1}} \left[R_{h:\tau-1} + \bar v_{\theta_{h+1:H+1}}(S_\tau)\right]}\right)}_{X_h^{-1}} \, .
    \end{align*}
    The first term can be bounded by \cref{eq:high probability least squares noise bound}, if we are under event $\cE_1$.
    For the second and third term we make use of \cref{lem:projection bound}, which ensures that for any sequence $(b_j)_{j \in [\nh]}$ such that $|b_j| \le c \in \bR$ the following holds
    \begin{align*}
        \norm{\sum_{j \in [\nh]} \phi_h^j b_j}_{X_h^{-1}} 
        \le c \sqrt{\nh} \, .
    \end{align*}
    For the second term and third term the respective $b_j$ terms can be bounded by using \cref{eq:target cover bound final}, giving us $c = 7 \sqrt{2d} H^2 \featurebound \xi/\alpha$.
    
    Putting the above three bounds together we have that under event $\cE_1$, which occurs with probability at least $1 - \delta/3$, 
    for all $h \in [H]$, 
    for all $G \in \bbG$,
    and $\theta_{h+1:H+1} \in \cB(\modthetabound)^{H-h+1}$,
    it holds that
    \begin{align*}
        &\norm{\sum_{j \in [\nh]} \phi_h^j \left(\bigE_{\tau \sim F^j_{G, h+1}}\left[r^j_{h:\tau-1} + \bar v_{\theta_{h+1:H+1}}\left(s^j_{\tau}\right)\right] 
            - \bigE_{\trajectoryrand \sim \bP_{\behavepi, s_h^j, a_h^j}} \bigE_{\tau \sim F_{G, \trajectoryrand, h+1}} \left[R_{h:\tau-1} + \bar v_{\theta_{h+1:H+1}}(S_\tau)\right]\right)}_{X_h^{-1}} \\
        &\le \sqrt{2 H^2 \log\left(\frac{3H (1 + 2\usedtobesqrtdoneplusone\modthetabound/\xi)^{dH(d_0+1)}}{\delta}\right) + \log\left(\sqrt{\frac{\det(X_h)}{\det(\lambda I)}}\right)} + 14 \sqrt{\nh} \sqrt{2d} H^2 \featurebound \xi/\alpha \\
        &= H \sqrt{2dH(d_0+1)\log\left(1 + 2\usedtobesqrtdoneplusone\modthetabound/\xi\right) + \log(\det(X_h)) - d \log(\lambda) + \log\left(\frac{3H}{\delta}\right)} + 14\sqrt{\nh} \sqrt{2d} H^2 \featurebound \xi/\alpha \\
        &\le H \sqrt{2dH(d_0+1)\log\left(1 + 2\usedtobesqrtdoneplusone\modthetabound/\xi\right) + d \log (\lambda + \nh \featurebound^2/d) - d \log(\lambda) + \log\left(\frac{3H}{\delta}\right)} + 14\sqrt{\nh} \sqrt{2d} H^2 \featurebound \xi/\alpha \, , 
    \end{align*}
    where the last inequality used the Determinant-Trace Inequality (see Lemma 10 in \citep{abbasi2011improved}).
    Setting $\xi^{-1} = 14 \sqrt{n} \sqrt{2d} H^2 \featurebound\alpha^{-1}$,
    we get that the above display is
    \begin{align}
        &\le H \sqrt{2dH(d_0+1)\log\left(1 + 28\sqrt{2d} H^2 \usedtobesqrtdoneplusone \modthetabound \featurebound \alpha^{-1}\right) + d \log (\lambda + \nh \featurebound^2/d) - d \log(\lambda) + \log\left(\frac{3H}{\delta}\right)} \nonumber + 1 \nonumber \\
        &= \bar \beta
        = \tilde \cO \left(H^{3/2} d\right) \, . 
        \label{eq:bar-beta definition}
    \end{align}
    The last equality holds by plugging in parameter values according to \cref{sec:parameter settings}.
\end{proof}

\newpage
\section{Lemmas Related to Covering $\bbG$} \label{sec:useful results}

\begin{lemma} \label{lem:opt prob cond concentrates}
    There is an event $\cE_2$, 
    that occurs with probability at least $1 - \delta/3$, 
    such that under event $\cE_2$, 
    for all $G \in \bbG$,
    and for all $h \in [H]$, 
    it holds that 
    \begin{align*}
        &\abs{\bigE_{(S, A) \sim \mu_h} \left[\max_{\theta \in \Theta_{G, h}} \bar q_{\theta}(S, A) - \min_{\theta \in \Theta_{G, h}} \bar q_{\theta}(S, A)\right]
            - \frac{1}{\nh} \sum_{i \in [\nh]} \left(\max_{\theta \in \Theta_{G, h}} \bar q_{\theta}(s_h^i, a_h^i) - \min_{\theta \in \Theta_{G, h}} \bar q_{\theta}(s_h^i, a_h^i)\right)} \\
        &\le \frac{H}{\sqrt{\nh}} \sqrt{\log\left(\frac{6H|C_\xi^\bbG|}{\delta}\right)} 
            + 2\modxi \, , 
    \end{align*}
    where $|C_\xi^\bbG|, \modxi$, are defined in \cref{def:G cover size,def:modxi}.
\end{lemma}
\begin{proof}
    Let $G \in \bbG$ be a feasible solution to \cref{opt:bpse start state}. 
    By the first result in \cref{lem:G cover q-max-min bound},
    there exists a set $C_\xi^\bbG \subset \bbG$ such that, there exists a $\tilde G \in C_\xi^\bbG$ such that, for any $h \in [H], (s, a) \in \cS_h \times \cA$, it holds that
    \begin{align}
        \abs{\left(\max_{\theta \in \Theta_{G, h}} \bar q_{\theta}(s, a) - \min_{\theta \in \Theta_{G, h}} \bar q_{\theta}(s, a)\right)
        - \left(\max_{\theta^\sim \in \Theta_{\tilde G, h}} \bar q_{\theta^\sim}(s, a) - \min_{\theta \in \Theta_{\tilde G, h}} \bar q_{\theta^\sim}(s, a)\right)}
        \le \modxi \, .
        \label{eq:result q-max-min cover bound}
    \end{align}
    Select $\tilde G$ as defined above.
    Let $h \in [H]$.
    Using \cref{eq:result q-max-min cover bound}, we know that
    \begin{align}
        \abs{\bigE_{(S, A) \sim \mu_h} \left[\max_{\theta \in \Theta_{G, h}} \bar q_{\theta}(S, A) - \min_{\theta \in \Theta_{G, h}} \bar q_{\theta}(S, A)\right]
            - \bigE_{(S, A) \sim \mu_h} \left[\max_{\theta \in \Theta_{\tilde G, h}} \bar q_{\theta^\sim}(S, A) - \min_{\theta \in \Theta_{\tilde G, h}} \bar q_{\theta^\sim}(S, A)\right]} 
        \le \modxi \, .
        \label{eq:result q-max-min close in expecation}
    \end{align}
    To bound the second term in the absolute value of \cref{eq:result q-max-min close in expecation} to its empirical mean, we can use the second result in \cref{lem:G cover q-max-min bound}, which gives us that under event $\cE_2$
    \begin{align}
        &\Bigg|\bigE_{(S, A) \sim \mu_h} \left[\max_{\theta^\sim \in \Theta_{\tilde G, h}} \bar q_{\theta^\sim}(S, A) - \min_{\theta^\sim \in \Theta_{\tilde G, h}} \bar q_{\theta^\sim}(S, A)\right] 
            - \frac{1}{\nh} \sum_{i \in [\nh]} \left(\max_{\theta^\sim \in \Theta_{\tilde G, h}} \bar q_{\theta^\sim}(s_h^i, a_h^i) - \min_{\theta^\sim \in \Theta_{\tilde G, h}} \bar q_{\theta^\sim}(s_h^i, a_h^i)\right)\Bigg| \nonumber \\
        &\le\frac{H}{\sqrt{\nh}} \sqrt{\log\left(\frac{6H|C_\xi^\bbG|}{\delta}\right)} \, .
       \label{eq:result q-max-min concentration}
    \end{align}
    We can relate the $\tilde G$ in the second term of the absolute value in \cref{eq:result q-max-min concentration} back to $G$, by once again using \cref{eq:result q-max-min cover bound}, to get that 
    \begin{align}
        &\abs{\frac{1}{\nh} \sum_{i \in [\nh]} \left(\max_{\theta^\sim \in \Theta_{\tilde G, h}} \bar q_{\theta^\sim}(s_h^i, a_h^i) - \min_{\theta^\sim \in \Theta_{\tilde G, h}} \bar q_{\theta^\sim}(s_h^i, a_h^i)\right) 
            - \frac{1}{\nh} \sum_{i \in [\nh]} \left(\max_{\theta \in \Theta_{G, h}} \bar q_{\theta}(s_h^i, a_h^i) - \min_{\theta \in \Theta_{G, h}} \bar q_{\theta}(s_h^i, a_h^i)\right)} \nonumber \\
        &\le \modxi \, . 
        \label{eq:result q-max-min close on average}
    \end{align}
    Putting together \cref{eq:result q-max-min close in expecation,eq:result q-max-min concentration,eq:result q-max-min close on average}, and noting that $h$ was arbitrary, gives the desired result.
\end{proof}

\begin{lemma} \label{lem:G cover q-max-min bound}
    Let $\xi > 0$.
    There exists a set $C_\xi^\bbG \subset \bbG$ such that, for any $G \in \bbG$, there exists a $\tilde G \in C_\xi^\bbG$ such that, for any $h \in [H], (s, a) \in \cS_h \times \cA$, it holds that
    \begin{align*}
        \abs{\left(\max_{\theta \in \Theta_{G, h}} \bar q_{\theta}(s, a) - \min_{\theta \in \Theta_{G, h}} \bar q_{\theta}(s, a)\right)
        - \left(\max_{\theta^\sim \in \Theta_{\tilde G, h}} \bar q_{\theta^\sim}(s, a) - \min_{\theta^\sim \in \Theta_{\tilde G, h}} \bar q_{\theta^\sim}(s, a)\right)}
        \le \modxi \, ,
    \end{align*}
    where $|C_\xi^\bbG|, \modxi$, are defined in \cref{def:G cover size,def:modxi}.
    Furthermore, there is an event $\cE_2$, which occurs with probability at least $1 - \delta/3$, such that under event $\cE_2$, 
    for any $\tilde G \in C_\xi^\bbG$, 
    and $h \in [H]$, 
    it holds that
    \begin{align*}
        &\Bigg|\bigE_{(S, A) \sim \mu_h} \left[\max_{\theta^\sim \in \Theta_{\tilde G, h}} \bar q_{\theta^\sim}(S, A) - \min_{\theta^\sim \in \Theta_{\tilde G, h}} \bar q_{\theta^\sim}(S, A)\right] 
            - \frac{1}{\nh} \sum_{i \in [\nh]} \left(\max_{\theta^\sim \in \Theta_{\tilde G, h}} \bar q_{\theta^\sim}(s_h^i, a_h^i) - \min_{\theta^\sim \in \Theta_{\tilde G, h}} \bar q_{\theta^\sim}(s_h^i, a_h^i)\right)\Bigg| \\
        &\le\frac{H}{\sqrt{\nh}} \sqrt{\log\left(\frac{6H|C_\xi^\bbG|}{\delta}\right)} \, .
    \end{align*}
\end{lemma}
\begin{proof}
    Let $\xi > 0, \kappa_t \ge 0, \forall t \in [2:H+1]$. 
    By \cref{lem:G cover useful results}, there exists a set $C_\xi^\bbG \subset \bbG$ with $|C_\xi^\bbG| \le (1 + 2\usedtobesqrtdoneplusone/\xi))^{dHd_0}$ such that, for any $G \in \bbG$, there exists a $\tilde G \in C_\xi^\bbG$ such that, 
    for any $h \in [H]$,
    for any $\theta_{h+1:H+1}, \theta^\sim_{h+1:H+1} \in \cB(\modthetabound)^{H-h+1}$, such that for all $t \in [h+1:H+1], s \in \cS_t, \left|\bar v_{\theta_t}(s) - \bar v_{\theta^\sim_t}(s) \right| \le \kappa_t$, 
    and for any $j \in [\nh]$,
    it holds that  
    \begin{align}
        &\left|\bigE_{\tau \sim F^j_{G, h+1}}\left[r_{h:\tau-1} + \bar v_{\theta_{h+1:H+1}} \left(s_{\tau}\right)\right] 
            - \bigE_{\tau \sim F^j_{\tilde G, h+1}}\left[r_{h:\tau-1} + \bar v_{\theta^\sim_{h+1:H+1}} \left(s_{\tau}\right)\right] \right| \nonumber \\
        &\le (H-h+1)6 \sqrt{2d} H \featurebound \xi/\alpha + \sum_{t=h}^{H} \kappa_{t+1} \, .
        \label{eq:targets close under cover}
    \end{align}
    For the remainder of the proof let $G, \tilde G$ be as described above.

    We first show the following intermediate result.
    \begin{lemma}
        For all $h \in [H]$, it holds that:
        \begin{enumerate}
            \item 
                For any $\hat \theta_h \in \hat \Theta_{G, h}, \theta_h \in \Theta_{G, h}$, there exists $\hat \theta_h^\sim \in \hat \Theta_{\tilde G, h}, \theta_h^\sim \in \Theta_{\tilde G, h}$ such that 
                \begin{align}
                    \norm{\hat \theta_h - \hat \theta_h^\sim}_{X_h}
                    \le c_h^\xi, 
                    \quad \norm{\theta_h - \theta_h^\sim}_{X_h}
                    \le c_h^\xi \, . 
                \label{eq:Theta cover inductive hypothesis forward}
                \end{align}
            \item
                For any $\hat \theta_h^\sim \in \hat \Theta_{\tilde G, h}, \theta_h^\sim \in \Theta_{\tilde G, h}$, there exists $\hat \theta_h \in \hat \Theta_{G, h}, \theta_h \in \Theta_{G, h}$ such that 
                \begin{align}
                    \norm{\hat \theta_h - \hat \theta_h^\sim}_{X_h}
                    \le c_h^\xi, 
                    \quad \norm{\theta_h - \theta_h^\sim}_{X_h}
                    \le c_h^\xi \, , 
                \label{eq:Theta cover inductive hypothesis backward}
                \end{align}
        \end{enumerate}
        where, 
        \begin{align}
            c_h^\xi = 6 \sqrt{\nh} \sqrt{2d} H^2 \featurebound \xi \alpha^{-1} \left(1 + \sqrt{\nh} \featurebound \modthetabound / (H^{3/2}d)\right)^{H-h} \, .
            \label{eq:c_h definition}
        \end{align}
    \end{lemma}
    \begin{proof}
        \textbf{Proof of result $1.$:}
        To show \cref{eq:Theta cover inductive hypothesis forward} we will use induction.
        The base case is when $h = H$, for which
        \begin{align*}
             \hat \Theta_{G, H} = \hat \Theta_{\tilde G, H} 
             = \left\{X_H^{-1} \sum_{j \in [\nh]} \phi^j_H r^j_{H} \right\}, 
             \implies \Theta_{G, H} = \Theta_{\tilde G, H} \, .
        \end{align*}
        Thus, for any $\hat \theta_H \in \hat \Theta_{G, H}, \theta_H \in \Theta_{G, H}$, select $\hat \theta_H^\sim = \hat \theta_H \in \hat \Theta_{\tilde G, H}, \theta_H^\sim = \theta_H \in \Theta_{\tilde G, H}$.
        Then
        \begin{align*}
            \norm{\hat \theta_H - \hat \theta_H^\sim}_{X_H}
            \le 0, 
            \quad \norm{\theta_H - \theta_H^\sim}_{X_H}
            \le 0 \, .
        \end{align*}
        Now, we show the inductive step. 
        Let $h \in [H-1]$ be arbitrary. Assume \cref{eq:Theta cover inductive hypothesis forward} holds for any $t \in [h+1:H]$. 
        We prove that \cref{eq:Theta cover inductive hypothesis forward} also holds for $h$.
        Let $\hat \theta_h \in \hat \Theta_{G, h}$ be arbitrary.
        Notice that $\hat \theta_h$ must have the following form.
        \begin{align*}
            \hat \theta_h = X_h^{-1} \sum_{j \in [\nh]} \phi^j_h \bigE_{\tau \sim F^j_{G, h+1}}\left[r^j_{h:\tau-1} + \bar v_{\theta_{h+1:H+1}} \left(s^j_{\tau}\right)\right] \in \hat \Theta_{G, h}, \,\, \text{for some } \theta_{h+1:H+1} \in \Theta_{G, h+1} \times \dots \times \Theta_{G, H+1} \, .
        \end{align*}
        Select $\theta_{h+1:H+1}^\sim \in \Theta_{\tilde G, h+1} \times \dots \times \Theta_{\tilde G, H+1}$ such that for all $t \in [h+1:H+1]$
        \begin{align}
            \norm{\theta_t - \theta_t^\sim}_{X_t}
            \le c_t^\xi \, ,
            \label{eq:future theta cover bound}
        \end{align}
        which exists by the inductive hypothesis (\cref{eq:Theta cover inductive hypothesis forward}) and since $\Theta_{G, H+1} = \Theta_{\tilde G, H+1} = \{\vec 0\}$.
        Define
        \begin{align*}
            \hat \theta_h^\sim = X_h^{-1} \sum_{j \in [\nh]} \phi^j_h \bigE_{\tau \sim F^j_{\tilde G, h+1}}\left[r^j_{h:\tau-1} + \bar v_{\theta_{h+1:H+1}^\sim} \left(s^j_{\tau}\right)\right] \in \hat \Theta_{\tilde G, h} \, . 
        \end{align*}
        Recall that we aim to bound $\norm{\hat \theta_h - \hat \theta_h^\sim}_{X_h}$.
        Plugging in the expressions for $\hat \theta_h, \hat \theta_h^\sim$, as defined above, we get that
        \begin{align*}
            \norm{\hat \theta_h - \hat \theta_h^\sim}_{X_h}
            &= \norm{X_h^{-1} \sum_{j \in [\nh]} \phi^j_h \left(\bigE_{\tau \sim F^j_{G, h+1}}\left[r^j_{h:\tau-1} + \bar v_{\theta_{h+1:H+1}} \left(s^j_{\tau}\right)\right] - \bigE_{\tau \sim F^j_{\tilde G, h+1}}\left[r^j_{h:\tau-1} + \bar v_{\theta_{h+1:H+1}^\sim} \left(s^j_{\tau}\right)\right]\right)}_{X_h} \\
            &= \norm{X_h^{-1} \sum_{j \in [\nh]} \phi^j_h b_j}_{X_h}
            = \norm{\sum_{j \in [\nh]} \phi^j_h b_j}_{X_h^{-1}} \, ,
        \end{align*}
        where in the second equality we let $b_j = \bigE_{\tau \sim F^j_{G, h+1}}\left[r^j_{h:\tau-1} + \bar v_{\theta_{h+1:H+1}} \left(s^j_{\tau}\right)\right] - \bigE_{\tau \sim F^j_{\tilde G, h+1}}\left[r^j_{h:\tau-1} + \bar v_{\theta_{h+1:H+1}^\sim} \left(s^j_{\tau}\right)\right]$.
        To bound the above term we can make use of \cref{lem:projection bound}, which ensures that for any sequence $(b_j)_{j \in [\nh]}$ such that $|b_j| \le a \in \bR$ the following holds
        \begin{align*}
            \norm{\sum_{j \in [\nh]} \phi_h^j b_j}_{X_h^{-1}} 
            \le a \sqrt{\nh} \, .
        \end{align*}
        Thus, we are left to bound $|b_j|$.
        To do so, we can make use of \cref{eq:targets close under cover}, which requires us to bound $\left|\bar v_{\theta_t}(s) - \bar v_{\theta_t^\sim}(s) \right|$ for all $t \in [h+1:H], s \in \cS_t$, which can be done as follows.
        \begin{align}
            \left|\bar v_{\theta_t}(s) - \bar v_{\theta_t^\sim}(s) \right|
            &= \left|\clip_{[0, H]} \max_{a \in \cA} \ip{\phi(s, a), \theta_t} - \clip_{[0, H]} \max_{a \in \cA} \ip{\phi(s, a), \theta_t^\sim} \right| \nonumber \\
            &\le \left|\max_{a \in \cA} \ip{\phi(s, a), \theta_t - \theta_t^\sim} \right| \nonumber \\
            &\le \max_{a \in \cA} \norm{\phi(s, a)}_{X_t^{-1}} \norm{\theta_t - \theta_t^\sim}_{X_t} \nonumber \\
            &\le \frac{\featurebound}{\sqrt{\lambda}} c_t^\xi
            = \featurebound\modthetabound c_t^\xi/(H^{3/2} d) \, .
            \label{eq:event-2 value func cover bound}
        \end{align}
        The second inequality uses the Cauchy-Schwarz inequality.
        The third inequality used \cref{def:feature bound}, that $\lambdamax(X_h^{-1}) \le 1/\lambda$ (by definition of $X_h$ \cref{def:covariance matrix}), and \cref{eq:future theta cover bound}.
        The last equality used that $\sqrt{\lambda} = H^{3/2} d / \modthetabound$ (\cref{def:lambda}).
        Plugging \cref{eq:event-2 value func cover bound} into \cref{eq:targets close under cover} we get that for any $j \in [\nh]$
        \begin{align*}
            \abs{b_j}
            &= \left|\bigE_{\tau \sim F^j_{G, h+1}}\left[r_{h:\tau-1} + \bar v_{\theta_{h+1:H+1}} \left(s_{\tau}\right)\right] 
                - \bigE_{\tau \sim F^j_{\tilde G, h+1}}\left[r_{h:\tau-1} + \bar v_{\theta^\sim_{h+1:H+1}} \left(s_{\tau}\right)\right] \right| \\
            &\le (H-h+1)6 \sqrt{2d} H \featurebound \xi/\alpha + \sum_{t=h}^{H} \featurebound\modthetabound c_{t+1}^\xi/(H^{3/2} d) \, .
        \end{align*}       
        Thus,
        \begin{align*}
            \norm{\hat \theta_h - \hat \theta_h^\sim}_{X_h}
            = \norm{\sum_{j \in [\nh]} \phi_h^j b_j}_{X_h^{-1}} 
            &\le \left(6 \sqrt{2d} H^2 \featurebound \xi \alpha^{-1} +  \frac{\featurebound\modthetabound}{H^{3/2}d} \sum_{t=h}^{H} c_{t+1}^\xi\right) \sqrt{\nh} \\
            &= 6 \sqrt{\nh} \sqrt{2d} H^2 \featurebound \xi \alpha^{-1} +  \sqrt{\nh}\frac{\featurebound\modthetabound}{H^{3/2}d} \sum_{t=h}^{H} c_{t+1}^\xi \\
            &= 6 \sqrt{\nh} \sqrt{2d} H^2 \featurebound \xi \alpha^{-1} \left(1 + \sqrt{\nh}\featurebound\modthetabound/(H^{3/2}d)\right)^{H-h} \\
            &= c_h^\xi \, .
        \end{align*}
        To see why the second last equality is true, let $x = 6 \sqrt{\nh} \sqrt{2d} H^2 \featurebound \xi \alpha^{-1}$ and $y = \sqrt{\nh} \featurebound \modthetabound / (H^{3/2}d)$.
        Then, for any $t \in [h:H], c_t^\xi = x(1+y)^{H-t}$ (by \cref{eq:c_h definition}) and 
        \begin{align*}
            6 \sqrt{\nh} \sqrt{2d} H^2 \featurebound \xi \alpha^{-1} +  \sqrt{\nh}\frac{\featurebound\modthetabound}{H^{3/2}d} \sum_{t=h}^{H} c_{t+1}^\xi
            = x + y \sum_{t=h+1}^H x (1+y)^{H-t} \, .
        \end{align*}
        Notice that the sum can be rewritten as a finite geometric series.
        \begin{align*}
            \sum_{t=h+1}^{H} x (1+y)^{H-t} 
            = x \sum_{k=0}^{H-h-1} (1+y)^k
            = x \frac{(1+y)^{H-h} - 1}{y} \, .
        \end{align*}
        Thus,
        \begin{align*}
            x + y \sum_{t=h+1}^H x (1+y)^{H-t}
            = x + y \cdot x \frac{(1 + y)^{H-h} - 1}{y} 
            = x (1 + y)^{H-h}
            = c_h^\xi \, .
        \end{align*}
        This proves the first result in \cref{eq:Theta cover inductive hypothesis forward}.

        Next, we show the second result in \cref{eq:Theta cover inductive hypothesis forward}.
        Let $\theta_h \in \Theta_{G, h}$.
        By the definition of the set $\Theta_{G, h}$, there exists a $\hat \theta_h \in \hat \Theta_{G, h}$ such that $\norm{\theta_h - \hat \theta_h}_{X_h} \le \beta$. 
        Then, by the first result in \cref{eq:Theta cover inductive hypothesis forward} (which we have shown holds for $h$ above), there exists a $\hat \theta_h^\sim \in \hat \Theta_{\tilde G, h}$, such that $\norm{\hat \theta_h - \hat \theta_h^\sim}_{X_h} \le c_h^\xi$.
        Let $\hat \theta_h, \hat \theta_h^\sim$ be as defined above, and select $\theta_h^\sim = \theta_h - \hat \theta_h + \hat \theta_h^\sim$, which is an element of $\Theta_{\tilde G, h}$ since 
        \begin{align*}
            \norm{\theta_h^\sim - \hat \theta_h^\sim}_{X_h} 
            = \norm{\theta_h - \hat \theta_h + \hat \theta_h^\sim - \hat \theta_h^\sim}_{X_h} 
            = \norm{\theta_h - \hat \theta_h}_{X_h} 
            \le \beta \, .
        \end{align*}
        Then,
        \begin{align*}
            \norm{\theta_h - \theta_h^\sim}_{X_h}
            = \norm{\theta_h - \theta_h - \hat \theta_h + \hat \theta_h^\sim}_{X_h}
            = \norm{\hat \theta_h^\sim - \hat \theta_h}_{X_h}
            \le c_h^\xi \, .
        \end{align*}
        This completes the proof of the second result in \cref{eq:Theta cover inductive hypothesis forward}.
        
        \textbf{Proof of result $2.$:}
        The proof is identical to that of the proof of result $1.$, except swapping the roles of $\hat \theta_h, \theta_h, G$ and $\hat \theta^\sim_h, \theta^\sim_h, \tilde G$. 
    \end{proof}

    Now, with the results of \cref{eq:Theta cover inductive hypothesis forward,eq:Theta cover inductive hypothesis backward} in hand, we return to proving \cref{lem:G cover q-max-min bound}.
    For any $h \in [H]$ let $\theta_h = \argmax_{\theta \in \Theta_{G, h}} \ip{\phi(s, a), \theta}$ then, by \cref{eq:Theta cover inductive hypothesis forward}, there exists a $\theta_h^\sim \in \Theta_{\tilde G, h}$ such that $\norm{\theta_h - \theta_h^\sim}_{X_t} \le c_h^\xi$.
    This gives that, for all $h \in [H], (s, a) \in \cS_h \times \cA$
    \begin{align*}
        \max_{\theta \in \Theta_{G, h}} \ip{\phi(s, a), \theta} - \max_{\theta^\sim \in \Theta_{\tilde G, h}} \ip{\phi(s, a), \theta^\sim}
        &= \ip{\phi(s, a), \theta_h} - \ip{\phi(s, a), \theta_h^\sim} + \ip{\phi(s, a), \theta_h^\sim} - \max_{\theta^\sim \in \Theta_{\tilde G, h}} \ip{\phi(s, a), \theta^\sim} \\
        &\le \norm{\phi(s, a)}_{X_h^{-1}} \norm{\theta_h - \theta_h^\sim}_{X_h} \\
        &\le \frac{\featurebound}{\sqrt{\lambda}} c_h^\xi
        = \featurebound \modthetabound c_h^\xi / (H^{3/2}d) \, .
    \end{align*}
    The second inequality used \cref{def:feature bound}, that $\lambdamax(X_h^{-1}) \le 1/\lambda$ (by definition of $X_h$ \cref{def:covariance matrix}), and \cref{eq:future theta cover bound}.
    The last equality used that $\sqrt{\lambda} = H^{3/2} d / \modthetabound$ (\cref{def:lambda}).
    Now, for the other direction, using similar steps as above, for any $h \in [H]$, let $\theta_h^\sim = \argmax_{\theta^\sim \in \Theta_{\tilde G, h}} \ip{\phi(s, a), \theta^\sim}$ then, by \cref{eq:Theta cover inductive hypothesis backward}, there exists a $\theta_h \in \Theta_{G, h}$ such that $\norm{\theta_h - \theta_h^\sim}_{X_h} \le c_h^\xi$.
    This gives that, for all $h \in [H], (s, a) \in \cS_h \times \cA$   
    \begin{align*}
        \max_{\theta^\sim \in \Theta_{\tilde G, h}} \ip{\phi(s, a), \theta^\sim} - \max_{\theta \in \Theta_{G, h}} \ip{\phi(s, a), \theta}
        &= \ip{\phi(s, a), \theta_h^\sim} - \ip{\phi(s, a), \theta_h} + \ip{\phi(s, a), \theta_h} - \max_{\theta \in \Theta_{G, h}} \ip{\phi(s, a), \theta} \\
        &\le \norm{\phi(s, a)}_{X_h^{-1}} \norm{\theta_h^\sim - \theta_h}_{X_h} \\
        &\le \frac{\featurebound}{\sqrt{\lambda}} c_h^\xi
        = \featurebound \modthetabound c_h^\xi / (H^{3/2}d) \, .
    \end{align*}
    The above two results together imply that, for all $h \in [H], (s, a) \in \cS_h \times \cA$
    \begin{align*}
        \abs{\max_{\theta \in \Theta_{G, h}} \ip{\phi(s, a), \theta} - \max_{\theta^\sim \in \Theta_{\tilde G, h}} \ip{\phi(s, a), \theta^\sim}}
        \le \featurebound \modthetabound c_h^\xi / (H^{3/2}d) \, .
    \end{align*}
    Following the same steps as above for $\min$ we can get that, for all $h \in [H], (s, a) \in \cS_h \times \cA$
    \begin{align*}
        \abs{\min_{\theta \in \Theta_{G, h}} \ip{\phi(s, a), \theta} - \min_{\theta^\sim \in \Theta_{\tilde G, h}} \ip{\phi(s, a), \theta^\sim}}
        \le \featurebound \modthetabound c_h^\xi / (H^{3/2}d) \, .
    \end{align*}
    The above two results together imply that, for all $h \in [H], (s, a) \in \cS_h \times \cA$, %
    \begin{align*}
        \abs{\left(\max_{\theta \in \Theta_{G, h}} \bar q_{\theta}(s, a) - \min_{\theta \in \Theta_{G, h}} \bar q_{\theta}(s, a)\right)
        - \left(\max_{\theta^\sim \in \Theta_{\tilde G, h}} \bar q_{\theta^\sim}(s, a) - \min_{\theta^\sim \in \Theta_{\tilde G, h}} \bar q_{\theta^\sim}(s, a)\right)}
        = \modxi \, ,
    \end{align*}
    where (by recalling \cref{eq:c_h definition}),
    \begin{align}
        \modxi 
        = 12 \sqrt{2d} H^2 \featurebound \xi \alpha^{-1} \left(2\sqrt{\nh}\featurebound\modthetabound / (H^{3/2}d)\right)^{H}
        \ge 2\featurebound \modthetabound c_h^\xi / (H^{3/2}d) \, .
        \label{eq:bar-xi definition}
    \end{align}
    This concludes the proof of the first result in \cref{lem:G cover q-max-min bound}.

    Now we prove the second result in \cref{lem:G cover q-max-min bound}.
    For any $\tilde G \in C_\xi^\bbG, h \in [H]$, define the event 
    \begin{align*}
        &\cE_2^{\tilde G, h} 
        = \Bigg\{ \Bigg|\bigE_{(S, A) \sim \mu_h} \left[\max_{\theta^\sim \in \Theta_{\tilde G, h}} \bar q_{\theta^\sim}(S, A) - \min_{\theta^\sim \in \Theta_{\tilde G, h}} \bar q_{\theta^\sim}(S, A)\right] \\
        &\qquad\qquad - \frac{1}{\nh} \sum_{i \in [\nh]} \left(\max_{\theta^\sim \in \Theta_{\tilde G, h}}\bar q_{\theta^\sim}(s_h^i, a_h^i) - \min_{\theta^\sim \in \Theta_{\tilde G, h}} \bar q_{\theta^\sim}(s_h^i, a_h^i)\right) \Bigg|
            \le \frac{H}{\sqrt{\nh}} \sqrt{\log\left(\frac{6H|C_\xi^\bbG|}{\delta}\right)} \Bigg\} \, .
    \end{align*}
    Then, since $\max_{\theta^\sim \in \Theta_{\tilde G, h}} \bar q_{\theta^\sim}(s, a) - \min_{\theta^\sim \in \Theta_{\tilde G, h}} \bar q_{\theta^\sim}(s, a) \in [0, H]$ for all $(s, a) \in \cS_h \times \cA$, by Hoeffding's inequality (\cref{lem:hoeffdings inequality}), we have that, for any $\tilde G \in C_\xi^\bbG, h \in [H]$, event $\cE_2^{\tilde G, h}$ occurs with probability at least $1 - \delta/(3H|C_\xi^\bbG|)$.
    Let 
    \begin{align}
        \cE_2 = \bigcap_{\tilde G \in C_\xi^\bbG, h \in [H]} \cE_2^{\tilde G, h} \, .
        \label{eq:event 2 definition}
    \end{align}
    Then, by applying a union bound over $\tilde G, h$ we have that the event $\cE_2$ occurs with probability at least $1 - \delta/3$.
\end{proof}

\begin{lemma} \label{lem:G cover useful results}
    Let $\xi > 0, \kappa_t \ge 0, \forall t \in [2:H+1]$. 
    Then, there exists a set $C_\xi^\bbG \subset \bbG$ with $|C_\xi^\bbG| \le (1 + 2\usedtobesqrtdoneplusone/\xi))^{dHd_0}$ such that, for any $G \in \bbG$, there exists a $\tilde G \in C_\xi^\bbG$ such that, 
    for any $h \in [H]$,
    for any $u \in [h]$ and trajectory $\trajectory = (s_t, a_t, r_t)_{t \in [u, H+1]}$,
    and for any $\theta_{h+1:H+1}, \theta^\sim_{h+1:H+1} \in \cB(\modthetabound)^{H-h+1}$ that are close in predictions, that is, such that for all $t \in [h+1:H+1], s \in \cS_t, \left|\bar v_{\theta_t}(s) - \bar v_{\theta^\sim_t}(s) \right| \le \kappa_t$, 
    it holds that  
    \begin{align*}
        &\left|\bigE_{\tau \sim F_{G, \trajectory, h+1}}\left[r_{h:\tau-1} + \bar v_{\theta_{h+1:H+1}} \left(s_{\tau}\right)\right] 
            - \bigE_{\tau \sim F_{\tilde G, \trajectory, h+1}}\left[r_{h:\tau-1} + \bar v_{\theta^\sim_{h+1:H+1}} \left(s_{\tau}\right)\right] \right| \\
        &\le (H-h+1)6 \sqrt{2d} H \featurebound \xi/\alpha + \sum_{t=h}^{H} \kappa_{t+1} \, .
    \end{align*}
\end{lemma}
\begin{proof}
    Recall that
    \begin{align*}
        \bbG = \left\{(G_h)_{h \in [2:H]} = (\vartheta_h^i)_{h \in [2:H], i \in [d_0]}: \text{ for all } h \in [2:H], i \in [d_0], \vartheta_h^i \in \cB(\usedtobesqrtdoneplusone)\right\} \, .
    \end{align*}
    By \cref{lem:cover number ball}, we know there exists a set $C_\xi \subset \cB(a), a, \xi > 0$ with $|C_\xi| = (1 + 2a/\xi)^d$ such that for any $x \in \cB(a)$ there exists a $y \in C_\xi$ such that $\norm{x - y}_2 \le \xi$.
    Define the set $C_\xi^\bbG = \bigtimes_{h \in [2:H], i \in [d_0]} C_\xi \subset \bbG$ with $|C_\xi^\bbG| \le (1 + 2\usedtobesqrtdoneplusone/\xi))^{dHd_0}$. 
    Then, for any $G=(\vartheta^i_h)_{h\in[2:H],i\in[d_0]} \in \bbG$, there exists a $\tilde G = (\tilde\vartheta^i_h)_{h\in[2:H],i\in[d_0]} \in C_\xi^\bbG$ such that 
    \begin{align*}
        \norm{\vartheta^i_h-\tilde\vartheta^i_h}_2\le\xi \quad \text{for all } h\in[2:H],i\in[d_0] \, . 
    \end{align*}
    Let $G, \tilde G$ be as defined above.
    Then, for all $s\in\cS\setminus\cS_1$ 
    \begin{align*}
        |\range^{G}(s) - \range^{\tilde G}(s)|
        &= \left|\max_{k \in [d_0]} \max_{a,a' \in \cA} \ip{\phi(s, a, a'), \vartheta_{\stage(s)}^k} - \max_{k \in [d_0]} \max_{a,a' \in \cA} \ip{\phi(s, a, a'), \tilde \vartheta_{\stage(s)}^k}\right| \\
        &\le \left|\max_{k \in [d_0]} \max_{a,a' \in \cA} \ip{\phi(s, a, a'), \vartheta_{\stage(s)}^k -  \tilde \vartheta_{\stage(s)}^k}\right| \\
        &\le \max_{k \in [d_0]} \max_{a,a' \in \cA} \norm{\phi(s, a, a')}_2\norm{\vartheta_{\stage(s)}^k -  \tilde \vartheta_{\stage(s)}^k}_2 \\
        &\le 2\featurebound \xi \, , 
    \end{align*}
    and, since we have by definition (\cref{eq:omega}) that $\omega_G$ is a smooth function in terms of $\range^G$, we get that
    \begin{align}
        |\omega_{G}(s) - \omega_{\tilde G}(s)|
        \le \left|2 - \frac{\sqrt{2d} \cdot \range^{G}(s)}{\alpha} - \left(2 - \frac{\sqrt{2d} \cdot \range^{\tilde G}(s)}{\alpha}\right)\right| 
        \le \frac{2 \sqrt{2d} \featurebound \xi}{\alpha} \, . 
        \label{eq:omega cover bound}
    \end{align}
    For all $h \in [H]$, let $\theta_{h+1:H+1}, \theta^\sim_{h+1:H+1} \in \cB(\modthetabound)^{H-h+1}$, such that for all $t \in [h+1:H+1], s \in \cS_t, \left|\bar v_{\theta_t}(s) - \bar v_{\theta^\sim_t}(s) \right| \le \kappa_t$.
    Then, for any $h \in [H], u \in [h]$, and trajectory $\trajectory = (s_t, a_t, r_t)_{t \in [u:H+1]}$,
    it holds that 
    \begin{align}
        &\bigE_{\tau \sim F_{G, \trajectory, h+1}}\left[r_{h:\tau-1} + \bar v_{\theta_{h+1:H+1}} \left(s_{\tau}\right)\right] \nonumber \\
        &= r_h + (1 - \omega_G(s_{h+1})) \bar v_{\theta_{h+1}}(s_{h+1}) + \omega_G(s_{h+1}) \left(r_{h+1} + (1- \omega_G(s_{h+2}))v_{\theta_{h+2}}(s_{h+2})\right) + \dots \nonumber \\
        &= r_h + (1 - \omega_G(s_{h+1})) \bar v_{\theta_{h+1}}(s_{h+1}) + \omega_G(s_{h+1}) \bigE_{\tau \sim F_{G, \trajectory, h+2}}\left[r_{h+1:\tau-1} + \bar v_{\theta_{h+2:H+1}} \left(s_{\tau}\right)\right] \, . 
        \label{eq:recursive target defn G}
    \end{align}
    Using similar steps to above it can be shown that
    \begin{align}
        &\bigE_{\tau \sim F_{\tilde G, \trajectory, h+1}}\left[r_{h:\tau-1} + \bar v_{\theta_{h+1:H+1}^\sim} \left(s_{\tau}\right)\right] \nonumber \\
        &= r_h + (1 - \omega_{\tilde G}(s_{h+1})) \bar v_{\theta_{h+1}}(s_{h+1}) + \omega_{\tilde G}(s_{h+1}) \bigE_{\tau \sim F_{\tilde G, \trajectory, h+2}}\left[r_{h+1:\tau-1} + \bar v_{\theta_{h+2:H+1}^\sim} \left(s_{\tau}\right)\right] \, . 
        \label{eq:recursive target defn tilde-G}
    \end{align}
    We claim that for any $h \in [H], u \in [h]$, and trajectory $\trajectory = (s_t, a_t, r_t)_{t \in [u:H+1]}$     
    \begin{align}
        &\left|\bigE_{\tau \sim F_{G, \trajectory, h+1}}\left[r_{h:\tau-1} + \bar v_{\theta_{h+1:H+1}} \left(s_{\tau}\right)\right] 
            - \bigE_{\tau \sim F_{\tilde G, \trajectory, h+1}}\left[r_{h:\tau-1} + \bar v_{\theta^\sim_{h+1:H+1}} \left(s_{\tau}\right)\right] \right| \nonumber \\
        &\le (H-h+1)6 \sqrt{2d} H \featurebound \xi/\alpha + \sum_{t=h}^{H} \kappa_{t+1} \, .
        \label{eq:target difference induction}
    \end{align}
    To show \cref{eq:target difference induction} we will use induction on $h$.
    The base case is when $h = H$, for which
    \begin{align*}
        &\left|\bigE_{\tau \sim F_{G, \trajectory, H+1}}\left[r_{H:\tau-1} + \bar v_{\theta_{H+1:H+1}} \left(s_{\tau}\right)\right] 
            - \bigE_{\tau \sim F_{\tilde G, \trajectory, H+1}}\left[r_{H:\tau-1} + \bar v_{\theta^\sim_{H+1:H+1}} \left(s_{\tau}\right)\right] \right| \nonumber \\ 
        &= \left|\bigE_{\tau \sim F_{G, \trajectory, H+1}}\left[r_{H:\tau-1}\right] 
            - \bigE_{\tau \sim F_{\tilde G, \trajectory, H+1}}\left[r_{H:\tau-1}\right] \right| 
       = 0 \, ,            
    \end{align*}
    where the first equality holds since $\Theta_{G, H+1} = \Theta_{\tilde G, H+1} = \{\vec 0\}$ (defined in \cref{eq:Theta definition}).
    The second equality holds since for any $u \in [H+1]$, and trajectory $\trajectory = (s_t, a_t, r_t)_{t \in [u:H+1]}$, $F_{G, \trajectory, H+1}(\tau = H+1) = F_{\tilde G, \trajectory, H+1}(\tau = H+1) = 1$.
    
    Now, we show the inductive step.
    Let $h \in [H-1]$ be arbitrary.
    Assume \cref{eq:target difference induction} holds for any $t \in [h+1, H]$.
    We prove that \cref{eq:target difference induction} also holds for $h$.
    Fix any $u \in [H+1]$, and trajectory $\trajectory = (s_t, a_t, r_t)_{t \in [u:H+1]}$.
    To shorten notation let $E_{G, \theta, h} = \bigE_{\tau \sim F_{G, \trajectory, h+1}}\left[r_{h:\tau-1} + \bar v_{\theta_{h+1:H+1}} \left(s_{\tau}\right)\right]$ (and similar for $E_{\tilde G, \theta^\sim, h}$).
    Then using the assumption that for all $u \in [2:H+1], s \in \cS_u, \left|\bar v_{\theta_u}(s) - \bar v_{\theta^\sim_u}(s) \right| \le \kappa_u$ along with \cref{eq:omega cover bound,eq:recursive target defn G,eq:recursive target defn tilde-G} and noting that for any $h \in [H], E_{G, \theta, h}, E_{\tilde G, \theta^\sim, h} \in [0, 2H], \bar v_{\theta_{h+1: H+1}}, \bar v_{\theta^\sim_{h+1: H+1}} \in [0, H], \omega_G, \omega_{\tilde G} \in [0, 1]$, we have that
    \begin{align}
        &\left|\bigE_{\tau \sim F_{G, \trajectory, h+1}}\left[r_{h:\tau-1} + \bar v_{\theta_{h+1:H+1}} \left(s_{\tau}\right)\right] 
            - \bigE_{\tau \sim F_{\tilde G, \trajectory, h+1}}\left[r_{h:\tau-1} + \bar v_{\theta^\sim_{h+1:H+1}} \left(s_{\tau}\right)\right] \right| \nonumber \\
        &= \bigg|(1 - \omega_G(s_{h+1})) \bar v_{\theta_{h+1}}(s_{h+1}) 
            - (1 - \omega_{\tilde G}(s_{h+1})) \bar v_{\theta^\sim_{h+1}}(s_{h+1}) \nonumber 
            + \omega_G(s_{h+1}) E_{G, \theta, h+1} 
            - \omega_{\tilde G}(s_{h+1}) E_{\tilde G, \theta^\sim, h+1}\bigg| \nonumber \\
        &\le \left|(1 - \omega_G(s_{h+1}) - (1 - \omega_{\tilde G}(s_{h+1}))\right| \left|\bar v_{\theta_{h+1}}(s_{h+1})\right| 
            + \left|1 - \omega_{\tilde G}(s_{h+1})\right| \left|\bar v_{\theta_{h+1}}(s_{h+1}) - \bar v_{\theta^\sim_{h+1}}(s_{h+1})\right| \nonumber \\
        &\qquad + \left|\omega_G(s_{h+1}) - \omega_{\tilde G}(s_{h+1})\right| \left|E_{G, \theta, h}\right| 
            + \left|\omega_{\tilde G}(s_{h+1})\right| \left|E_{G, \theta, h+1} - E_{\tilde G, \theta^\sim, h+1}\right| \nonumber \\
        &\le 2\sqrt{2d} H \featurebound \xi / \alpha + \kappa_{h+1} + 4\sqrt{2d} H \featurebound \xi / \alpha + \left|E_{G, \theta, h+1} - E_{\tilde G, \theta^\sim, h+1}\right| \nonumber \\
        &\le 6 \sqrt{2d} H \featurebound \xi / \alpha + \kappa_{h+1} + (H-h)6 \sqrt{2d} H \featurebound \xi / \alpha + \sum_{t=h+1}^{H+1} \kappa_{t+1} \nonumber \\
        &\le (H-h+1)6 \sqrt{2d} H \featurebound \xi/\alpha + \sum_{t=h}^{H} \kappa_{t+1} \nonumber \, ,
    \end{align}
    where the second last inequality holds by the inductive hypothesis.
\end{proof}

\newpage
\section{Other Useful Results and Definitions} \label{ss:other useful results}

\begin{definition}\label{def:nearopt}
A finite set $G\subset \R^d$ is the basis of a near-optimal design for a set $\Theta\subseteq\R^d$, if there exists a probability distribution $\rho$ over elements of $G$, such that
for any $\theta\in\Theta$,
\begin{align}
&\ip{v,\theta} = 0 \quad\text{for all } v\in\kernel(V(G, \rho)), \text{ and}\label{eq:v-max-rank}\\
&\norm{\theta}_{V(G, \rho)^\dag}^2\le 2d, \label{eq:near-opt-2d}\\
&\text{where } V(G, \rho)=\sum_{x\in G} \rho(x) xx^\top\,,\label{eq:design-matrix}
\end{align}
where for some matrix $X$, $X^\dag$ denotes the Moore-Penrose inverse of some, and 
$\kernel(X)$ its kernel (or null space).
\end{definition}

\begin{lemma}[Hoeffding's Inequality (Theorem 2 in \citep{hoeffding1994probability})] \label{lem:hoeffdings inequality}
    Let $(X_i)_{i \in \bN}$ be independent random variables such that $X_i \in [a, b]$ for some $a, b \in \bR$, and let $S_n = \frac{1}{n} \sum_{i=1}^n X_i$. 
    Then, with probability at least $1 - \zeta$ it holds that
    \begin{align*}
        \abs{\bigE S_n - S_n}
        \le \frac{(b-a)}{\sqrt{n}} \sqrt{\log\left(\frac{2}{\zeta}\right)} \, .
    \end{align*}
\end{lemma}

\begin{lemma}[Covering number of the Euclidean ball] \label{lem:cover number ball}
    Let $a > 0, \eps > 0, d \ge 1$, and $\cB_d(a)=\{x \in \bR^d: \norm{x}_2 \le a\}$ denote the $d$-dimensional Euclidean ball of radius $a$ centered at the origin.
    The covering number of $\cB_d(a)$ is upper bounded by $\left(1 + \frac{2a}{\eps}\right)^d$. 
\end{lemma}
\begin{proof}
    Same as the proof of Corollary 4.2.13 in \citep{vershynin2018high} with $\cB(1)$ replaced with $\cB(a)$.
\end{proof}

\begin{lemma}[Performance Difference Lemma (Lemma 3.2 in \citep{cai2020provably})] \label{lem:performance difference lemma}
    For any policies $\pi, \bar \pi$, it holds that  
    \begin{align*}
        v^{\pi}(s_1) - v^{\bar \pi}(s_1)
        = \sum_{h=1}^H \bigE_{(S_h, A_h) \sim \bPmarg{h}_{\pi, s_1}} \left(q^{\bar \pi}(S_h, A_h) - v^{\bar \pi}(S)\right) \, .
    \end{align*}
\end{lemma}

\begin{lemma}[Elliptical Potential Lemma (Lemma 19.4 in \citep{lattimore2020bandit})] \label{lem:elliptical potential}
    Let $V_0 \in \bR^{d \times d}$ be positive definite and $a_1, \dots, a_n \in \bR^d$ be a sequence of vectors with $\norm{a_t}_2 \le L \le \infty$ for all $t \in [n], V_t = V_0 + \sum_{s \le t} a_s a_s^\T$. then,
    \begin{align*}
        \sum_{t=1}^n \min \left\{ 1, \norm{a_t}_{V_{t-1}^{-1}}^2 \right\} 
        \le 2 \log \left( \frac{\det V_n}{\det V_0} \right) 
        \le 2d \log \left( \frac{\tr V_0 + n L^2}{d \det(V_0)^{1/d}} \right)
        \le 2d \log \left( \frac{d \lambda + n L^2}{d \lambda} \right) \, .
    \end{align*}
\end{lemma}

\begin{lemma}[Projection Bound (Lemma 8 in \citep{zanette2020learning})] \label{lem:projection bound}
    Let $(a_i)_{i \in [n]}$ be any sequence of vectors in $\bR^d$ and $(b_i)_{i\ in [n]}$ be any sequence of scalars such that $|b_i| \le c$.
    For any $\lambda \ge 0$ and $k \in \bN$ we have
    \begin{align*}
        \norm{\sum_{i=1}^n a_i b_i}^2_{(\sum_{i=1}^n a_i a_i^\top + \lambda I)^{-1}}
        \le n c^2 \, .
    \end{align*}
\end{lemma}

\newpage

\end{document}